\documentclass[10pt,journal,compsoc]{IEEEtran}
\usepackage{graphicx}
\usepackage{multirow}
\usepackage{subfigure}
\usepackage{longtable}
\usepackage{diagbox}
\usepackage{amsfonts,amssymb}
\usepackage{booktabs}
\usepackage{array}
\usepackage{amsmath}
\usepackage{lineno,hyperref}
\modulolinenumbers[5]
\usepackage{verbatim}
\usepackage{comment}
\usepackage{xcolor}
\usepackage{amsthm}
\usepackage{multicol}
\usepackage{threeparttable}
\usepackage{pifont}
\newcommand{\xmark}{\ding{55}}%
\newcommand{\cmark}{\ding{51}}%

\graphicspath{{./images/}}

\usepackage{tikz}

\newcommand\encircle[1]{%
  \tikz[baseline=(X.base)]
    \node (X) [draw, shape=circle, inner sep=0] {\strut #1};}

\usepackage{subfigure}

\usepackage[linesnumbered,ruled,lined]{algorithm2e}

\usepackage{algorithmic}

\newcolumntype{L}[1]{>{\raggedright\let\newline\\\arraybackslash\hspace{0pt}}m{#1}}
\newcolumntype{C}[1]{>{\centering\let\newline\\\arraybackslash\hspace{0pt}}m{#1}}
\newcolumntype{R}[1]{>{\raggedleft\let\newline\\\arraybackslash\hspace{0pt}}m{#1}}

\newcommand{\ind}{\perp\!\!\!\!\!\!\perp}

\DeclareMathOperator{\tr}{tr}
\DeclareMathOperator{\diag}{diag}

\begin{document}

\newtheorem{theorem}{Theorem}
\newtheorem{lemma}{Lemma}
\newtheorem{definition}{Definition}
\newtheorem{assumption}{Assumption}
\newtheorem{proposition}{Proposition}
\newtheorem{remark}{Remark}

\newsavebox\CBox
\def\textBF#1{\sbox\CBox{#1}\resizebox{\wd\CBox}{\ht\CBox}{\textbf{#1}}}

\title{The Conditional Cauchy-Schwarz Divergence with Applications to Time-Series Data and Sequential Decision Making}

\author{Shujian Yu, Hongming Li, Sigurd L{\o}kse, Robert Jenssen, and Jos\'{e} C. Pr\'{i}ncipe~\IEEEmembership{Life Fellow,~IEEE}
\thanks{This work was funded in part by the Research Council of Norway (RCN) under grant 309439, and the U.S. ONR under grant ONR N00014-21-1-2295.}
\thanks{Shujian Yu is with the Machine Learning Group, UiT - The Arctic University of Norway, Troms\o, Norway, and with the Quantitative Data Analytics Group, Vrije Universiteit Amsterdam (email:yusj9011@gmail.com).}
\thanks{Hongming Li and Jos\'{e} C. Pr\'{i}ncipe are with the Department of Electrical and Computer Engineering,
University of Florida, Gainesville, FL 32611, USA (e-mail:hongmingli1995@gmail.com; principe@cnel.ufl.edu).}
\thanks{Sigurd L{\o}kse is with the Drones and Autonomous Systems group at NORCE Norwegian Research Centre, Troms\o, Norway (email:sigl@norceresearch.no).}
\thanks{Robert Jenssen is with the Machine Learning Group, UiT - The Arctic University of Norway, Troms\o, Norway, with the Pioneer AI Centre, Copenhagen University, and with the Norwegian Computing Center, Oslo, Norway (email:robert.jenssen@uit.no).}

}




\IEEEcompsoctitleabstractindextext{
\begin{abstract}

The Cauchy-Schwarz (CS) divergence was developed by Pr\'{i}ncipe \emph{et al.} in $2000$. In this paper, we extend the classic CS divergence to quantify the closeness between two conditional distributions and show that the developed conditional CS divergence can be elegantly estimated by a kernel density estimator from given samples. We illustrate the advantages (e.g., rigorous faithfulness guarantee, lower computational complexity, higher statistical power, and much more flexibility in a wide range of applications) of our conditional CS divergence over previous proposals, such as the conditional KL divergence and the conditional maximum mean discrepancy. We also demonstrate the compelling performance of conditional CS divergence in two machine learning tasks related to time series data and sequential inference, namely time series clustering and uncertainty-guided exploration for sequential decision making. The code of conditional CS divergence is available at \url{https://github.com/SJYuCNEL/conditional_CS_divergence}.


\end{abstract}

}

\maketitle
\IEEEdisplaynotcompsoctitleabstractindextext
\IEEEpeerreviewmaketitle

\section{Introduction}


Quantifying the dissimilarity or distance between two probability distributions (i.e., $p(\mathbf{y})$ with respect to $q(\mathbf{y})$, where $\mathbf{y}$ is a random variable or vector) is a fundamental problem in pattern analysis and machine intelligence, receiving significant interests in both machine learning and statistics~\cite{flamary2016optimal,cherian2012jensen}. It also plays a central role in various problems, such as the generative models~\cite{arjovsky2017wasserstein}, the independence or conditional independence tests~\cite{gretton2007kernel,wang2015conditional}, and the design of robust deep learning loss functions~\cite{amid2019robust}.

Different types of divergence or discrepancy measures have been developed over the past few decades~\cite{basu2011statistical}. Among them, integral probability metrics (IPMs)~\cite{muller1997integral} and $f$-divergences~\cite{renyi1961measures,csiszar1967information} are perhaps the two most popular choices. IPMs aim to find a well-behaved function $f$ from a class of function $\mathcal{F}$ such that $|\mathbb{E}_{p}f(\mathbf{y})-\mathbb{E}_{q}f(\mathbf{y})|$ is maximized. Notable examples in this category include the Wasserstein distance and maximum mean discrepancy (MMD)~\cite{gretton2012kernel}. On the other hand, $f$-divergences, such as the Kullback-Leibler (KL) divergence and the Jensen-Shannon divergence, regard two distributions as identical if they assign the same likelihood to every point~\cite{zhao2021comparing}.



In the early $2000$s, Principe \emph{et al}.~\cite{principe2000information,principe2000learning} suggested a way to quantify the distributional dissimilarity simply by measuring the tightness of the famed Cauchy-Schwarz (CS) inequality associated with two distributions:
\begin{equation}
\Big| \int p(\mathbf{y})q(\mathbf{y})d\mathbf{y} \Big|^2 \leq \int \mid p(\mathbf{y})\mid^2 d\mathbf{y} \int \mid q(\mathbf{y})\mid^2 d\mathbf{y},
\end{equation}
with equality if and only if $p(\mathbf{y})$ and $q(\mathbf{y})$ are linearly dependent, a measure of the ``distance'' between the probability density functions (PDFs) can be defined with:
\begin{equation} \label{eq:CS_divergence}
\small
\begin{split}
& D_{\text{CS}} (p;q) = -\log\left( \frac{\Big| \int p(\mathbf{y})q(\mathbf{y})d\mathbf{y} \Big|^2}{\int \mid p(\mathbf{y})\mid^2 d\mathbf{y} \int \mid q(\mathbf{y})\mid^2 d\mathbf{y}} \right) \\
& = -2\log(\int p(\mathbf{y})q(\mathbf{y})d\mathbf{y}) + \log(\int p(\mathbf{y})^2 d\mathbf{y})
 + \log(\int q(\mathbf{y})^2 d\mathbf{y}),
\end{split}
\end{equation}
which was named the CS divergence.



The CS divergence enjoys a few appealing properties~\cite{jenssen2006cauchy}. For example, it has a closed-form expression for mixture-of-Gaussians (MoG)~\cite{kampa2011closed}, a property that KL divergence does not hold. Moreover, it can be simply evaluated with the kernel density estimator (KDE)~\cite{parzen1962estimation}. Specifically, suppose that we are
given $\{\mathbf{y}_i^p\}_{i=1}^M$ and $\{\mathbf{y}_j^q\}_{j=1}^N$, drawn~\emph{i.i.d.} from $p(\mathbf{y})$ and $q(\mathbf{y})$, respectively. Using KDE with Gaussian kernel $G_{\sigma}(\cdot)=\frac{1}{\sqrt{2\pi}\sigma}\exp(-\frac{\|\cdot\|^2}{2\sigma^2})$, Eq.~(\ref{eq:CS_divergence}) can be estimated as~\cite{jenssen2006cauchy}:
\begin{equation}\label{eq:CS_divergence_est}
\small
\begin{aligned}
& \widehat{D}_{\text{CS}} (p(\mathbf{y});q(\mathbf{y})) = -2 \log\left(\frac{1}{MN}\sum_{i=1}^M \sum_{j=1}^N G_{\sigma}({\bf y}_i^p-{\bf y}_j^q)\right) + \\
& \log\left(\frac{1}{M^2}\sum_{i,i'=1}^M G_{\sigma}({\bf y}_i^p-{\bf y}_{i'}^p)\right)
+ \log\left(\frac{1}{N^2}\sum_{j,j'=1}^N G_{\sigma}({\bf y}_j^q-{\bf y}_{j'}^q)\right).
\end{aligned}
\end{equation}
The estimator in Eq.~(\ref{eq:CS_divergence_est}) is also closely related to the famed MMD. We refer interested readers to Appendix~\ref{sec:CS_and_MMD} for more details.

Due to these properties, the CS divergence has been widely used in a variety of practical machine learning applications. Popular examples include measuring the similarity of two Poisson point processes~\cite{hoang2015cauchy} or data clustering by maximizing the sum of divergences between pairwise clusters~\cite{kampffmeyer2019deep}. Recently, it has been applied to variational autoencoders (VAE)~\cite{kingma2014auto,tran2022cauchy} by imposing a MoG prior, which significantly improves the representation power and the quality of generation. In another application to scene flow estimation for a pair of consecutive point clouds, the CS divergence demonstrated noticeable performance gain over the Chamfer distance and the Earth Mover's distance~\cite{he2022self}.

All the above-mentioned measures can only be used to evaluate the dissimilarity of two (marginal) distributions, whereas, in practice, one is also frequently interested in quantifying the dissimilarity of two conditional distributions, i.e., $p(\mathbf{y}|\mathbf{x})$ with respect to $q(\mathbf{y}|\mathbf{x})$. In fact, the conditional divergence itself has also been used, as a key ingredient, in recent machine learning applications. For example, in domain adaptation and generalization, given input $\mathbf{x}$ and class label $y$, suppose $\phi$ is a feature extractor, it is crucially important to align both the (marginal) distribution of latent representations $p(\phi(\mathbf{x}))$ and the conditional distribution $p(y|\phi(\mathbf{x}))$~\cite{zhao2019learning}, although in practice people resort to matching the class conditional distribution $p(\phi(\mathbf{x})|y)$ for simplicity (because $y$ is discrete)~\cite{long2017deep,yan2017mind}. In self-supervised learning, let $\mathbf{v}_1$ and $\mathbf{v}_2$ denote two different augmented views of the same input entity $\mathbf{x}$, suppose $\mathbf{z}_1$ and $\mathbf{z}_2$ are the latent representations of $\mathbf{v}_1$ and $\mathbf{v}_2$, respectively. To ensure minimum sufficient information in $\mathbf{z}_1$ and $\mathbf{z}_2$, it is necessary to minimize the conditional divergence between $p_\theta (\mathbf{z}_1 |\mathbf{v}_1)$ and $p_\varphi (\mathbf{z}_2 |\mathbf{v}_2)$, in which $\theta$ and $\varphi$ are parameters of view-specific encoders~\cite{federici2020learning}.

Despite the growing attention and increasing importance of the conditional divergence, methods on quantifying the dissimilarity or discrepancy of two conditional distributions are less investigated. In this paper, we follow the line of the classic CS divergence in Eq.~(\ref{eq:CS_divergence}) and extend it to conditional distributions. We illustrate the advantages of the developed conditional CS divergence over the conditional KL divergence and the conditional maximum mean discrepancy (CMMD)~\cite{ren2016conditional}. We also demonstrate its versatile applications and consistent performance boosts in two tasks related to time series data, namely time series clustering and uncertainty-guided exploration for reinforcement learning without rewards. We finally discuss the potential usage of the new divergence to other fundamental problems and identify its limitations that deserve further research.



\section{Background Knowledge}
\subsection{Problem Formulation}
We have two sets of observations $\psi_p=\{(\mathbf{x}_i^p,\mathbf{y}_i^p )\}_{i=1}^M$ and $\psi_q=\{(\mathbf{x}_j^q,\mathbf{y}_j^q )\}_{j=1}^N$ that are assumed to be independently and identically distributed (\emph{i.i.d.}) with density functions $p (\mathbf{x},\mathbf{y})$ and $q (\mathbf{x},\mathbf{y})$, respectively. Here, $\mathbf{x}$ is a random vector which contains input or explanatory variables, whereas $\mathbf{y}$ is the response or dependent variable (or vector).

Typically, the conditional distributions $p(\mathbf{y}|\mathbf{x})$ in $\psi_p$ and $q(\mathbf{y}|\mathbf{x})$ in $\psi_q$ are unknown and unspecified.
The aim of this paper is to develop a computationally efficient measure to quantify the closeness between $p(\mathbf{y}|\mathbf{x})$ and $q(\mathbf{y}|\mathbf{x})$ based on observations from $\psi_p$ and $\psi_q$, which we denote as $D(p(\mathbf{y}|\mathbf{x}); q(\mathbf{y}|\mathbf{x}))$.

\subsection{Existing Measures of $D(p (\mathbf{y}|\mathbf{x}); q(\mathbf{y}|\mathbf{x}))$}\label{sec:existing_measures}

We now briefly review previous efforts to measure $D(p (\mathbf{y}|\mathbf{x});q (\mathbf{y}|\mathbf{x}))$. The most natural choice is obviously the Kullback-Leibler (KL) divergence. Indeed, by the chain rule for KL divergence, $D_{\text{KL}}(p(\mathbf{y}|\mathbf{x});q (\mathbf{y}|\mathbf{x}))$ can be decomposed as:
\begin{equation}\label{eq:KL}
\begin{split}
D_{\text{KL}}(p (\mathbf{y}|\mathbf{x});q (\mathbf{y}|\mathbf{x})) & = D_{\text{KL}}(p (\mathbf{x},\mathbf{y});q (\mathbf{x},\mathbf{y})) \\
& - D_{\text{KL}}(p (\mathbf{x});q (\mathbf{x})).
\end{split}
\end{equation}


Eq.~(\ref{eq:KL}) suggests that one can measure $D_{\text{KL}}(p (\mathbf{y}|\mathbf{x});q (\mathbf{y}|\mathbf{x}))$ by simply taking the difference between $D_{\text{KL}}(p (\mathbf{x},\mathbf{y});q (\mathbf{x},\mathbf{y}))$ and $D_{\text{KL}}(p (\mathbf{x});q (\mathbf{x}))$, in which both terms can be evaluated by popular plug-in KL divergence estimators, such as the $k$-nearest neighbors ($k$-NN) estimator~\cite{wang2009divergence} and the KDE estimator. Albeit its simplicity, the decomposition rule requires estimation in the joint space of two variables, which further increases the dimensionality and makes estimation more difficult. On the other hand, a common drawback for $k$-NN estimator and its variants (e.g., \cite{noshad2017direct}) is that they are not differentiable, which limits their practical applications in modern deep learning tasks.

The second approach defines a distance metric through the embedding of probability functions in a reproducing kernel Hilbert space (RKHS) $\mathcal{F}$.
The RKHS estimators are much more well-behaved with respect to the dimensionality of the input data, which is also particularly important for our CS divergence.
Specifically, let $\{\mathbf{y}_i^p\}_{i=1}^M$ and $\{\mathbf{y}_j^q\}_{j=1}^N$ be the sets of samples from $p (\mathbf{y})$ and $q (\mathbf{y})$ respectively, the maximum mean discrepancy (MMD)~\cite{gretton2012kernel} is defined as the distance of feature means $\mu_p$ and $\mu_q$ in $\mathcal{F}$. That is,
\begin{equation}
\text{MMD}^2[p (\mathbf{y}),q (\mathbf{y})] =\|\mu_p-\mu_q \|_{\mathcal{F}}^2
\end{equation}

In practice, an estimate of the MMD objective compares the square difference between the empirical kernel mean embeddings:
\begin{equation}\label{eq:MMD}
\small
\begin{aligned}
& \widehat{\text{MMD}}^2[p (\mathbf{y}),q (\mathbf{y})] = \|\frac{1}{M}\sum_{i=1}^M \phi(\mathbf{y}_i^p)- \frac{1}{N}\sum_{j=1}^N \phi(\mathbf{y}_j^q) \|_{\mathcal{F}}^2 \\
& = \left[ \frac{1}{M^2}\sum_{i,i'}^M \kappa(\mathbf{y}_i^p,\mathbf{y}_{i'}^p) - \frac{2}{MN}\sum_{i,j}^{MN}\kappa(\mathbf{y}_i^p,\mathbf{y}_j^q)
 + \frac{1}{N^2}\sum_{j,j'}^N \kappa(\mathbf{y}_j^q,\mathbf{y}_{j'}^q) \right],
\end{aligned}
\end{equation}
where $\phi(\mathbf{y}):=\kappa(\mathbf{y},\cdot)$ refers to a (usually infinite dimension) feature map of $\mathbf{y}$, $\kappa: \mathcal{Y}\times \mathcal{Y} \rightarrow \mathbb{R}$ is a positive definite kernel (usually Gaussian).

An unbiased and more commonly used MMD estimator is given by:
\begin{equation}
\begin{split}
& \widehat{\text{MMD}}_u^2[p (\mathbf{y}),q (\mathbf{y})]
= \biggr[ \frac{1}{M(M-1)}\sum_{i}^M \sum_{i\neq i'}^M \kappa(\mathbf{y}_i^p,\mathbf{y}_{i'}^p) \\
& - \frac{2}{MN}\sum_{i,j}^{MN}\kappa(\mathbf{y}_i^p,\mathbf{y}_j^q)
 + \frac{1}{N(N-1)}\sum_{j}^N \sum_{j\neq j'}^N \kappa(\mathbf{y}_j^q,\mathbf{y}_{j'}^q) \biggr].
\end{split}
\end{equation}

The RKHS embedding of conditional distributions was developed by Song \emph{et al.}~\cite{song2009hilbert,song2013kernel} via a conditional covariance operator $\mathcal{C}_{Y|X}=\mathcal{C}_{YX} \mathcal{C}_{XX}^{-1}=\mathcal{C}_{YX} \left(\mathcal{C}_{XX}+\lambda I\right)^{-1}$, in which $\mathcal{C}_{XX}$ and $\mathcal{C}_{YX}$ are respectively the uncentered covariance and cross-covariance operator. Later, Ren \emph{et al.}~\cite{ren2016conditional} measure the difference of two conditional distributions by the square difference between the empirical estimates of the conditional embedding operators. The so-called conditional maximum mean discrepancy (CMMD) is defined as:
\begin{equation}\label{eq:CMMD}
\begin{aligned}
& \widehat{\text{CMMD}}^2[p (\mathbf{y}|\mathbf{x}),q (\mathbf{y}|\mathbf{x})] = \| \widehat{\mathcal{C}}_{Y|X}^p - \widehat{\mathcal{C}}_{Y|X}^q \|_{\mathcal{F}\otimes \mathcal{G}}^2 \\
& = \| \Phi_p (K_p + \lambda I)^{-1}\Upsilon_p^T - \Phi_q (K_q + \lambda I)^{-1}\Upsilon_q^T \|_{\mathcal{F}\otimes \mathcal{G}}^2  \\
& = \tr (K_p \tilde{K}_p^{-1} L_p \tilde{K}_p^{-1}) + \tr (K_q \tilde{K}_q^{-1} L_q \tilde{K}_q^{-1}) \\
& - 2\tr (K_{pq} \tilde{K}_q^{-1} L_{qp} \tilde{K}_p^{-1}),
\end{aligned}
\end{equation}
in which $\otimes$ denotes tensor product, $\mathcal{G}$ is the RKHS corresponding to $\mathbf{x}$, $\tilde{K}=K+\lambda I$, $\tr$ denotes matrix trace.
$\Phi_p = [\phi(\mathbf{y}_1^p),\phi(\mathbf{y}_2^p),\cdots,\phi(\mathbf{y}_M^p)]$ and $\Upsilon_p = [\phi(\mathbf{x}_1^p),\phi(\mathbf{x}_2^p),\cdots,\phi(\mathbf{x}_M^p)]$ are implicitly formed feature matrices for dataset $\psi_p$. $\Phi_q$ and $\Upsilon_q$ are defined similarly for dataset $\psi_q$. $K_p = \Upsilon_p^T \Upsilon_p \in \mathbb{R}^{M\times M}$ and $K_q = \Upsilon_q^T \Upsilon_q \in \mathbb{R}^{N\times N}$ are the Gram matrices for input variable $X$,
$L_p = \Phi_p^T \Phi_p \in \mathbb{R}^{M\times M}$ and $L_q = \Phi_q^T \Phi_q \in \mathbb{R}^{N\times N}$ are the Gram matrices for output variable $Y$. Finally,
$K_{pq} = \Upsilon_p^T \Upsilon_q \in \mathbb{R}^{M\times N} $ and $L_{qp} = \Phi_q^T \Phi_p \in \mathbb{R}^{N\times M} $ are the Gram matrices
between $\psi_p$ and $\psi_q$ on input and output variables, i.e.,
$(K_{pq})_{ij}=\kappa(\mathbf{x}_i^p,\mathbf{x}_j^q)$ and $(L_{qp})_{ji}=\kappa(\mathbf{y}_j^q,\mathbf{y}_i^p)$.

The conditional covariance operator $\mathcal{C}_{Y|X}$ relies on stringent assumptions that are often violated, hindering its analysis. To overcome this limitation, Park \emph{et al.}~\cite{park2020measure} developed an operator-free, measure-theoretic approach for conditional mean embedding and defined the maximum conditional mean discrepancy (MCMD) between \( p(Y|X) \) and \( q(Y|X) \) as a function from \(\mathcal{X}\) to \(\mathbb{R}\), which has a closed-form estimator. Informally, \(\mathrm{MCMD}_{p(Y|X),q(Y|X)}(\bf x)\) can be viewed as the MMD between $p(Y|X=\bf x)$ and $q(Y|X=\bf x)$.

The recently developed conditional Bregman matrix divergence~\cite{yu2020measuring} applies the decomposition rule in Eq.~(\ref{eq:KL}) and quantifies the difference of two conditional distributions via:
\begin{equation}
D_{\varphi,B} (p (\mathbf{y}|\mathbf{x});q (\mathbf{y}|\mathbf{x}))=D_{\varphi,B} (C_{\mathbf{xy}}^p;C_{\mathbf{xy}}^q )- D_{\varphi,B} (C_\mathbf{x}^p;C_\mathbf{x}^q),
\end{equation}
in which \( D_{\varphi, B} \) refers to the Bregman matrix divergence~\cite{kulis2009low}. \( C_{\mathbf{xy}} \in \mathcal{S}_{+}^{d_{\mathbf{x}}+d_{\mathbf{y}}} \) denotes the sample covariance matrix or the centered correntropy matrix~\cite{rao2011test}, which characterizes the joint distribution of $\mathbf{x}$ and $\mathbf{y}$. Meanwhile, \( C_{\mathbf{x}} \in \mathcal{S}_{+}^{d_{\mathbf{x}}} \) denotes the sample covariance matrix or the centered correntropy matrix that characterizes the marginal distribution of $\mathbf{x}$\footnote{We denote \( \mathcal{S}_{+}^{d} \) as the set of all \( d \times d \) symmetric positive semidefinite (SPS) matrices, i.e., \( \mathcal{S}_{+}^{d} = \{ A \in \mathbb{R}^{d \times d} | A = A^T, A \succcurlyeq 0 \} \).} Formally, given a strictly convex, differentiable function $\varphi$, the matrix Bregman divergence from a matrix $\rho$ to a matrix $\sigma$ is defined as:
\begin{equation}
D_{\varphi,B}(\sigma;\rho) = \varphi(\sigma) - \varphi(\rho) - \tr\left( (\nabla\varphi(\rho))^T(\sigma-\rho) \right),
\end{equation}
where $\tr(\cdot)$ denotes the trace. For example, when $\varphi(\sigma)=\tr(\sigma \log\sigma-\sigma)$, where $\log\sigma$ is the matrix logarithm, the resulting Bregman matrix divergence is:
\begin{equation}
D_{\text{vN}}(\sigma;\rho) = \tr (\sigma\log\sigma - \sigma\log\rho - \sigma +\rho),
\end{equation}
which is also referred to as von Neumann divergence~\cite{nielsen2010quantum}.


We summarize in Table~\ref{tab:property} different properties (such as computational complexity, differentiability, and faithfulness) associated with different existing measures. We also include our new measure for a comparison. Here, we define ``faithfulness" of a measure when $D(p (\mathbf{y}|\mathbf{x});q (\mathbf{y}|\mathbf{x}))\geq 0$ in general and $D(p (\mathbf{y}|\mathbf{x});q (\mathbf{y}|\mathbf{x})) = 0$ if and only if $p (\mathbf{y}|\mathbf{x}) = q (\mathbf{y}|\mathbf{x})$.
The computational complexity of conditional KL divergence estimator by $k$-NN graph can be reduced to $\mathcal{O}(kN\log N)$. The computational complexity of conditional MMD comes from the computation of Gram matrices ($\mathcal{O}(N^2d)$ complexity) and matrix inversion ($\mathcal{O}(N^3)$ complexity). Our conditional CS divergence also requires computation of Gram matrices, but avoids matrix inversion. The computational complexity of conditional Bregman divergence is dominated by the value of $d$, rather than $N$. This is because it requires evaluating a covariance matrix of size $d\times d$ and its eigenvalues.
Note that, the computation of Gram matrices can take only $\mathcal{O}(RN\log d)$ complexity, where $R$ is the number of basis functions for approximating kernels which determines the approximation accuracy~\cite{zhao2015fastmmd}. 

\begin{table}[]
\centering
\caption{Properties of different conditional divergences. ``Diff." refers to the differentiability; ``Faith." refers to the faithfulness.}
\begin{threeparttable}[t]
\begin{tabular}{l|l|l|l|l}
\hline
& Hyperparameter & Complexity\tnote{2} & Diff. & Faith. \\ \hline
Cond. KL\tnote{1}   & $k$           &    $\mathcal{O}(kN\log N)$    &     \xmark     & \cmark   \\ \hline
Cond. MMD     &     kernel size $\sigma$; $\lambda$        &   $\mathcal{O}(N^2d+N^3)$    & \cmark        &    \textbf{?}\tnote{3}    \\ \hline
Cond. Bregman\tnote{4} &    free        &     $\mathcal{O}(Nd^2+d^3)$    & \cmark        &   \xmark       \\ \hline
Cond. CS (ours)    &   kernel size $\sigma$   &   $\mathcal{O}(N^2d)$      & \cmark        & \cmark   \\ \hline
\end{tabular}
\begin{tablenotes}
     \item[1] We assume the conditional KL divergence is estimated nonparametrically with $k$-nearest neighbors ($k$-NN) graph.
     \item[2] $N$ refers to number of samples, $d$ is the dimension of $\mathbf{x}$ or $\mathbf{y}$.
     \item[3] Indeed, there is no universally agreed-upon definition for conditional MMD, and the faithfulness property depends on the specific definition~\cite{park2020measure}. But this is beyond the scope of our work.
     \item[4] We assume the conditional Bregman divergence is measured with the sample covariance matrix.
\end{tablenotes}
\end{threeparttable}
\label{tab:property}
\end{table}

\section{The Conditional Cauchy-Schwarz divergence}

\subsection{Extending Cauchy-Schwarz divergence for conditional distributions}

Following Eq.~(\ref{eq:CS_divergence}), the CS divergence for two conditional distributions $p(\mathbf{y}|\mathbf{x})$ and $q(\mathbf{y}|\mathbf{x})$ can be expressed naturally as:
\begin{equation} \label{eq:conditional_CS}
\begin{split}
& D_{\text{CS}}(p(\mathbf{y}|\mathbf{x});q(\mathbf{y}|\mathbf{x}))  = - 2 \log \left(\int_\mathcal{X}\int_\mathcal{Y} p(\mathbf{y}|\mathbf{x})q(\mathbf{y}|\mathbf{x}) d\mathbf{x}d\mathbf{y} \right) \\
& + \log \left(\int_\mathcal{X}\int_\mathcal{Y} p^2(\mathbf{y}|\mathbf{x})d\mathbf{x}d\mathbf{y}\right) + \log \left(\int_\mathcal{X}\int_\mathcal{Y} q^2(\mathbf{y}|\mathbf{x})d\mathbf{x}d\mathbf{y}\right) \\
& = - 2 \log \left(\int_\mathcal{X}\int_\mathcal{Y} \frac{p(\mathbf{x},\mathbf{y})q(\mathbf{x},\mathbf{y})}{p(\mathbf{x})q(\mathbf{x})} d\mathbf{x}d\mathbf{y} \right) \\
& + \log \left(\int_\mathcal{X}\int_\mathcal{Y} \frac{p^2(\mathbf{x},\mathbf{y})}{p^2(\mathbf{x})} d\mathbf{x}d\mathbf{y}\right) + \log \left(\int_\mathcal{X}\int_\mathcal{Y} \frac{q^2(\mathbf{x},\mathbf{y})}{q^2(\mathbf{x})} d\mathbf{x}d\mathbf{y}\right),
\end{split}
\end{equation}
which contains two conditional quadratic terms (i.e., $\int_\mathcal{X}\int_\mathcal{Y} \frac{p^2(\mathbf{x},\mathbf{y})}{p^2(\mathbf{x})} d\mathbf{x}d\mathbf{y}$ and $\int_\mathcal{X}\int_\mathcal{Y} \frac{q^2(\mathbf{x},\mathbf{y})}{q^2(\mathbf{x})} d\mathbf{x}d\mathbf{y}$) and a cross term (i.e., $\int_\mathcal{X}\int_\mathcal{Y} \frac{p(\mathbf{x},\mathbf{y})q(\mathbf{x},\mathbf{y})}{p(\mathbf{x})q(\mathbf{x})} d\mathbf{x}d\mathbf{y}$).

\begin{proposition}
The conditional CS divergence defined in Eq.~(\ref{eq:conditional_CS}) is a ``faithful" measure on the closeness between $p(\mathbf{y}|\mathbf{x})$ and $q(\mathbf{y}|\mathbf{x})$.
\end{proposition}

\begin{proof}
All proof(s) are demonstrated in Appendix~\ref{sec:appendix_proof}.
\end{proof}

We demonstrate below that Eq.~(\ref{eq:conditional_CS}) has a closed-form empirical estimator. The estimation technique used in our paper is a bit different from that in~\cite{jenssen2006cauchy}, but enjoys more advantages. Interested readers can refer to Appendix~\ref{sec:two_estimators}.


\begin{proposition}
Assume that we are given observations $\psi_p=\{(\mathbf{x}_i^p,\mathbf{y}_i^p )\}_{i=1}^M$ and $\psi_q=\{(\mathbf{x}_j^q,\mathbf{y}_j^q )\}_{j=1}^N$, sampled from distributions $p(\mathbf{x},\mathbf{y})$ and $q(\mathbf{x},\mathbf{y})$, respectively. Let $K^p$ and $L^p$ denote, respectively, the Gram matrices for variables $\mathbf{x}$ and $\mathbf{y}$ in the distribution $p$. Similarly, let $K^q$ and $L^q$ denote, respectively, the Gram matrices for variables $\mathbf{x}$ and $\mathbf{y}$ in the distribution $q$. Meanwhile, let $K^{pq}\in \mathbb{R}^{M\times N}$ (i.e., $\left(K^{pq}\right)_{ij}=\kappa(\mathbf{x}^p_i - \mathbf{x}^q_j)$) denote the Gram matrix from distribution $p$ to distribution $q$ for input variable $\mathbf{x}$, and $L^{pq}\in \mathbb{R}^{M\times N}$ the Gram matrix from distribution $p$ to distribution $q$ for output variable $\mathbf{y}$.
Similarly, let $K^{qp}\in \mathbb{R}^{N\times M}$ (i.e., $\left(K^{qp}\right)_{ji}=\kappa(\mathbf{x}^q_j - \mathbf{x}^p_i)$) denote the Gram matrix from distribution $q$ to distribution $p$ for input variable $\mathbf{x}$, and $L^{qp}\in \mathbb{R}^{N\times M}$ the Gram matrix from distribution $q$ to distribution $p$ for output variable $\mathbf{y}$.
The empirical estimation of $D_{\text{CS}}(p(\mathbf{y}|\mathbf{x});q(\mathbf{y}|\mathbf{x}))$ is given by:
\begin{equation}\label{eq:conditional_CS_est}
\begin{split}
& \widehat{D}_{\text{CS}}(p(\mathbf{y}|\mathbf{x});q(\mathbf{y}|\mathbf{x}))  \approx \log\left( \sum_{j=1}^M \left( \frac{ \sum_{i=1}^M K_{ji}^p L_{ji}^p }{ (\sum_{i=1}^M K_{ji}^p)^2 } \right) \right) \\
& + \log\left( \sum_{j=1}^N \left( \frac{ \sum_{i=1}^N K_{ji}^q L_{ji}^q }{ (\sum_{i=1}^N K_{ji}^q)^2 } \right) \right) \\
& - \log \left( \sum_{j=1}^M \left( \frac{ \sum_{i=1}^N K_{ji}^{pq} L_{ji}^{pq} }{ (\sum_{i=1}^M K_{ji}^p) (\sum_{i=1}^N K_{ji}^{pq}) } \right) \right) \\
& - \log \left( \sum_{j=1}^N \left( \frac{ \sum_{i=1}^M K_{ji}^{qp} L_{ji}^{qp} }{ (\sum_{i=1}^M K_{ji}^{qp}) (\sum_{i=1}^N K_{ji}^q) } \right) \right).
\end{split}
\end{equation}
\end{proposition}

\begin{remark}[Difference between CS and conditional CS]
Upon closely examining the expressions for CS divergence and conditional CS divergence, we can identify some interesting connections. Suppose $M=N$ for simplicity, the expression of CS divergence according to Eq.~(\ref{eq:CS_divergence_est}) can be reformulated as\footnote{In fact, the last two terms in Eq.~(\ref{eq:CS_reformulate}) are equal. We reformulate the CS divergence in this form only for the ease of the following analysis.}:
\begin{equation}\label{eq:CS_reformulate}
\begin{split}
& D_{\text{CS}} (p(\mathbf{y});q(\mathbf{y})) = \underbrace{\log \left(\frac{1}{M^2} \tr(L^p\cdot\mathbf{1})\right)}_{\text{within-distrib. similarity}} + \underbrace{\log \left(\frac{1}{N^2} \tr(L^q\cdot\mathbf{1})\right)}_{\text{within-distrib. similarity}} \\
& - \underbrace{\log \left(\frac{1}{MN} \tr(L^{pq}\cdot\mathbf{1})\right)}_{\text{cross-distrib. similarity}} - \underbrace{\log \left(\frac{1}{NM} \tr(L^{qp}\cdot\mathbf{1})\right)}_{\text{cross-distrib. similarity}},
\end{split}
\end{equation}
where $\mathbf{1}$ is the all-ones matrix. $\log \left(\frac{1}{M^2} \tr(L^p\cdot\mathbf{1})\right) = \log\left(\frac{1}{M^2}\sum_{i,i'=1}^M \kappa({\bf y}_i^p-{\bf y}_{i'}^p)\right) \approx \log \left( \mathbb{E}_{\mathbf{y}_i\sim p,\mathbf{y}_{i'}\sim p} \kappa(\mathbf{y}_i^p-\mathbf{y}_{i'}^p) \right)$ measures the logarithm of the expected \emph{within-distribution similarity} for any two samples that are drawn from $p$ (the same interpretation applies to the term $\log \left(\frac{1}{N^2} \tr(L^q\cdot\mathbf{1})\right)$). By contrast, $\log \left(\frac{1}{MN} \tr(L^{pq}\cdot\mathbf{1})\right) = \log\left(\frac{1}{MN}\sum_{i=1}^M \sum_{j=1}^N \kappa({\bf y}_i^p-{\bf y}_j^q)\right) \approx  \log \left( \mathbb{E}_{\mathbf{y}_i\sim p,\mathbf{y}_j\sim q} \kappa(\mathbf{y}_i^p-\mathbf{y}_j^q) \right)$ measures the logarithm of the expected \emph{cross-distribution similarity} for any two samples such that one is drawn from $p$, whereas another is drawn from $q$ (the same interpretation applies to the term $\log \left(\frac{1}{NM} \tr(L^{qp}\cdot\mathbf{1})\right)$).

Similarly, the expression of conditional CS divergence (in Eq.~(\ref{eq:conditional_CS_est})) can be reformulated as:
\begin{equation}\label{eq:conditional_CS_reformulate}
\begin{split}
& D_{\text{CS}} (p(\mathbf{y}|\mathbf{x});q(\mathbf{y}|\mathbf{x})) = \log \left(\tr(L^p\cdot C_1)\right) + \log \left(\tr(L^q\cdot C_2)\right) \\
& - \log \left(\tr(L^{pq}\cdot C_3)\right) - \log \left(\tr(L^{qp}\cdot C_4)\right),
\end{split}
\end{equation}
where $C_1$, $C_2$, $C_3$ and $C_4$ are some matrices based on the conditional variables $\mathbf{x}$ in both data sets.
\end{remark}

Comparing Eq.~(\ref{eq:CS_reformulate}) with Eq.~(\ref{eq:conditional_CS_reformulate}), it is easy to observe that both CS divergence and conditional CS divergence leverage the general idea of summing up within-distribution similarities and then subtracting cross-distribution similarity, and rely on the Gram matrices $L^p$, $L^q$, $L^{qp}$ and $L^{pq}$. However, the CS divergence assigns uniform weights on elements of those Gram matrices, whereas the conditional CS divergence applies non-uniform weights, in which the weights are only determined by the conditional variable $\mathbf{x}$. Specifically, $C_1 \in \mathbb{R}^{M\times M}$, $C_2 \in \mathbb{R}^{N\times N}$, $C_3 \in \mathbb{R}^{N\times M}$ and $C_4 \in \mathbb{R}^{M\times N}$ are given by:
\begin{equation}
C_1 =
\begin{pmatrix}
\frac{K_{11}^p}{(\sum_{i=1}^M K_{1i}^p)^2}  & \cdots & \frac{K_{M1}^p}{(\sum_{i=1}^M K_{Mi}^p)^2} \\
\vdots & \ddots & \vdots\\
\frac{K_{1M}^p}{(\sum_{i=1}^M K_{1i}^p)^2} & \cdots & \frac{K_{MM}^p}{(\sum_{i=1}^M K_{Mi}^p)^2}
\end{pmatrix},
\end{equation}

\begin{equation}
C_2 =
\begin{pmatrix}
\frac{K_{11}^q}{(\sum_{i=1}^N K_{1i}^q)^2}  & \cdots & \frac{K_{N1}^q}{(\sum_{i=1}^N K_{Ni}^q)^2} \\
\vdots & \ddots & \vdots\\
\frac{K_{1N}^q}{(\sum_{i=1}^N K_{1i}^q)^2} & \cdots & \frac{K_{NN}^q}{(\sum_{i=1}^N K_{Ni}^q)^2} ,\
\end{pmatrix},
\end{equation}

\begin{equation}
C_3 =
\begin{pmatrix}
\frac{K_{11}^{pq}}{(\sum_{i=1}^M K_{1i}^p)(\sum_{i=1}^N K_{1i}^{pq})}  & \cdots & \frac{K_{M1}^{pq}}{(\sum_{i=1}^M K_{Mi}^p)(\sum_{i=1}^N K_{Mi}^{pq})} \\
\vdots & \ddots & \vdots\\
\frac{K_{1N}^{pq}}{(\sum_{i=1}^M K_{1i}^p)(\sum_{i=1}^N K_{1i}^{pq})} & \cdots & \frac{K_{MN}^{pq}}{(\sum_{i=1}^M K_{Mi}^p)(\sum_{i=1}^N K_{Mi}^{pq})}
\end{pmatrix},
\end{equation}

\begin{equation}
C_4 =
\begin{pmatrix}
\frac{K_{11}^{qp}}{(\sum_{i=1}^M K_{1i}^{qp})(\sum_{i=1}^N K_{1i}^{q})}  & \cdots & \frac{K_{1N}^{qp}}{(\sum_{i=1}^M K_{Mi}^{qp})(\sum_{i=1}^N K_{Mi}^{q})} \\
\vdots & \ddots & \vdots\\
\frac{K_{1M}^{qp}}{(\sum_{i=1}^M K_{1i}^{qp})(\sum_{i=1}^N K_{1i}^{q})} & \cdots & \frac{K_{NM}^{qp}}{(\sum_{i=1}^M K_{Mi}^{qp})(\sum_{i=1}^N K_{Mi}^{q})}
\end{pmatrix},
\end{equation}

We provide a geometric interpretation in Fig.~\ref{fig:interpretation} to further clarify the difference and use the cross-distribution similarity term as an example, i.e., $\tr(L^{pq}\cdot\mathbf{1})$ with respect to $\tr(L^{pq}\cdot C_3)$.
For both marginal and conditional CS divergences, given a point $\mathbf{y}_i^p$ (i.e., the $i$-th sample that is drawn from $p(\mathbf{y})$), we need to evaluate all its cross-distribution similarities, i.e., $L_{i1}^{pq}=\kappa(\mathbf{y}_i^p-\mathbf{y}_1^q)$ (solid line), $L_{i2}^{pq}=\kappa(\mathbf{y}_i^p-\mathbf{y}_2^q)$ (dotted line), $\cdots$, and $L_{iN}^{pq}=\kappa(\mathbf{y}_i^p-\mathbf{y}_N^q)$ (dashed line).

Without random variable $\mathbf{x}$, the sum of all similarities for $\mathbf{y}_i^p$ is simply $\sum_{j=1}^N L_{ij}^{pq}$ (i.e., equal weight).
Now, if there is a conditional variable $\mathbf{x}$, it can be expected that the sum of cross-distribution similarities will be influenced by this variable. Due to the \emph{i.i.d.} assumption, we find that the weight on $L_{ij}^{pq}$ is only determined by the similarity between $\mathbf{x}_i^p$ and $\mathbf{x}_j^q$ (i.e., $K_{ij}^{pq}$), and is independent to $K_{i1}^{pq}$, $\cdots$, $K_{i,j-1}^{pq}$, $K_{i,j+1}^{pq}$, $\cdots$, $K_{iN}^{pq}$ (if we ignore the normalization term), which makes sense. Moreover, the reason why the impact of $K_{ij}^{pq}$ to $L_{ij}^{pq}$ takes the form $L_{ij}^{pq} K_{ij}^{pq}$ (rather than other nonlinear forms) can be attributed to the Gaussian kernel that we used in our paper, which satisfies $\kappa\left(\begin{bmatrix} \mathbf{x}_i \\ \mathbf{y}_i \end{bmatrix} - \begin{bmatrix} \mathbf{x}_j \\ \mathbf{y}_j \end{bmatrix} \right) = \kappa (\mathbf{x}_i - \mathbf{x}_j) \kappa (\mathbf{y}_i - \mathbf{y}_j)$.

\begin{figure}[t]
	\centering
		\includegraphics[width=.7\linewidth]{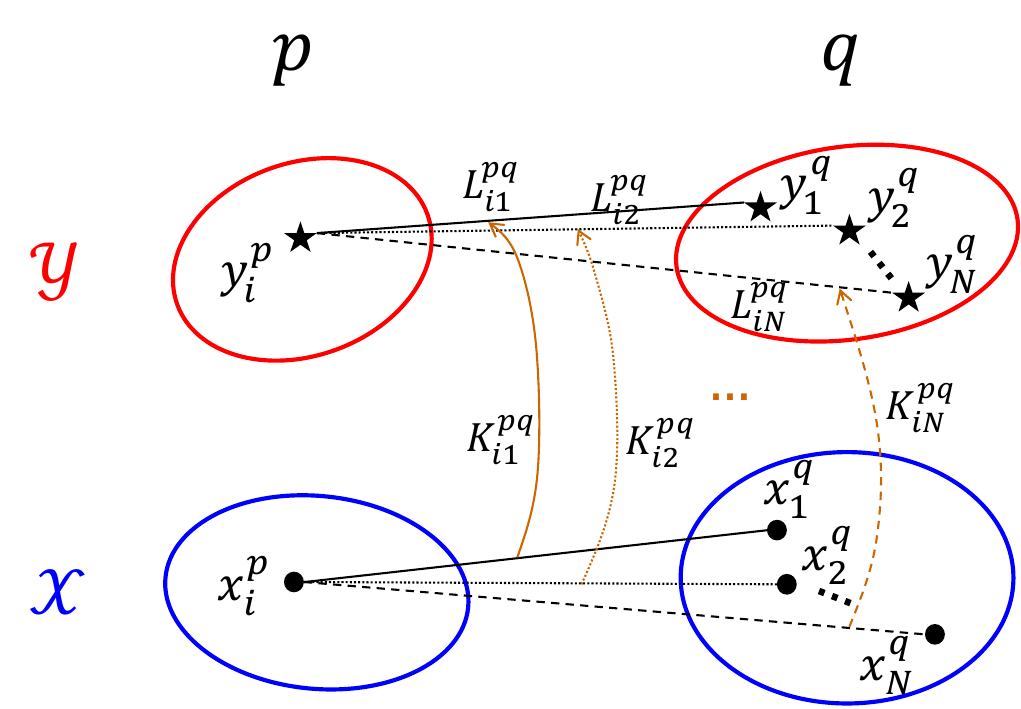}
	\caption{To evaluate the expected value of cross-distribution similarity for $\mathbf{y}_i^p$, the weight on $L_{ij}^{pq}$ is only determined by $K_{ij}^{pq}$, and is independent to $K_{i1}^{pq}$, $\cdots$, $K_{i,j-1}^{pq}$, $K_{i,j+1}^{pq}$, $\cdots$, $K_{iN}^{pq}$.}
	\label{fig:interpretation}
\end{figure}


\begin{remark}[Difference between conditional CS and conditional MMD]
Interestingly, the relationship between CS divergence (Eq.~(\ref{eq:CS_reformulate})) and conditional CS divergence (Eq.~(\ref{eq:conditional_CS_reformulate})) also holds for MMD and conditional MMD~\cite{ren2016conditional}. Specifically, suppose $M=N$, the objective of MMD in Eq.~(\ref{eq:MMD}) can be rewritten as:
\begin{equation}
    \widehat{\text{MMD}}^2 = \frac{1}{M^2}\tr(L^p\cdot\mathbf{1}) + \frac{1}{N^2}\tr(L^q\cdot\mathbf{1}) - \frac{2}{MN}\tr(L^{pq}\cdot\mathbf{1}).
\end{equation}

The objective of conditional MMD in Eq.~(\ref{eq:CMMD}) can be expressed as:
\begin{equation}\label{eq:CMMD_reformulation}
    \widehat{\text{CMMD}}^2 = \frac{1}{M^2}\tr(L^p\cdot C'_1) + \frac{1}{N^2}\tr(L^q\cdot C'_2) - \frac{2}{MN}\tr(L^{pq}\cdot C'_3),
\end{equation}
where $C'_1$, $C'_2$ and $C'_3$ are some matrices based on the conditional variables x in both data sets.

Despite the high similarity, we highlight two major differences between conditional CS divergence and conditional MMD.

\begin{itemize}
    \item ($C_1\neq C'_1$, $C_2\neq C'_2$ and $C_3\neq C'_3$). Conditional CS divergence is not simply computed by just putting a ``log” on each term of conditional MMD. That is, $C_1\neq C'_1$, $C_2\neq C'_2$ and $C_3\neq C'_3$. For simplicity, we only analyze in detail the difference between $C_1$ and $C'_1$.

    Let $D^p = \diag \left(K^p\cdot \mathbf{1}\right) \in \mathbb{R}^{M\times M}$, where $\diag(\cdot)$ refers to a diagonal matrix, i.e.,
    \begin{equation}
    D^p =
    \begin{pmatrix}
    \sum_{i=1}^M K_{1i}^p  & \cdots & 0 \\
    \vdots & \ddots & \vdots\\
    0 & \cdots & \sum_{i=1}^M K_{Mi}^p
    \end{pmatrix}.
    \end{equation}

    We have $C_1 = K^p (D^p)^{-2}$. By comparing Eq.~(\ref{eq:CMMD}) with Eq.~(\ref{eq:CMMD_reformulation}), we also have $C'_1 = (\tilde{K}^p)^{-1} K^p (\tilde{K}^p)^{-1}$, where $\tilde{K}^p = K^p + \lambda I$. Therefore,
    \begin{equation}
    \begin{split}
         C'_1 & = (\tilde{K}^p)^{-1} K^p (\tilde{K}^p)^{-1} \\
         & = (\tilde{K}^p)^{-1} K^p (D^p)^{-2} (D^p)^2 (\tilde{K}^p)^{-1} \\
         & = (\tilde{K}^p)^{-1} C_1 (D^p)^2 (\tilde{K}^p)^{-1} \neq C_1.
    \end{split}
    \end{equation}

    $C'_1$ involves matrix inverse of $\tilde{K}^p$ with computational complexity $\mathcal{O}(M^3)$. By contrast, $C_1$ only involves matrix inverse of a diagonal matrix $D^p$ with computational complexity $\mathcal{O}(M)$. Additionally, $C_1$ avoids the introduction of an additional hyperparameter $\lambda$, which is actually hard to tune in practice.
    The difference between $C_2$ (or $C_3$) and $C'_2$ (or $C'_3$) can be analyzed similarly.

    \item ($\tr(L^{pq}\cdot C_3) \neq \tr(L^{qp}\cdot C_4)$). The last two terms in conditional CS divergence are not the same.
\end{itemize}

\end{remark}

\subsection{Two special cases of conditional CS divergence}\label{sec:extensions}

We then discuss two special cases of the basic conditional CS divergence. Our purpose is to illustrate the flexibility and versatility offered by the definition of the conditional CS divergence, the elegance of its sample estimator (by just relying on the quadratic form and the inner products of samples), and its great potential to different downstream applications.

The last advantage does not hold with other divergence measures that have been discussed in Section~\ref{sec:existing_measures}. For example, if we stick to the decomposition rule in Eq.~(\ref{eq:KL}) (as in conditional KL divergence and conditional Bregman divergence), we implicitly assume that the dimension of $\mathbf{x}$ remains the same in $p$ and $q$, which may not hold in certain scenarios (e.g., comparing $p(\mathbf{y})$ with respect to $q(\mathbf{y}|\mathbf{x})$). On the other hand, there is still no universal agreement on a rigorous definition of the embedding of conditional distributions in RKHS (due to improper assumptions or difficulty of interpretation)~\cite{park2020measure}, which limits the usage of conditional MMD in a wider setting.


\subsubsection{$p(\mathbf{y}_1|\mathbf{x})$ with respect to $p(\mathbf{y}_2|\mathbf{x})$}

Our first case assumes that the variable $\mathbf{x}$ in both distributions remains the same and aims to quantify the divergence between $p(\mathbf{y}_1|\mathbf{x})$ and $p(\mathbf{y}_2|\mathbf{x})$, in which $\mathbf{y}_1$ and $\mathbf{y}_2$ are dependent. This case is common in supervised learning.

Recall a standard supervised learning paradigm, we have a training set $\mathcal{D}=\{\mathbf{x}_i,y_i\}_{i=1}^N$ of input feature $\mathbf{x}$ and desired response variable $y$. We assume that $\mathbf{x}_i$ and $y_i$ are sampled \emph{i.i.d.} from a true but unknown data distribution $p(\mathbf{x},y)=p(y|\mathbf{x})p(\mathbf{x})$. The high-level goal of supervised learning is to use the dataset $\mathcal{D}$ to learn a particular conditional distribution $q_\theta(\hat{y}|\mathbf{x})$ of the task outputs given the input features parameterized by $\theta$, which is a good approximation of $p(y|\mathbf{x})$, in which $\hat{y}$ refers to the predicted output.

If we measure the closeness between $p(y|\mathbf{x})$ and $q_\theta (\hat{y}|\mathbf{x})$ with the KL divergence, the learning objective becomes~\cite{rodriguez2019information}:
\begin{equation}\label{eq:KL_loss}
\begin{split}
\min D_{\text{KL}}(p(y|\mathbf{x});q_\theta (\hat{y}|\mathbf{x})) & = \min \mathbb{E}\left(-\log (q_\theta (\hat{y}|\mathbf{x}))\right) - H(y|\mathbf{x})
\\ & \Leftrightarrow \min \mathbb{E}\left(-\log (q_\theta (\hat{y}|\mathbf{x}))\right),
\end{split}
\end{equation}
where $H(y|\mathbf{x})$ only depends on $\mathcal{D}$ that is independent to the optimization over parameters $\theta$.


For regression, suppose $q_\theta (\hat{y}|\mathbf{x})$ is distributed normally $\mathcal{N}(h_\theta(\mathbf{x}),\sigma^2 I)$, and the network $h_\theta(\mathbf{x})$ gives the prediction of the mean of the Gaussian, the objective reduces to $\mathbb{E}\left(\|y-h_\theta(\mathbf{x})\|_2^2\right)$, which amounts to the mean squared error (MSE) loss\footnote{Note that, $\log (q_\theta (\hat{y}|\mathbf{x})) = \log \left( \frac{1}{\sqrt{2\pi}\sigma} \exp\left( -\frac{\|y-h_\theta(\mathbf{x})\|_2^2}{2\sigma^2} \right) \right) = -\log \sigma - \frac{1}{2}\log (2\pi) -  \frac{\|y-f_\theta(\mathbf{x})\|_2^2}{2\sigma^2} $.} and is empirically estimated by $\frac{1}{N}\sum_{i=1}^N (y_i-\hat{y}_i)^2$.

The CS divergence between $p(y|\mathbf{x})$ and $q_\theta(\hat{y}|\mathbf{x})$ is defined as:
\begin{equation}
\begin{split}
& D_{\text{CS}}(p(y|\mathbf{x});q_\theta(\hat{y}|\mathbf{x})) = - 2 \log \left(\int_\mathcal{X}\int_\mathcal{Y} p(y|\mathbf{x}) q_\theta(\hat{y}|\mathbf{x}) d\mathbf{x}dy \right) \\
& + \log \left(\int_\mathcal{X}\int_\mathcal{Y} p^2(y|\mathbf{x}) d\mathbf{x}dy\right) + \log \left(\int_\mathcal{X}\int_\mathcal{Y} q_\theta^2(\hat{y}|\mathbf{x}) d\mathbf{x}dy\right) \\
& = - 2 \log \left(\int_\mathcal{X}\int_\mathcal{Y} \frac{p(\mathbf{x},y)q_\theta(\mathbf{x},\hat{y})}{p^2(\mathbf{x})} d\mathbf{x}dy \right) \\
& + \log \left(\int_\mathcal{X}\int_\mathcal{Y} \frac{p^2(\mathbf{x},y)}{p^2(\mathbf{x})} d\mathbf{x}dy\right) + \log \left(\int_\mathcal{X}\int_\mathcal{Y} \frac{q_\theta^2(\mathbf{x},\hat{y})}{p^2(\mathbf{x})} d\mathbf{x}dy\right),
\end{split}
\end{equation}
which can be elegantly estimated as shown in Proposition~\ref{proposition_1} in a non-parametric way, without any parametric assumptions (e.g., Gaussian) on the underlying distribution $q_\theta(\hat{y}|\mathbf{x})$ as in the KL divergence case.

\begin{proposition}\label{proposition_1}
Assume that we are given observations $\{(\mathbf{x}_i,y_i,\hat{y}_i )\}_{i=1}^N$, where $\mathbf{x}\in \mathbb{R}^{d_\mathbf{x}}$ denotes an input variable with dimensionality $d_\mathbf{x}$, $y$ is the desired response, and $\hat{y}$ is the predicted output generated by a model $h_\theta$. Let $K$, $L^1$ and $L^2$ denote, respectively, the Gram matrices for the variables $\mathbf{x}$, $y$, and $\hat{y}$ (i.e., $K_{ij}=\kappa(\mathbf{x}_i,\mathbf{x}_j)$, $L_{ij}^1=\kappa(y_i,y_j)$ and $L_{ij}^2=\kappa(\hat{y}_i,\hat{y}_j)$). Further, let $L^{21}$ denote the Gram matrix between $\hat{y}$ and $y$ (i.e., $L_{ij}^{21}=\kappa(\hat{y}_i,y_j)$). The prediction term $D_{\text{CS}}(p(y|\mathbf{x});q_\theta(\hat{y}|\mathbf{x}))$ is given by:
\begin{equation}\label{eq:CS_ext1}
\small
\begin{split}
& \widehat{D}_{\text{CS}}(p(y|\mathbf{x});q_\theta(\hat{y}|\mathbf{x}))  = \log\left( \sum_{j=1}^N \left( \frac{ \sum_{i=1}^N K_{ji} L_{ji}^1 }{ (\sum_{i=1}^N K_{ji})^2 } \right) \right) \\
& + \log\left( \sum_{j=1}^N \left( \frac{ \sum_{i=1}^N K_{ji} L_{ji}^2 }{ (\sum_{i=1}^N K_{ji})^2 } \right) \right)
 - 2 \log \left( \sum_{j=1}^N \left( \frac{ \sum_{i=1}^N K_{ji} L_{ji}^{21} }{ (\sum_{i=1}^N K_{ji})^2 } \right) \right).
\end{split}
\end{equation}
\end{proposition}

To test the effectiveness of Eq.~(\ref{eq:CS_ext1}) as a valid loss function, we train neural networks for prediction purposes. The first data is the benchmark California housing\footnote{\url{https://scikit-learn.org/stable/modules/generated/sklearn.datasets.fetch_california_housing.html}.}, which consists of $20,640$ samples and $8$ features. We randomly select $70\%$ training samples and $30\%$ test samples. The task is to predict the median house value which has been rescaled between $0.15$ and $5$. The second data is the rotation MNIST\footnote{\url{https://de.mathworks.com/help/deeplearning/ug/train-a-convolutional-neural-network-for-regression.html}.}, in which the goal is to predict the rotation angles of handwritten digits. Specifically, $10,000$ samples were selected from MNIST dataset. Each sample was randomly rotated with a degree that is uniformly distributed between $-45^{\circ}$ and $45^{\circ}$. The training and test sets each contain $5,000$ images. For simplicity, we use fully-connected networks ($8-128-32-128-1$ for California housing and $784-256-196-36-1$ for rotation MNIST) and Sigmoid activation function. We choose SGD optimizer with learning rate $1e-3$ and mini-batch size $128$. We observe that the conditional CS divergence loss achieves slightly better prediction accuracy than the MSE loss, as shown in Fig.~\ref{fig:loss}.

In case of multi-output regression, also known as multi-task regression~\cite{chen2010graph} (i.e., when there are more than two target variables to be predicted), some tasks are often more closely related and more likely to share common relevant covariates than others. Thus, it is necessary to take into account the complex correlation structure in the outputs for a more effective multi-task learning~\cite{chen2010graph,yu2020measuring}. It is reasonable to measure relatedness between the $i$-th regression task and the $j$-th regression task using the conditional CS divergence between $p(y_i|\mathbf{x})$ and $p(y_j|\mathbf{x})$. This proposal is empirically justified in our Section~\ref{sec:simulation2}.

\begin{figure}[t]
	\centering
	\subfigure[Rotation MNIST]{
		\centering
		\includegraphics[width=.4\linewidth]{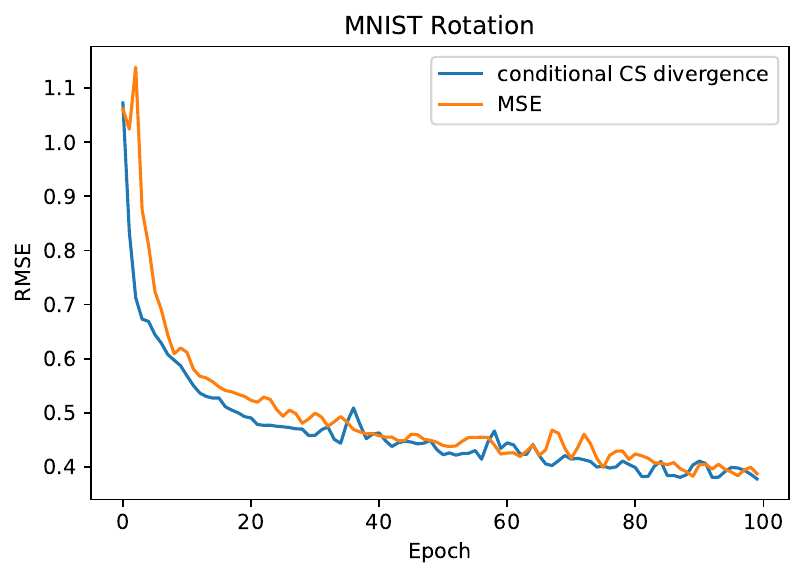}
	}
\hfill
	\subfigure[California Housing]{
		\centering
		\includegraphics[width=.4\linewidth]{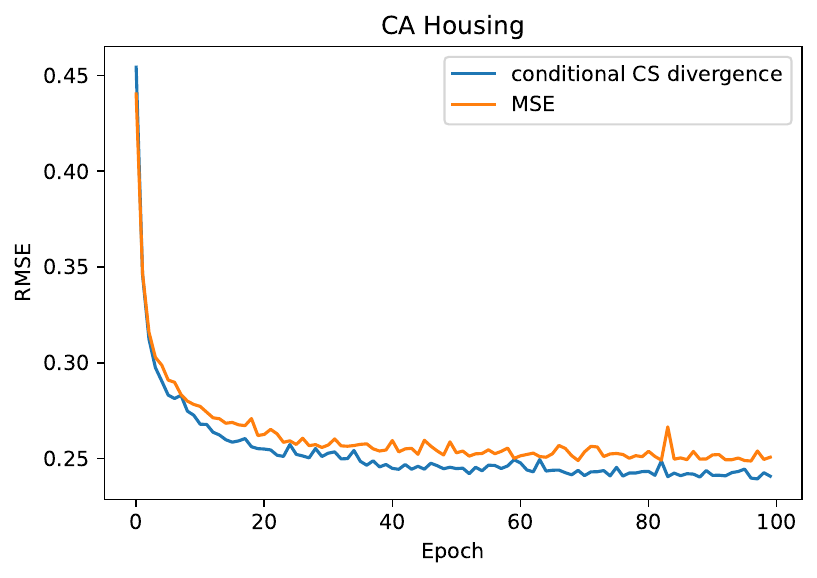}
	}
	\caption{The root mean square error (RMSE) of the regression network trained with MSE loss and conditional CS divergence loss on test data in each epoch.}
	\label{fig:loss}
\end{figure}

\subsubsection{$p(\mathbf{y}|\mathbf{x}_1)$ with respect to $p(\mathbf{y}|\{\mathbf{x}_1,\mathbf{x}_2\})$} \label{sec:3_2_2}

Our second case assumes that the variable $\mathbf{y}$ remains the same and aims to quantify the divergence between $p(\mathbf{y}|\mathbf{x}_1)$ and $p(\mathbf{y}|\{\mathbf{x}_1,\mathbf{x}_2\})$, i.e., there is a third variable $\mathbf{x}_2$ that influences the distribution of $\mathbf{y}$.
From a probabilistic perspective, $p(\mathbf{y}|\mathbf{x}_1)=p(\mathbf{y}|\{\mathbf{x}_1,\mathbf{x}_2\})$ implies that $\mathbf{y}$ is conditionally independent to $\mathbf{x}_2$ given $\mathbf{x}_1$, written symbolically as $\mathbf{y} \ind \mathbf{x}_2 | \mathbf{x}_1$. Hence, the divergence between $p(\mathbf{y}|\mathbf{x}_1)$ and $p(\mathbf{y}|\{\mathbf{x}_1,\mathbf{x}_2\})$ is also a good indicator on the degree of conditional independence, which is considerably difficult to measure.

Let us denote $\vec{\mathbf{x}}=[\mathbf{x}_1;\mathbf{x}_2] \in \mathcal{X}$, i.e., the concatenation of $\mathbf{x}_1$ and $\mathbf{x}_2$, the CS divergence for $p(\mathbf{y}|\mathbf{x}_1)$ and $p(\mathbf{y}|\{\mathbf{x}_1,\mathbf{x}_2\})$ can be expressed as:
\begin{equation}
\small
\begin{split}
& D_{\text{CS}}(p(\mathbf{y}|\mathbf{x}_1);p(\mathbf{y}|\{\mathbf{x}_1,\mathbf{x}_2\})  = - 2 \log \left(\int_\mathcal{X}\int_\mathcal{Y} p(\mathbf{y}|\mathbf{x}_1) p(\mathbf{y}|\vec{\mathbf{x}}) d\vec{\mathbf{x}}d\mathbf{y} \right) \\
& + \log \left(\int_\mathcal{X}\int_\mathcal{Y} p^2(\mathbf{y}|\mathbf{x}_1) d\vec{\mathbf{x}}d\mathbf{y}\right) + \log \left(\int_\mathcal{X}\int_\mathcal{Y} p^2(\mathbf{y}|\vec{\mathbf{x}} ) d\vec{\mathbf{x}}d\mathbf{y}\right) \\
& = - 2 \log \left(\int_\mathcal{X}\int_\mathcal{Y} \frac{p(\mathbf{x}_1,\mathbf{y})p(\vec{\mathbf{x}},\mathbf{y})}{p(\mathbf{x}_1)p(\vec{\mathbf{x}})} d\vec{\mathbf{x}}d\mathbf{y} \right) \\
& + \log \left(\int_\mathcal{X}\int_\mathcal{Y} \frac{p^2(\mathbf{x}_1,\mathbf{y})}{p^2(\mathbf{x}_1)} d\vec{\mathbf{x}}d\mathbf{y}\right) + \log \left(\int_\mathcal{X}\int_\mathcal{Y} \frac{p^2(\vec{\mathbf{x}},\mathbf{y})}{p^2(\vec{\mathbf{x}})} d\vec{\mathbf{x}}d\mathbf{y}\right),
\end{split}
\end{equation}
which can be efficiently estimated from samples as shown in Proposition~\ref{proposition_2}.

\begin{proposition}\label{proposition_2}
Assume that we are given observations $\psi=\{(\mathbf{x}_i^1,\mathbf{x}_i^2,\mathbf{y}_i)\}_{i=1}^N$, where $\mathbf{x}^1\in \mathbb{R}^{d_1}$, $\mathbf{x}^2\in \mathbb{R}^{d_2}$ and $\mathbf{y}\in \mathbb{R}^{d_\mathbf{y}}$. Let $K^1$, $K^{12}$ and $L$ denote the Gram matrices for the variable $\mathbf{x}^1$, the concatenation of variables $\{\mathbf{x}^1,\mathbf{x}^2\}$, and the variable $\mathbf{y}$, respectively. That is, $(K^{12})_{ji} = \kappa\left(\begin{bmatrix} \mathbf{x}_j^1 \\ \mathbf{x}_j^2 \end{bmatrix} - \begin{bmatrix} \mathbf{x}_i^1 \\ \mathbf{x}_i^2 \end{bmatrix} \right) = \kappa (\mathbf{x}_j^1 - \mathbf{x}_i^1) \kappa (\mathbf{x}_j^2 - \mathbf{x}_i^2)$. The empirical estimation of $D_{\text{CS}}(p(\mathbf{y}|\mathbf{x}_1);p(\mathbf{y}|\{\mathbf{x}_1,\mathbf{x}_2\})$ is given by:
\begin{equation}\label{eq:CS_ext2}
\small
\begin{split}
& D_{\text{CS}}(p(\mathbf{y}|\mathbf{x}_1);p(\mathbf{y}|\{\mathbf{x}_1,\mathbf{x}_2\}) \approx \log\left( \sum_{j=1}^N \left( \frac{ \sum_{i=1}^N K_{ji}^1 L_{ji} }{ (\sum_{i=1}^N K_{ji}^1)^2 } \right) \right) \\
& + \log\left( \sum_{j=1}^N \left( \frac{ \sum_{i=1}^N K_{ji}^{12} L_{ji} }{ (\sum_{i=1}^N K_{ji}^{12})^2 } \right) \right) \\
& - 2 \log \left( \sum_{j=1}^N \left( \frac{ \sum_{i=1}^N K_{ji}^1 L_{ji} }{ (\sum_{i=1}^N K_{ji}^1)(\sum_{i=1}^N K_{ji}^{12}) } \right) \right).
\end{split}
\end{equation}
\end{proposition}

There are vast AI applications that may benefit from an efficient sample estimator to the conditional independence~\cite{pogodin2023efficient}. Taking representation learning as an example, we aim to learn a representation function $\varphi$ for the features $\mathbf{x}$, such that our predictions $\hat{y}$ are ``invariant" to some metadata $\mathbf{z}$. If $\mathbf{z}$ refers to some protected attribute(s) such as race or gender, the \emph{Equalized Odds} condition~\cite{hardt2016equality} requires the conditional independence between prediction and protected attribute(s) given ground truth of the target, i.e., $\hat{y} \ind \mathbf{z} | y$ or equivalently $p(\hat{y}|y)=p(\hat{y}|\mathbf{z},y)$. On the other hand, if $\mathbf{z}$ is the environment index in which the data was collected, the condition $ y \ind \mathbf{z} | \varphi(\mathbf{x})$ or equivalently $p(y|\varphi(\mathbf{x}))=p(y|\mathbf{z},\varphi(\mathbf{x}))$ is commonly used as a target for invariant learning in domain generalization~\cite{li2022invariant}.

Since the attention of this paper is on dynamic data, we demonstrate the implication of Eq.~(\ref{eq:CS_ext2}) on time series causal discovery~\cite{assaad2022survey}: given two time series $\{x_t\}$ and $\{y_t\}$, determine the true causal direction between them, i.e., does $\{x_t\}$ cause $\{y_t\}$, or does $\{y_t\}$ cause $\{x_t\}$, which is also known as bivariate \emph{causal direction identification}~\cite{mooij2016distinguishing}.

According to Granger~\cite{granger1969investigating,granger1980testing} (the $2003$ Nobel Prize laureate in Economics), a time series (or process) $\{x_t\}$ causes another time series (or process) $\{y_t\}$ if the past of $\{x_t\}$ has unique information about the future of $\{y_t\}$. Essentially, Granger proposed to test the following hypothesis for identification of a causal effect of $\{x_t\}$ on $\{y_t\}$~\cite{su2008nonparametric}:
\begin{equation}
\left\{
\begin{array}{lr}
\mathcal{H}_0: p(y_{t+1}|\mathbf{y}_t^n) = p(y_{t+1}|\mathbf{y}_t^n,\mathbf{x}_t^m), \\
\quad\quad\quad \text{$\{\mathbf{x}_t\}$ is not the cause of $\{\mathbf{y}_t\}$}  \\
\mathcal{H}_1: p(y_{t+1}|\mathbf{y}_t^n) \neq p(y_{t+1}|\mathbf{y}_t^n,\mathbf{x}_t^m), \\
\quad\quad\quad \text{$\{\mathbf{x}_t\}$ is the cause of $\{\mathbf{y}_t\}$}
\end{array}
\right.
\end{equation}
where $y_{t+1}$ refers to the future observation of $\{y_t\}$. $\mathbf{x}_t^m =  [x_{t},x_{t-\tau},\cdots,x_{t-(m-1)\tau}]$ denotes the past observation (or reconstructed state-space vector) of $\{x_t\}$, in which $\tau$ is the time delay, $m$ is the embedding dimension. $\mathbf{y}_t^n =  [y_{t},y_{t-\tau},\cdots,y_{t-(n-1)\tau}]$ denotes the past observation of $\{y_t\}$ with embedding dimension $n$.


From an information-theoretic perspective, one can directly evaluate the closeness between $p(y_{t+1}|\mathbf{y}_t^n)$ and $p(y_{t+1}|\mathbf{y}_t^n,\mathbf{x}_t^m)$ to perform the above test. If we use the expected KL divergence, we get:
\begin{equation}\label{eq:transfer_entropy}
\begin{split}
& \mathbb{E}\left( \log\left(\frac{p(y_{t+1}|\mathbf{y}_t^n,\mathbf{x}_t^m)}{p(y_{t+1}|\mathbf{y}_t^n)} \right) \right) \\
& = \int\int\int p(y_{t+1},\mathbf{y}_t^n,\mathbf{x}_t^m) \log\left(\frac{p(y_{t+1}|\mathbf{y}_t^n,\mathbf{x}_t^m)}{p(y_{t+1}|\mathbf{y}_t^n)} \right) dy_t d\mathbf{y}_t^n d\mathbf{x}_t^m \\
& = - \mathbb{E}\left( \log p(\mathbf{y}_t^n,\mathbf{x}_t^m) \right) + \mathbb{E}\left( \log p(y_{t+1},\mathbf{y}_t^n,\mathbf{x}_t^m) \right) \\
& \quad - \mathbb{E}\left( \log p(y_{t+1},\mathbf{y}_t^n) \right) + \mathbb{E}\left( \log p(\mathbf{y}_t^n) \right) \\
& = H(\mathbf{y}_t^n,\mathbf{x}_t^m) - H(y_{t+1},\mathbf{y}_t^n,\mathbf{x}_t^m)+ H(y_{t+1},\mathbf{y}_t^n) - H(\mathbf{y}_t^n),
\end{split}
\end{equation}
which is also known as the transfer entropy (TE)~\cite{schreiber2000measuring}.


An alternative choice is our conditional CS divergence as shown in Eq.~(\ref{eq:CS_ext2}), i.e.,
 $D_{\text{CS}} (p(y_{t+1}|\mathbf{y}_t^n);p(y_{t+1}|\mathbf{y}_t^n,\mathbf{x}_t^m))$, by simply taking $\mathbf{y}=y_{t+1}$, $\mathbf{x}^1=\mathbf{y}_t^n$, and $\mathbf{x}^2=\mathbf{x}_t^m$. Hence, we can define a causal score for direction $\mathbf{x} \rightarrow \mathbf{y}$ by:
\begin{equation}
    C_{\mathbf{x} \rightarrow \mathbf{y}} = D_{\text{CS}} (p(y_{t+1}|\mathbf{y}_t^n);p(y_{t+1}|\mathbf{y}_t^n,\mathbf{x}_t^m)).
\end{equation}

A causal direction $\mathbf{x} \rightarrow \mathbf{y}$ is confirmed if $C_{\mathbf{x} \rightarrow \mathbf{y}}$ is significant. On the other hand, the inverse direction $\mathbf{y} \rightarrow \mathbf{x}$ is confirmed if $C_{\mathbf{y} \rightarrow \mathbf{x}}$, defined as $D_{\text{CS}} (p(x_{t+1}|\mathbf{x}_t^m);p(x_{t+1}|\mathbf{x}_t^m,\mathbf{y}_t^n))$, is significant.

To demonstrate the effectiveness of our causal score, we test its performances on two benchmark simulations: the $5$ coupled H\'enon chaotic maps in~\cite{kugiumtzis2013direct} with the true causal relation $x_{i-1}\rightarrow x_i$, and the NLVAR3 model in~\cite{gourevitch2006linear} which is a nonlinear vector autoregression (VAR) process of order $2$ with $3$ variables. We also compare our measure of causality with the classic linear Granger causality test~\cite{granger1969investigating}, the popular kernel Granger causality (KGC)~\cite{marinazzo2008kernel} (a generalization of linear Granger causality to nonlinear case by kernel method), and TE with $k$NN estimator~\cite{zhu2015contribution}.
We generate $1,024$ samples for each model and determine all pairwise causal directions by the corresponding causal score coupled with a significance test. Equations of these two models, details about the implementation regarding different competing methods and the significance test all can be found in Appendix~\ref{sec:causal_setup}.

As can be seen from Fig.~\ref{fig:causal}, the linear Granger causality fails in nonlinear data, whereas our causal score with CS divergence $(\text{CS})^{2}$, the popular KGC and TE can precisely detect all pairwise causal directions. Compared with KGC, our measure is much more computationally efficient; compared with TE with $k$NN estimator, our measure is differentiable and easy-to-implement (it only requires evaluation of three Gram matrices), which facilitates more potential usages as demonstrated in the next section.


\begin{figure}[t]
	\centering
	\subfigure[H\'enon chaotic maps]{
		\centering
		\includegraphics[width=.9\linewidth]{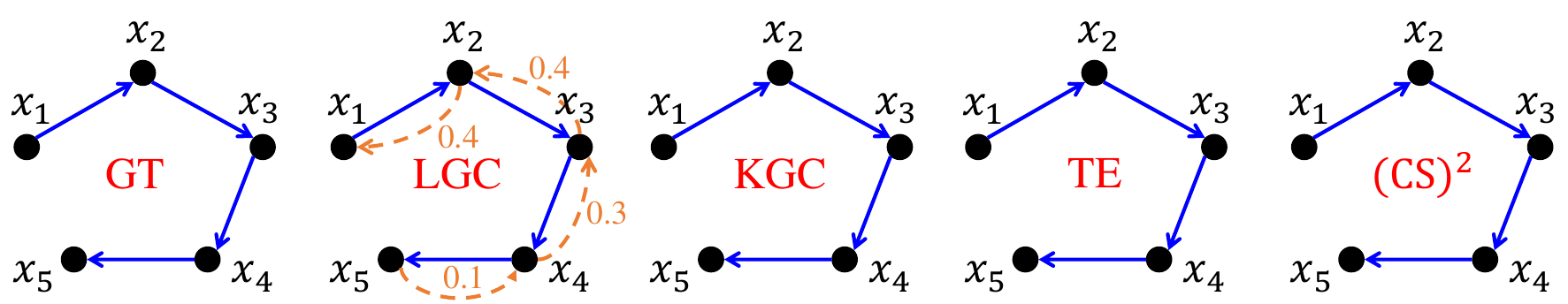}
	} \\
	\subfigure[NLVAR3]{
		\centering
		\includegraphics[width=.9\linewidth]{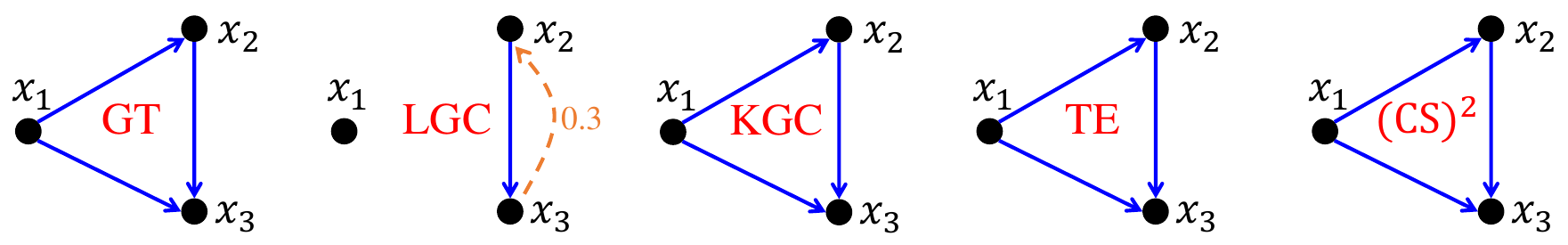}
	}
	\caption{The ground truth (GT) causal graph and that was identified by linear Granger causality (LGC), kernel Granger causality (KGC), transfer entropy (TE) with $k$NN estimator, and our causal score with CS divergence $(\text{CS})^{2}$. The blue solid line represents the detected bivariate causal direction (after the significance test). The orange dashed curve represents the anti-causal direction that could be incorrectly detected (i.e., a false positive). The ratio behind the curve is the possibility of a false positive over $10$ independent trials.}
	\label{fig:causal}
\end{figure}

\section{Numerical Simulations on Synthetic Data}

We carry out two numerical simulations on synthetic data to demonstrate the behaviors (especially the statistical power) of our conditional CS divergence with respect to previous proposals mentioned in Table~\ref{tab:property}. Both simulations are designed to examine, both qualitatively and quantitatively, the effectiveness of our divergence in distinguishing between two different conditional distributions.

\begin{table*}\centering
\scriptsize
\caption{Power test for conditional CS divergence, conditional KL divergence (with $k$-NN graph estimator), conditional Bregman divergence (operated on covariance matrix $C$), and conditional MMD.}
\begin{tabular}{@{}rrrrrrrrrrrrrrrrrrrrrrrr@{}}
\toprule
& \multicolumn{5}{c}{Conditional CS} & & \multicolumn{5}{c}{Conditional KL}  & & \multicolumn{5}{c}{von Neumann ($C$)} & & \multicolumn{5}{c}{Conditional MMD}\\
\cmidrule{2-6} \cmidrule{8-12} \cmidrule{14-18} \cmidrule{20-24}
& (a) & (b) & (c) & (d) & (e) && (a) & (b) & (c) & (d) & (e) && (a) & (b) & (c) & (d) & (e) && (a) & (b) & (c) & (d) & (e) \\
\midrule
(a) & 0.05 & 1 & 1 & 1 & 1 && 0.03 & 1 & 1 & 1 & 1 && 0.03 & 0.02 & 0.86 & 1 & 1 && 0.06 & 0 & 0 & 0 & 0 \\
(b) & 1 & 0.05 & 1 & 1 & 1 && 1 & 0.05 & 0.98 & 1 & 1 && 0.04 & 0.09 & 0.87 & 1 & 1 && 0 & 0.07 & 0 & 0 & 0 \\
(c) & 1 & 1 & 0.05 & 1 & 1 && 0.99 & 0.99 & 0.06 & 1 & 1 && 0.87 & 0.88 & 0.06 & 1 & 1 && 0 & 0 & 0.02 & 0 & 0 \\
(d) & 1 & 1 & 1 & 0.08 & 0.92 && 1 & 1 & 1 & 0.03 & 0.79 && 1 & 1 & 1 & 0.07 & 0.11 && 0 & 0 & 0 & 0.04 & 0 \\
(e) & 1 & 1 & 1 & 0.91 & 0.10 && 1 & 1 & 1 & 0.79 & 0.04 && 1 & 1 & 1 & 0.14 & 0.04 && 0 & 0 & 0 & 0 & 0.10 \\
\bottomrule
\end{tabular}
\label{tab:power_test}
\end{table*}

\begin{figure*}[t]
	\centering
	\subfigure[Ground Truth]{
		\centering
		\includegraphics[width=.18\linewidth]{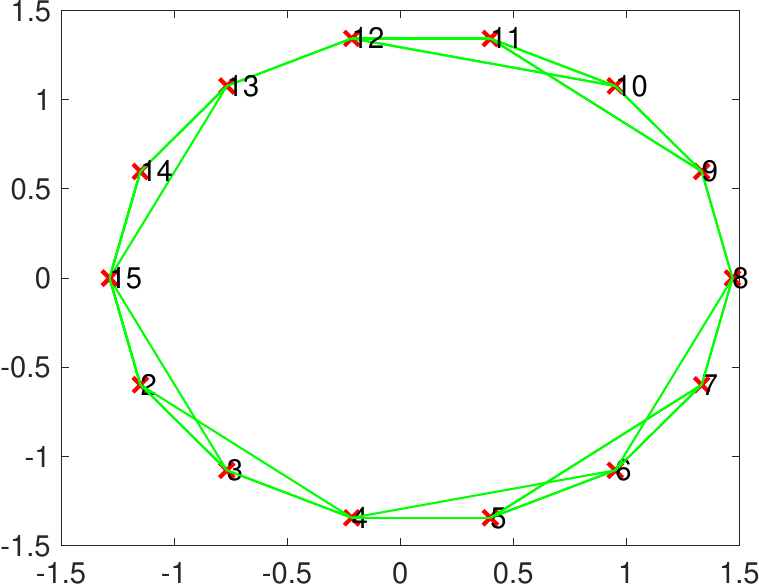}
	}
	\subfigure[Conditional CS]{
		\centering
		\includegraphics[width=.18\linewidth]{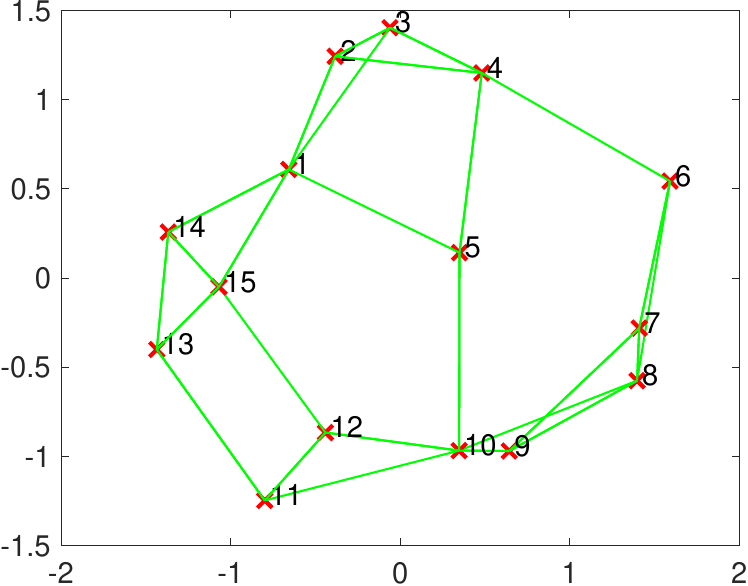}
	}
	\subfigure[Conditional KL]{
		\centering
		\includegraphics[width=.18\linewidth]{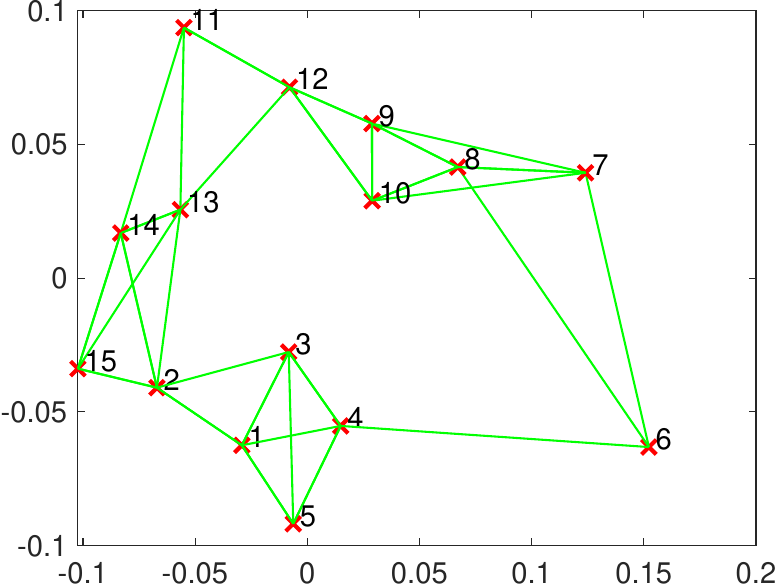}
	}
 	\subfigure[Conditional von Neumann]{
		\centering
		\includegraphics[width=.18\linewidth]{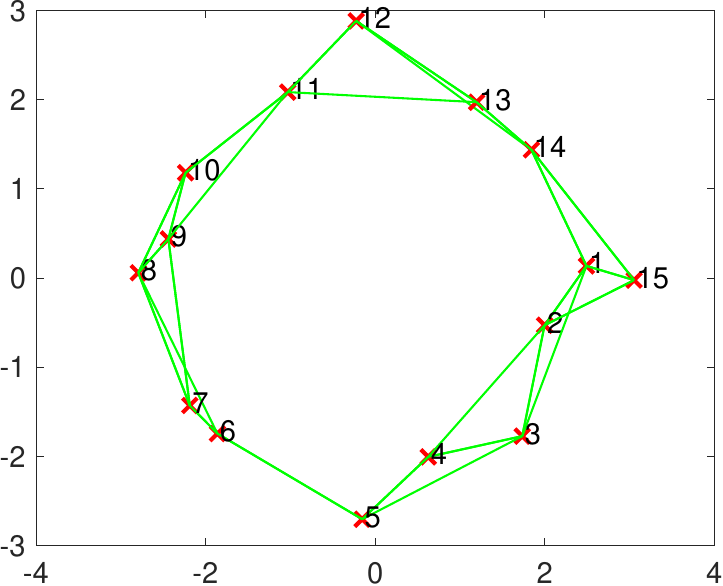}
	}
 	\subfigure[Conditional MMD]{
		\centering
		\includegraphics[width=.18\linewidth]{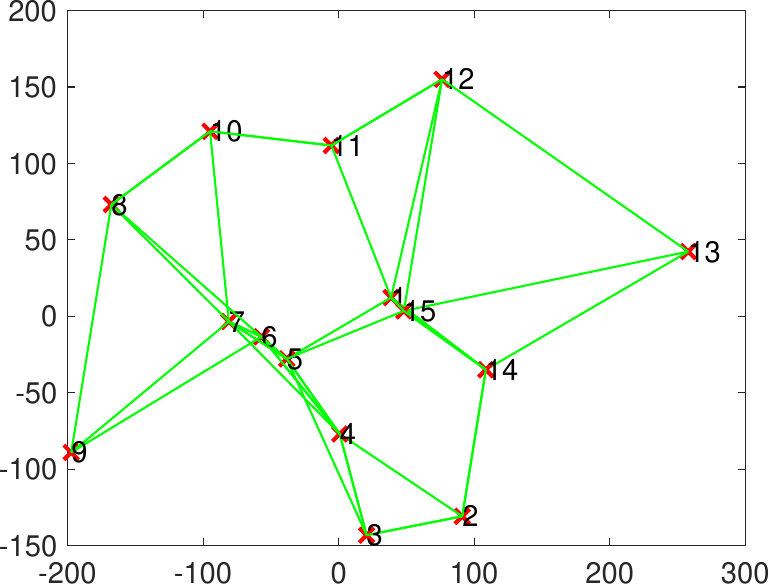}
	} \\
	\subfigure[Ground Truth]{
		\centering
		\includegraphics[width=.18\linewidth]{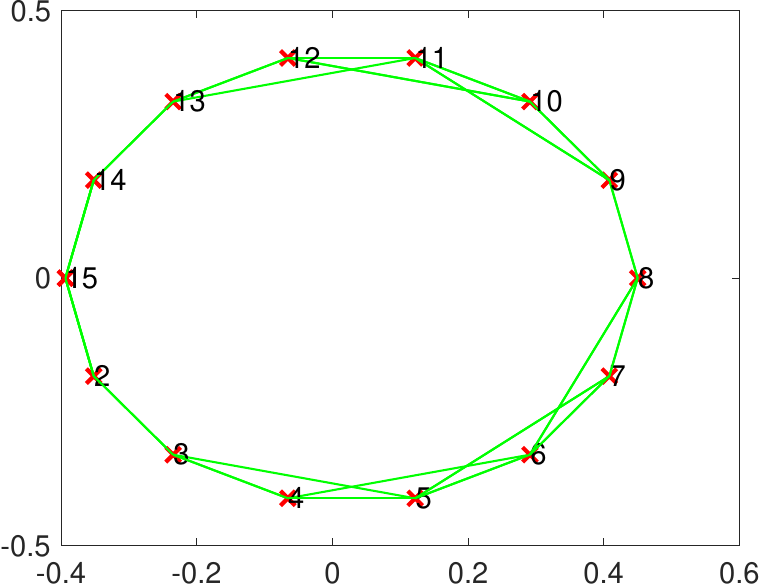}
	}
 	\subfigure[Conditional CS]{
		\centering
		\includegraphics[width=.18\linewidth]{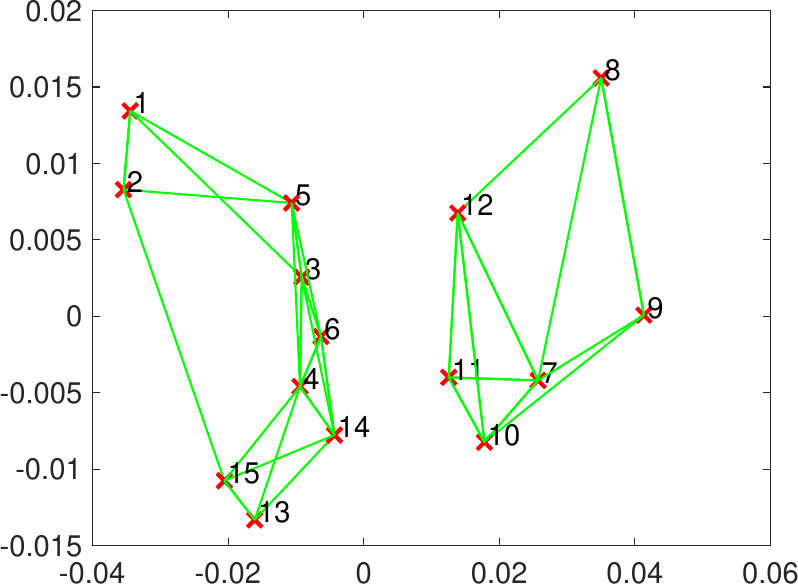}
	}
	\subfigure[Conditional KL]{
		\centering
		\includegraphics[width=.18\linewidth]{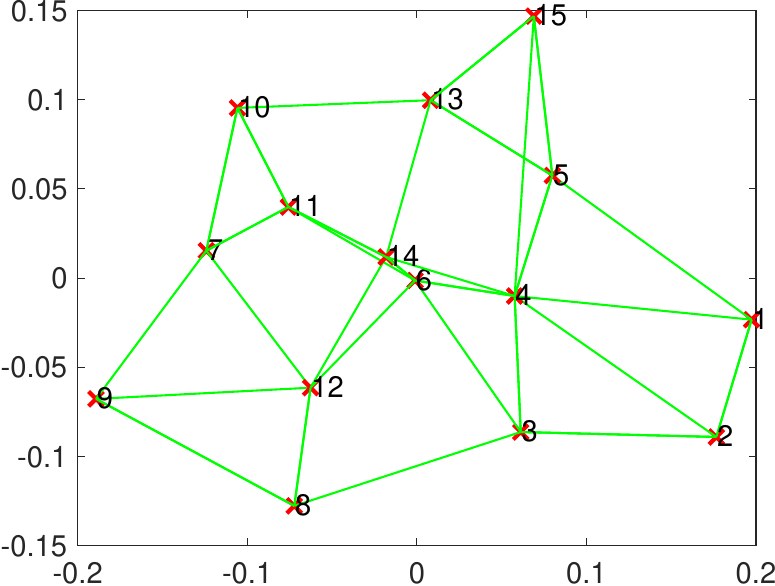}
	}
 	\subfigure[Conditional von Neumann]{
		\centering
		\includegraphics[width=.18\linewidth]{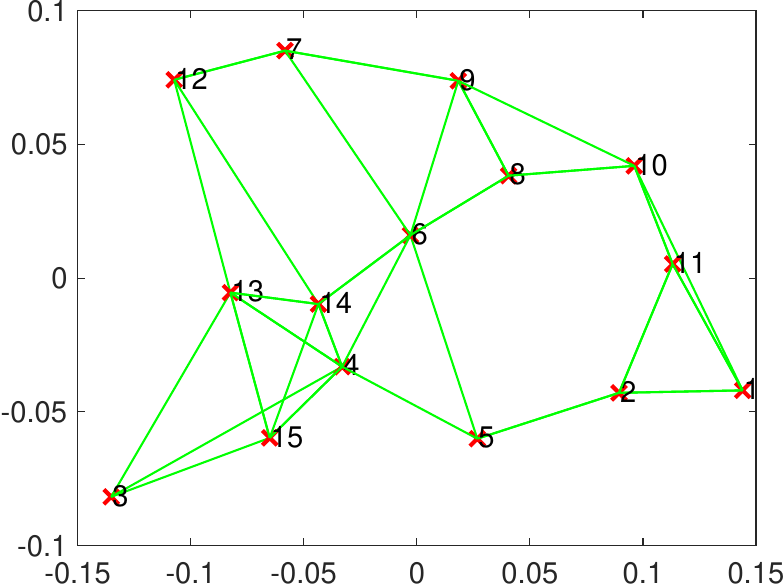}
	}
 	\subfigure[Conditional MMD]{
		\centering
		\includegraphics[width=.18\linewidth]{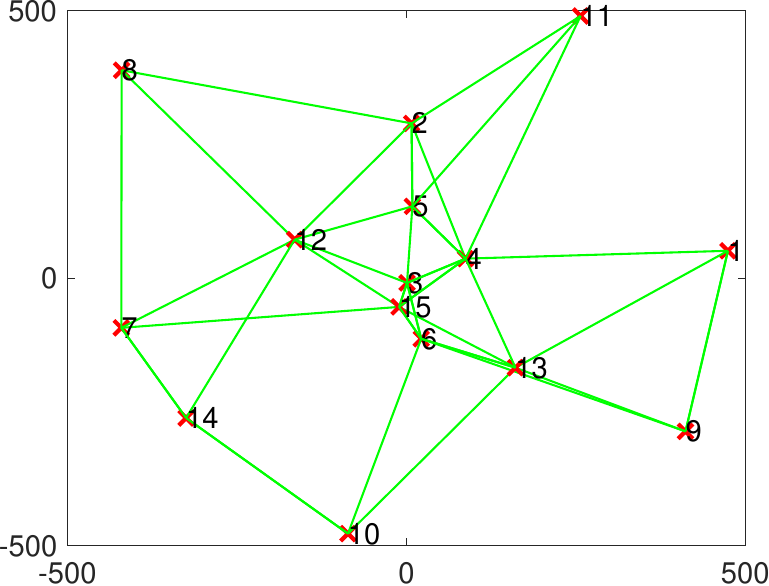}
	}
	\caption{The ground truth task structure (first column) and that is learned by the conditional CS divergence (second column); the conditional KL divergence (third column); the conditional von Neumann divergence (fourth column); and the conditional MMD (fifth column) when the input variable $\mathbf{x}$ is Gaussian distributed (first row) and uniformly distributed (second row), respectively. We connect each task with its $3$ nearest tasks.}
	\label{fig:MTL_structure}
\end{figure*}

\subsection{Simulation I} \label{sec:simulation1}
Motivated by previous literature on two-sample conditional distribution test~\cite{zheng2000consistent}, we generate $5$ sets of data that have distinct conditional distributions. Specifically, in set (a), the dependent variable $y$ is generated by $y = 1 + \sum_{i=1}^p x_i + \epsilon$, where $p$ refers to the dimension of explanatory variable $\mathbf{x}$, $\epsilon$ denotes standard normal distribution. In set (b), $y = 4 + \sum_{i=1}^p x_i + \epsilon$, i.e., there is a mean (or $1$st order moment) shift with respect to set (a).
In set (c), $y = 1 + \sum_{i=1}^p x_i + \psi$, where $\psi$ denotes a Logistic distribution with the location parameter $\mu=1$ and the scale parameter $s=1$. In set (d), $y = 1 + \sum_{i=1}^p x_i^2 + \epsilon$. In set (e), $y = 1 + \sum_{i=1}^p x_i^2 + \psi$. For each set, the input distribution $p(\mathbf{x})$ is an isotropic Gaussian, but the conditional distribution $p(y|\mathbf{x})$ differs from each other.
For example, in set (a), $p(y|\mathbf{x})\sim \mathcal{N}(y-\sum_{i=1}^p x_i -1,1)$, whereas in set (e), $p(y|\mathbf{x})\sim \textit{Logistic}(y-\sum_{i=1}^p x_i^2 -2,1)$.


To evaluate the statistical power of each conditional divergence measure to distinguish any two sets of data, we randomly simulate $500$ samples from each set, each with input dimension $p=10$. We apply a non-parametric permutation test with the number of permutations $P=500$ and significance level $\eta=0.05$ to test if one measure can distinguish these two sets\footnote{Please refer to Appendix~\ref{sec:permutation_test} on the details of the permutation test.}. We repeat this procedure $100$ independent times and define the statistical power as the percentage of successful trials, where success refers to the tested measure's ability to distinguish between the two sets. Table~\ref{tab:power_test} summarizes the power test results. Specifically, the $(i,j)$-th element of each matrix reports the quantitative statistical power of a specific measure to distinguish set $(i)$ from set $(j)$ via the following hypothesis test:
\begin{equation}\label{eq:hypothesis_test}
\left\{
\begin{array}{lr}
\mathcal{H}_0: p_i(y|\mathbf{x}) = p_j(y|\mathbf{x}),   \\
\mathcal{H}_1: p_i(y|\mathbf{x}) \neq p_j(y|\mathbf{x}).
\end{array}
\right.
\end{equation}
Given our chosen significance level of $\eta = 0.05$, one would ideally expect the main diagonal elements to be close to $0.05$ and the off-diagonal elements to be $1$.



As can be seen, the conditional CS divergence is much more powerful to distinguish two different conditional distributions (especially the ability to distinguish set (c) from set (a), and set (d) from set (e)), although it also has few false alarms along the main diagonal.

If we look deeper, the conditional von Neumann divergence on covariance matrix $C$ has nearly zero power to distinguish set (a) from set (b), in which there is only a mean shift. This is because the covariance matrix only encodes the 2nd moment information such that it is insufficient to distinguish two conditional distributions if the distributional shift comes from 1st or higher-order moments. This also indicates that the conditional von Neumann divergence, when used with the sample covariance matrix, does not guarantee ``faithfulness", which is a big limitation. On the other hand, it is surprising to find that the conditional MMD fails to identify the distinctions in these tests. Note that, this result does not mean that the conditional MMD is incapable of quantifying the conditional discrepancy. It simply suggests that a non-parametric permutation test is not a reliable way to characterize the tail probability of the extremal value of conditional MMD statistics. Hence, a more computationally efficient way is required such that one can determine the optimal detection threshold. However, to the best of our knowledge, this is still an open problem.

\subsection{Simulation II}\label{sec:simulation2}
Motivated by the multi-task learning literature~\cite{flamary2014learning}, we construct a synthetic data set with $15$ related regression tasks, each with an input dimension of $20$.
For each task, the input variable $\mathbf{x}_t$ is generated $i.i.d.$ from the same distribution. The corresponding output is generated as $y_t =\mathbf{w}_t^T\mathbf{x}_t+\epsilon$, where $\mathbf{w}_t\in \mathbb{R}^{20}$ is the regression coefficients or weights of the $t$-th task, $\epsilon\sim\mathcal{N}(0,1)$ is the independent noise. Because different tasks have the same input distribution $p(\mathbf{x})$, their relatedness is mainly manifested by the conditional distribution $p(y|\mathbf{x})$, which is also parameterized by $\mathbf{w}$.

In our data, each task is mostly related to its neighboring tasks to manifest strong locality relationships. Specifically, the weight in the $1$st task is generated by $\mathbf{w}_1\sim\mathcal{N}(\mathbf{0},\mathbf{I}_{20})$,
The weights from the $2$nd task to the $15$th task (i.e., $\mathbf{w}_2$ to $\mathbf{w}_{15}$) share the same regression coefficients with $\mathbf{w}_1$ on dimensions $3$ to $20$. However, the first two dimensions of $\mathbf{w}_2$ to $\mathbf{w}_{15}$ are generated by applying a rotation matrix of the form $R=\left[\begin{matrix}\cos(\theta)&-\sin(\theta)\\\sin(\theta)&\cos(\theta)\\\end{matrix}\right]$ to the first two dimensions of $\mathbf{w}_1$, in which $\theta$ is evenly spaced between $[0,2\pi]$. This way, $\mathbf{w}_{15}$ gets back to $\mathbf{w}_1$ and the task $i$ is mostly related to task $(i-1)$ and task $(i+1)$. In other words, the relatedness amongst these $15$ tasks forms a circular structure.

We simulate $200$ samples from each task and apply the four different conditional divergence measures to quantify the closeness between pairs of tasks, i.e., the discrepancy from task $i$ to task $j$ is quantified by $D(p_i (y|\mathbf{x});p_j (y|\mathbf{x}))$. For each measure, we can obtain a $15\times 15$ matrix that encodes all pairwise discrepancies. We then apply the multidimensional scaling (MDS) to project the obtained discrepancy matrix into a $2$-dimensional space to visualize the (dis)similarity between individual tasks.
We consider two types of input distributions: $\mathbf{x}_t$ follows an isotropic multivariate Gaussian distribution (i.e., $\mathbf{x}_t\sim\mathcal{N}(\mathbf{0},\mathbf{I}_{20})$) and each element of $\mathbf{x}_t$ is drawn from a uniform distribution $\mathcal{U}(0,1)$.

As can be seen from Fig.~\ref{fig:MTL_structure}, when input data is Gaussian distributed, our conditional CS divergence (excluding task $5$), the conditional KL divergence, and the conditional von Neumann divergence are capable of identifying a roughly circular structure across all tasks. The result of conditional von Neumann divergence is the closest to a standard circle. This is because the data in each task is Gaussian distributed and does not contain mean shift, such that the $2$nd moment information in covariance matrix $C$ is sufficient to distinguish two distributions. The conditional MMD can discover precisely the locality relationships (e.g., tasks $2$, $3$, and $4$ are closely related and tasks $1$, $14$, and $15$ are closely related). However, it is hard to identify the global circular structure.
When input data is uniformly distributed, the performance of the von Neumann divergence drops significantly, whereas our conditional CS divergence still identifies that tasks $7-12$ are closely related, forming a small circle, and tasks $13-15$, $1-6$ are closely related, forming another small circle. It should be noted that our measure only missed capturing the locality relationship between tasks $6$ and $7$. In contrast, both conditional KL divergence and conditional MMD revealed a few spurious relationships. For example, task $6$ with respect to tasks $12$ and $14$ in Fig.~\ref{fig:MTL_structure}(h) and task $6$ with respect to tasks $3,13,15$ in Fig.~\ref{fig:MTL_structure}(j).

The quantitative evaluation in Table~\ref{tab:structure_results} is consistent with the visualization results, where we utilize two measures to quantify the closeness between the estimated graph structure and the ground truth. The first measure involves calculating the percentage $p$ based on the difference between the estimated graph adjacency matrix $A_S$ and the true adjacency matrix $A_G$. The second measure is the geodesic distance $d_{\vec{x}}(L_G,L_S)$~\cite{bravo2019unifying}:
\begin{equation}
    d_{\vec{x}}(L_G,L_S) = \mathrm{arccosh}\left(1+\frac{\|(L_G-L_S)\vec{x}\|_2^2\|\vec{x}\|_2^2}{2(\vec{x}^T L_G \vec{x})(\vec{x}^T L_S \vec{x})}\right),
\end{equation}
in which $L_G$ is the ground truth graph Laplacian, $L_S$ is the estimated graph Laplacian, and we select $\vec{x}$ to be the smallest non-trivial eigenvector of $L_G$ which encodes the global structure of a graph. For both measures, a smaller value indicates better performance.

\begin{table}[]
\centering
\caption{Quantitative evaluation on task structure discovery for Gaussian input (left of $/$) and uniformly distributed input (right of $/$).}
\begin{tabular}{|c|c|c|c|c|}
\hline
Methods    & cond. CS & cond. KL & cond. vN & cond. MMD \\ \hline
$p(A_G\neq A_S)$   & 0.196/$\mathbf{0.187}$ & 0.204/0.293  & $\mathbf{0.062}$/0.293   & 0.231/0.418  \\ \hline
$d_{\vec{x}}(L_G,L_S)$  & 1.785/$\mathbf{2.607}$ & 2.246/2.696  & $\mathbf{1.494}$/2.658  & 2.872/3.273    \\ \hline
\end{tabular}
\label{tab:structure_results}
\end{table}

\section{Applications to Time Series Data and Sequential Decision Making} \label{sec:applications}

Our conditional CS divergence can be used in diverse applications associated with time series and sequential data. In the following, we comprehensively evaluate its performances against other SOTA methods in time series clustering and exploration in the absence of explicit rewards.


\subsection{Time Series Clustering}\label{sec:clustering}

Time series clustering is an unsupervised machine learning technique to partition time series data into groups. The similarity based approach is a dominating direction for time series clustering, in which the general idea is to infer the similarity (or distance) between pairwise time series and perform clustering based on the obtained similarities.

Popular time series similarity measures include for example dynamic time warping (DTW)~\cite{berndt1994using}, the time warp edit distance (TWED)~\cite{marteau2008time}, and the move-split-merge (MSM)~\cite{stefan2012move}. However, many of these measures cannot be straightforwardly applied to multivariate time series as they did not take relations between different attributes into account~\cite{banko2012correlation}. The recently proposed learned pattern similarity (LPS)~\cite{baydogan2016time} and time-series cluster kernel (TCK)~\cite{mikalsen2018time} are two exceptions. Both measures rely on an ensemble strategy to compute the similarity, whereas the latter leverages the Gaussian mixture model (GMM) to fit the data which makes it very effective to deal with missing values.

In contrast to the above-mentioned measures that aim to align two sequences locally or rely on ensemble strategies which is computationally expensive, we suggest a new way to measure time series similarity from a probabilistic perspective. Specifically, given two time series $\{\mathbf{x}_t\}$ and $\{\mathbf{y}_t\}$, let $K$ denote a predefined embedding dimension, the conditional distributions $p(\mathbf{x}_t |\mathbf{x}_{t-1},\mathbf{x}_{t-2},\cdots,\mathbf{x}_{t-K})$ and $p(\mathbf{y}_t |\mathbf{y}_{t-1},\mathbf{y}_{t-2},\cdots,\mathbf{y}_{t-K})$ characterize the predictive behavior or \emph{dynamics} of $\{\mathbf{x}_t\}$ and $\{\mathbf{y}_t\}$, respectively. Hence, our new measure directly evaluates the dissimilarity between $\{\mathbf{x}_t\}$ and $\{\mathbf{y}_t\}$ by the conditional CS divergence between $p(\mathbf{x}_t |\mathbf{x}_{t-1},\mathbf{x}_{t-2},\cdots,\mathbf{x}_{t-K})$ and $p(\mathbf{y}_t |\mathbf{y}_{t-1},\mathbf{y}_{t-2},\cdots,\mathbf{y}_{t-K})$ , i.e.,
\begin{equation}\label{eq:clustering_dynamics}
D_{\text{CS}}(p(\mathbf{x}_t |\mathbf{x}_{t-1},\mathbf{x}_{t-2},\cdots,\mathbf{x}_{t-K});p(\mathbf{y}_t |\mathbf{y}_{t-1},\mathbf{y}_{t-2},\cdots,\mathbf{y}_{t-K})).
\end{equation}

That is, we expect to quantify the closeness of two time series by measuring the discrepancy of their internal \emph{dynamics}. One can also understand our conditional divergence in Eq.~(\ref{eq:clustering_dynamics}) from a kernel adaptive filtering (KAF)~\cite{liu2008kernel} perspective. Specifically, given a time series $\{\mathbf{x}_t\}$ (or $\{\mathbf{y}_t\}$), the KAF aims to learn a nonlinear (kernel) regression function $f_\mathbf{x}$ (or $f_\mathbf{y}$) to predict $\mathbf{x}_t$ (or $\mathbf{y}_t$) using its past $K$ values $\{\mathbf{x}_{t-1},\mathbf{x}_{t-2},\cdots,\mathbf{x}_{t-K}\}$ (or $\{\mathbf{y}_{t-1},\mathbf{y}_{t-2},\cdots,\mathbf{y}_{t-K}\}$) in an online manner, i.e., $\mathbf{x}_t=f_\mathbf{x}(\mathbf{x}_{t-1},\mathbf{x}_{t-2},\cdots,\mathbf{x}_{t-K})$ and $\mathbf{y}_t=f_\mathbf{y}(\mathbf{y}_{t-1},\mathbf{y}_{t-2},\cdots,\mathbf{y}_{t-K})$. Here, $K$ is also called the filter order. Obviously, $f_\mathbf{x}$ is an approximation to $p(\mathbf{x}_t |\mathbf{x}_{t-1},\mathbf{x}_{t-2},\cdots,\mathbf{x}_{t-K})$. Hence, our divergence can also be interpreted as the discrepancy between filters $f_\mathbf{x}$ and $f_\mathbf{y}$.
However, we would like to emphasize that, although our divergence has such an interpretation, it does not mean we need to explicitly learn filters $f_\mathbf{x}$ and $f_\mathbf{y}$ or identify their parameters. That is, our divergence is model-free.

For simplicity, we assume $\{\mathbf{x}_t\}$ and $\{\mathbf{y}_t\}$ have the same length $L$. For each time series, we can reformulate the sequential observations into a so-called \emph{Hankel matrix} of size $(K+1)\times (L-K)$ as shown in Fig.~\ref{fig:hankel}. That is, for sample index $i$, we have a vector observation $\{\mathbf{x}_{i},\mathbf{x}_{i+1},\cdots,\mathbf{x}_{i+K-1}\}$ (or $\{\mathbf{y}_{i},\mathbf{y}_{i+1},\cdots,\mathbf{y}_{i+K-1}\}$) and its corresponding desired response $\mathbf{x}_{i+K}$ (or $\mathbf{y}_{i+K}$). Then, the problem reduces to how to estimate Eq.~(\ref{eq:clustering_dynamics}) from the $L-K$ pairs of observations $\{ [\mathbf{x}_{i},\mathbf{x}_{i+1},\cdots,\mathbf{x}_{i+K-1}]^T , \mathbf{x}_{i+K} \}_{i=1}^{L-K}$ that are drawn from $p(\mathrm{X})$ and another $L-K$ observations $\{ [\mathbf{y}_{i},\mathbf{y}_{i+1},\cdots,\mathbf{y}_{i+K-1}]^T , \mathbf{y}_{i+K} \}_{i=1}^{L-K}$ that are drawn from $q(\mathrm{Y})$, in which $\mathrm{X}\in\mathbb{R}^{K+1}$ and $\mathrm{Y}\in\mathbb{R}^{K+1}$. By Eq.~(\ref{eq:conditional_CS_est}), we only need to compute eight Gram matrices of size $(L-K) \times (L-K)$, without any parametric model or parametric assumptions on the underlying distribution. For example, the Gram matrix $L^p$ for the desired response variable is evaluated as $\left(L^{p}\right)_{ij}=\kappa(\mathbf{x}_{i+K} - \mathbf{x}_{j+K})$; whereas the cross Gram matrix $L^{pq}$ is $\left(L^{pq}\right)_{ij}=\kappa(\mathbf{x}_{i+K} - \mathbf{y}_{j+K})$, in which $1\leq i,j \leq L-K$.

\begin{figure}[t]
	\centering
		\includegraphics[width=.8\linewidth]{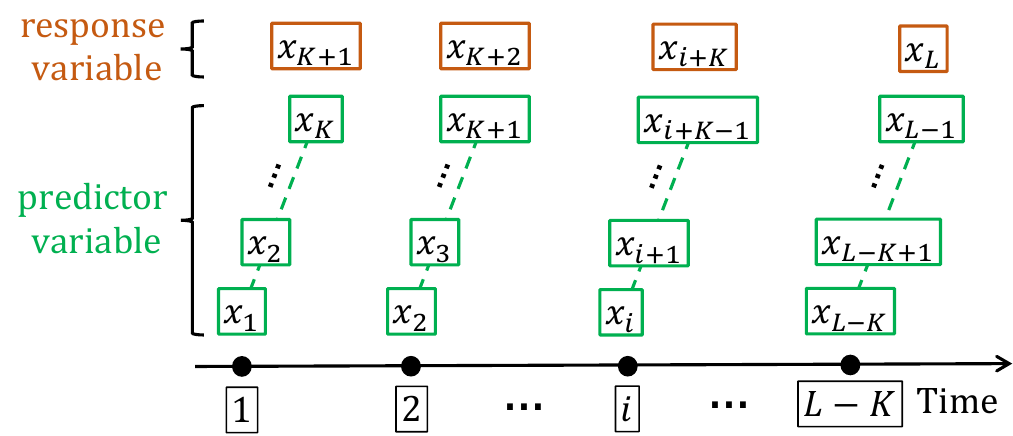}
	\caption{Reformulating a time series $\{\mathbf{x}_t\}$ into a \emph{Hankel} matrix.}
	\label{fig:hankel}
\end{figure}


We test our measure on $12$ benchmark time series datasets from the UCI\footnote{\url{https://archive.ics.uci.edu/ml/datasets.php}} and UCR\footnote{\url{https://www.cs.ucr.edu/~eamonn/time_series_data/}} databases, which cover a wide spectrum of domains, ranging from biomedical data to industrial processes. Additionally, we also test on two challenging dynamic texture (DT) datasets: Highway Traffic~\cite{chan2005classification} and UCLA~\cite{saisan2001dynamic}. The DT is a sequence of images of moving scenes such as flames, smoke, and waves that exihibits certain stationarity in time. The statistics of all datasets are summarized in Appendix~\ref{sec:data_clustering}. Fig.~\ref{fig:clustering_dataset} demonstrates exemplar time series from different classes in Synthetic Control and UCLA datasets. For datasets like Traffic and UCLA, the dimension $d$ is substantially larger than the length $T$, posing significant challenges for clustering tasks.

Our measure is compared to four other similarity measures, namely DTW, TWED, MSM, and TCK. The original DTW can only be applied to univariate time series, which has later been extended for multivariate scenarios~\cite{shokoohi2017generalizing}. In our work, we use a state-of-the-art (SOTA) multivariate implementation in \cite{schultz2018nonsmooth} for comparison.
Details of all competing measures and the setting of their hyperparameters are discussed in Appendix~\ref{sec:baseline_clustering}. For our method, the embedding dimension, i.e., $K$, is a crucial parameter. Practically, we can rely on Takens' embedding theorem~\cite{takens1981detecting} or set it heuristically with some prior knowledge. In our experiment, we observed that the Taken's embedding usually underestimates the value of $K$ that gives the best clustering performance. Therefore, for all the univariate time series, the value of $K$ is selected among $3$ values: $10$, $15$, and $20$. For multivariate time series, such as PenDigits and Robot failures, we set $K=1$ due to data prior knowledge. For Traffic and UCLA, we also set $K=1$, because the frame dimension is much larger than time series length.


\begin{figure}[t]
	\centering
	\subfigure[$10$ time series from classes ``Normal", ``Cyclic", and ``Increasing trend" in Synthetic Control dataset]{
		\centering
		\includegraphics[width=.4\linewidth]{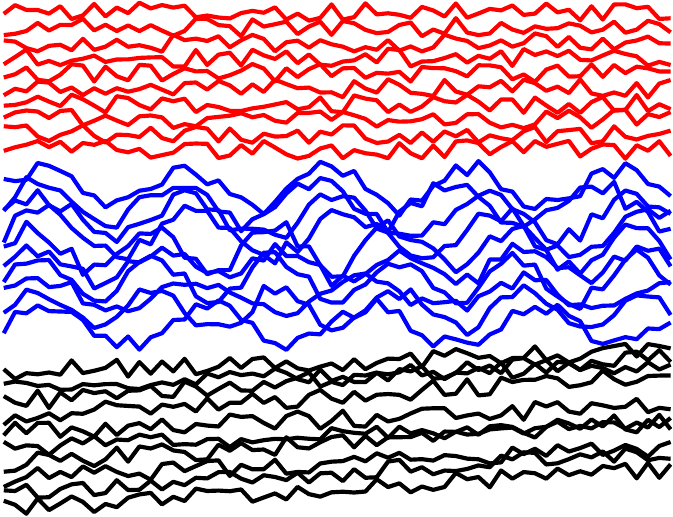}
	}
    \hfill
	\subfigure[Example snapshots of classes ``fire", ``fountain", ``water" in UCLA dataset]{
		\centering
		\includegraphics[width=.4\linewidth]{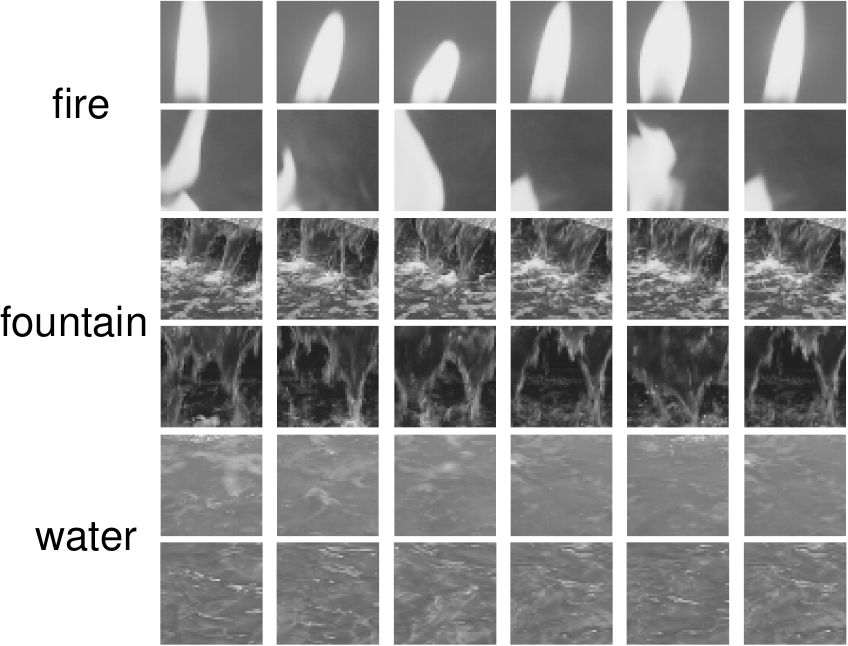}
	}
	\caption{Exemplar time series in (a) Synthetic Control; and (b) UCLA datasets. Each row is a time series.}
	\label{fig:clustering_dataset}
\end{figure}

For each test dataset, we can obtain a $N\times N$ dissimilarity matrix $D$ that encodes all pairwise dissimilarities using any competing measures ($N$ is the number of time series). To quantitatively evaluate the quality of $D$, we apply two clustering methods on top of $D$: $k$-medoids and spectral clustering~\cite{ng2001spectral}. To apply spectral clustering, we first convert the dissimilarity matrix $D$ to a valid adjacency matrix $A$ by a kernel smooth function, i.e., $A_{ij} = \exp\left(-D_{ij}/b\right)$. For each measure, the parameter $b$ is chosen from $\{0.1, 0.2, 1, 2, 10, 20\}$, and the best performance is taken.
We use normalized mutual information (NMI) as the clustering evaluation metric. Please refer to \cite{xu2003document} for detailed definitions of NMI. Table~\ref{tab:clustering_nmi_spectral} and Table~\ref{tab:clustering_nmi_kmedoids} summarize the clustering results using, respectively, spectral clustering and $k$-medoids. We can summarize a few observations: 1) the clustering performances in terms of two different clustering methods roughly remain consistent;
 2) there is no obvious winner for univariate time series, all methods can achieve competitive performance; this makes sense, as DTW, MSM, TWED, and TCK are all established methods; 3) our conditional CS divergence has obvious performance gain for multivariate time series; it is also generalizable to Traffic and UCLA, in which the dimension is significantly larger than the length; 4) the performance of our conditional CS divergence is stable in the sense that our measure does not have a failing case; by contrast, DTW gets very low NMI values in Robert failure LP1-LP5, whereas TCK completely fails in Traffic and UCLA.



\begin{table}[]
\centering
\caption {Clustering performance comparison (by spectral clustering) in terms of normalized mutual information (NMI). ``-" indicates the corresponding measures cannot be extended to multivariate time series or fail to obtain meaningful results. The best performance is in bold; the second best performance is underlined.}
\begin{tabular}{c|c|c|c|c|l}
\toprule
Datasets  & DTW  & MSM  & TWED   & TCK   & \multicolumn{1}{c}{\textbf{\begin{tabular}[c]{@{}c@{}}Cond.\\ CS (ours)\end{tabular}}} \\ \hline
Coffee    & $\underline{0.689}$  & 0.592  & $\underline{0.689}$  & $\underline{0.689}$ &  $\mathbf{1}$   \\
Diatom     & 0.788  & $\mathbf{0.837}$  & 0.774  & $\underline{0.806}$ & 0.743  \\
DistalPhalanxTW  & 0.570 & 0.522  & $\underline{0.571}$ & 0.491 & $\mathbf{0.584}$   \\
ECG5000    & $\mathbf{0.833}$ & $\underline{0.777}$ & 0.725 & 0.695  & 0.727   \\
FaceAll   & 0.508 & $\underline{0.796}$  &  $\mathbf{0.849}$  & 0.559 & 0.750 \\
Synthetic control & 0.744 & 0.565 & 0.565 & $\mathbf{0.781}$ & $\underline{0.758}$ \\
\midrule
PenDigits    &  $\mathbf{0.725}$ &    &     &  0.488  & $\underline{0.655}$  \\
Libras    & $\mathbf{0.652}$  &      &     &  0.554  & $\underline{0.606}$ \\
uWave     & $\mathbf{0.834}$  &      &     &  0.624  & $\underline{0.759}$   \\
Robot failure LP1  & 0.0607 &    &   & $\underline{0.176}$ & $\mathbf{0.686}$  \\
Robot failure LP2  & 0.227 &    &   & $\underline{0.363}$  & $\mathbf{0.438}$  \\
Robot failure LP3  & $\underline{0.170}$ &    &   & 0.144  & $\mathbf{0.328}$  \\
Robot failure LP4  & $\underline{0.241}$  &   &   & 0.080  & $\mathbf{0.516}$ \\
Robot failure LP5  & $\underline{0.091}$  &   &   & 0.080  & $\mathbf{0.393}$   \\
\midrule
Traffic  &  $\underline{0.144}$  &    &   &  -  & $\mathbf{0.145}$   \\
UCLA     &  $\underline{0.149}$  &    &   &  -  & $\mathbf{0.559}$  \\
\bottomrule
\end{tabular}
\label{tab:clustering_nmi_spectral}
\end{table}

\begin{table}[]
\centering
\caption {Clustering performance comparison (by $k$-medoids) in terms of normalized mutual information (NMI). ``-" indicates the corresponding measures cannot be extended to multivariate time series or fail to obtain meaningful results. The best performance is in bold; the second best performance is underlined.}
\begin{tabular}{c|c|c|c|c|l}
\toprule
Datasets  & DTW  & MSM  & TWED   & TCK   & \multicolumn{1}{c}{\textbf{\begin{tabular}[c]{@{}c@{}}Cond.\\ CS (ours)\end{tabular}}} \\ \hline
Coffee    & 0.592  & 0.258  & 0.689  & $\underline{0.811}$  &  $\mathbf{1}$   \\
Diatom     & 0.552  &  $\underline{0.768}$  & 0.760 & $\mathbf{0.821}$  & 0.730  \\
DistalPhalanxTW & 0.484 & 0.456 & $\underline{0.495}$ & 0.458 &   $\mathbf{0.587}$   \\
ECG5000    & $\mathbf{0.846}$ & $\underline{0.756}$ & 0.707 & 0.727  & 0.595   \\
FaceAll   & 0.694 & $\underline{0.722}$  &  $\mathbf{0.849}$  & 0.584 & 0.705 \\
Synthetic control & $\mathbf{0.909}$ & $\underline{0.899}$ & 0.856 & 0.869 & 0.745 \\
\midrule
PenDigits    & $\mathbf{0.722}$  &    &     &  0.416  &  $\underline{0.640}$ \\
Libras    &  $\underline{0.440}$  &         &      & $\mathbf{0.596}$  &  0.390 \\
uWave     &  $\underline{0.748}$ &       &        & 0.707 & $\mathbf{0.754}$ \\
Robot failure LP1  & 0.132 &    &    & $\underline{0.455}$  & $\mathbf{0.658}$ \\
Robot failure LP2  & 0.232 &    &    & $\underline{0.315}$  & $\mathbf{0.427}$ \\
Robot failure LP3  & 0.106 &    &    & $\underline{0.149}$  & $\mathbf{0.453}$ \\
Robot failure LP4  & 0.084 &    &    & $\underline{0.113}$  & $\mathbf{0.365}$ \\
Robot failure LP5  & 0.098 &    &    & $\underline{0.234}$  & $\mathbf{0.261}$ \\
\midrule
Traffic  &  $\underline{0.139}$  &    &   &  -  & $\mathbf{0.152}$   \\
UCLA     &  $\underline{0.381}$  &    &   &  -  & $\mathbf{0.416}$  \\
\bottomrule
\end{tabular}
\label{tab:clustering_nmi_kmedoids}
\end{table}

Finally, one should note that, different from DTW and other competing measures, our conditional CS divergence is not specifically designed for just measuring similarity between two temporal sequences, it is versatile, flexible, and has wide usages in other machine learning problems.


Apart from the above quantitative analysis, we provide in Appendix~\ref{sec:exploratory} how to perform exploratory data analysis for real-world time series data to identify useful patterns. Additionally, we planned to conduct the same experiment using conditional KL divergence and conditional MMD. However, applying conditional KL divergence straightforwardly presents a challenge as the $k$NN estimator may yield negative values with a noticeable likelihood, hindering the application of spectral clustering or $k$-medoids. This issue is not encountered with conditional CS. On the other hand, the performance of conditional MMD is always poor, due to the difficulty of tuning the additional hyperparameter $\lambda$ in Eq.~(\ref{eq:CMMD}). Furthermore, the computational burden of conditional MMD is cubic, stemming from matrix inversion.

\subsection{Uncertainty-Guided Exploration for Sequential Decision Making}\label{CS-dtg}

Our second application deals with sequential decision making under uncertainty, especially in scenarios where clear rewards are sparse or not available. Let us recall a (discounted) Markov Decision Process (MDP) defined by a $5$-tuple ($\mathcal{S},\mathcal{A},P,R,\gamma$), where $\mathcal{S}$ is the set of all possible states (state space), $\mathcal{A}$ is the set of all possible actions (action space).
The agent learns a policy function $\pi: \mathcal{S}\mapsto \mathcal{A}$ mapping states to actions at every time step.
$P:\mathcal{S}\times \mathcal{A}\mapsto \mathcal{S}$ is the transition function or probability. At each timestep $t$, upon observing the state $s_t$, the execution of action $a_t$ triggers an extrinsic reward $r_t=r(s_t,a_t)\in R$ given by the reward function $R:\mathcal{S}\times \mathcal{A}\mapsto \mathbb{R}$, and a transition to a new state $s_{t+1}\sim P(\cdot|s_t,a_t)$. The optimal agent maximizes the discounted cumulative extrinsic reward $Q=\mathbb{E}\left(\sum_{t=0}^\infty \gamma^t r_{t+1}|s_t,a_t \right)$, where $\gamma\in (0,1)$ is a discounting factor that decays rewards received further in the future.

Even though the reward paradigm is fundamentally flexible in many ways, it is also brittle and limits the agent's ability to learn about its environment, especially when the rewards are sparse or unavailable. To make agents to discover the environment without the requirement of a reward signal, such that the learned policy is more flexible and generalizable, recent studies seek to find an alternative form of reinforcement learning with an objective that is reward-independent and favors exploration~\cite{hazan2019provably}.

Many exploration schemes have been developed over the past years. Here, we consider an uncertainty-guided exploration in the sense that if a state has not been visited sufficiently for the agent to be familiar with it, then that state will have high uncertainty and an agent will be driven to it. This ensures that an agent will thoroughly investigate new areas of the action-state space~\cite{sledge2018guided}.


At each timestep $t$, we observe a transition of state $s_t\rightarrow s_{t+1}$, we can update our belief distribution $p(s_{t+1}|s_t,a_t)$. Following the observation of a new state transition, we can compute the divergence $D(p_{\text{new}} (s_{t+1}|s_t,a_t );p_{\text{old}} (s_{t+1}|s_t,a_t))$ between the two belief functions, before and after incorporating the new knowledge. In practice, we introduce a replay buffer to record $2\tau$ steps of $\{ s_{t+1}, s_t,a_t\}$ trios. Here we define $p_{\text{old}}$ to be the conditional distribution of trios in the old half of the buffer, while $p_{\text{new}}$ corresponds the new half. When a new state transition is observed that greatly changes the model, this means that the agent was quite uncertain about the outcome of the action taken in that particular state.
In order to encourage exploration of the agent to states that have not been visited sufficiently, the action selection module should be trained by maximizing $D(p_{\text{new}} (s_{t+1}|s_t,a_t );p_{\text{old}} (s_{t+1}|s_t,a_t))$.
We term such exploration strategy the ``divergence-to-go (DTG)", and the optimal policy $\pi_{\text{dtg}}$ aims to maximize
the divergence between old and new experiences at each visited state:
\begin{equation}\label{eq:obj_RL}
    \pi_{\text{dtg}} = \underset{\pi}{\text{argmax}} \mathbb{E} \left( \sum_{t=0}^\infty D(p_{\text{new}} (s_{t+1}|s_t,a_t );p_{\text{old}} (s_{t+1}|s_t,a_t))
    \right).
\end{equation}


The DTG exploration was initially proposed in \cite{emigh2015model}, but the authors estimate $D(p_{\text{new}} (s_{t+1}|s_t,a_t );p_{\text{old}} (s_{t+1}|s_t,a_t))$ by the Euclidean distance and make a few additional assumptions on $p(s_{t+1}|s_t,a_t )$ to make the estimation tractable. In this section, we straightforwardly estimate the conditional CS divergence between $p_{\text{new}} (s_{t+1}|s_t,a_t )$ and $p_{\text{old}} (s_{t+1}|s_t,a_t )$ using Eq.~(\ref{eq:conditional_CS_est}) in which we treat $s_{t+1}$ as variable $\mathbf{y}$ and the concatenation of $[s_t, a_t]$ as variable $\mathbf{x}$, without any distributional assumptions. We denote our improved exploration method the DTG-CS. Identical to the original DTG, we also incorporate a kernelized Q-learning backbone~\cite{bae2011reinforcement} into our methodology. The detailed algorithm of DTG-CS and its major differences to standard reward-based Q-learning are provided in Appendix~\ref{sec:dtg_details}.

We apply our method to a mountain car, pendulum and maze task, as shown in Fig.~\ref{fig:env}. The mountain car is a benchmark reinforcement learning (RL) test bed in which the agent attempts to drive an underpowered car up a $2$-dimensional hill. Each time step, the agent selects an action from $\mathcal{A}=\{-1,0,1\}$ based on the $2$-dimensional state space $\mathcal{S}$ consisting of the cart’s position and velocity. Zero represents no action. Negative actions drive the cart to the left while positive actions drive it to the right. To reach the goal, the agent must climb the hill to the left and then allow the combination of gravity and positive actions to drive the cart to the top of the hill on the right. The pendulum is another classic control problem in reinforcement learning, which consists of a pendulum attached at one end to a fixed point, and the other end being free. We set it to start at the downward position and the goal is to apply torque on the free end to swing it into an upright position. In each step, the agent chooses an action (torque) from the range $[-1, +1]$. In maze game, an agent is placed from a start (red point). The goal of the agent is to reach the exit (blue point) as quickly as possible. For every step, the agent must decide which action to take among four options: left, right, up, or down.


We compare our DTG-CS with basic Q-learning, the random exploration, the baseline DTG~\cite{emigh2015model} and the recently proposed maximum entropy (MaxEnt) exploration~\cite{hazan2019provably}. To demonstrate the superiority of conditional CS divergence over its KL divergence and MMD counterparts, we also implement DTG-KL and DTG-MMD. Basic Q-learning usually requires a pre-collected state-reward buffer with goal rewards or random policy to avoid converging to local optimum (or even non-optimum), if rewards are sparse. Hence, we introduce random actions using $\epsilon$-greedy ($\epsilon_{0} = 0.5$) in the Q-learning for the mountain car and the Maze. MaxEnt encourages efficient discovery of an unknown state space by maximizing the entropy of state distribution $-\mathbb{E}_{s \in \mathcal{S}}[\log_2p(s)]$, in which $\mathcal{S}$ is the set of all visited states.
We repeat the training process for $100$ independent runs with different global seeds and record the mean value of the number of steps used to achieve the goal for the first time. Note that, for maze game, the agent may not reach the exit given a sufficient number of training steps (because it has a chance to converge to a dead end). Therefore, we additionally report the success rate over $100$ independent runs. If the method cannot achieve the goal within $50,000$ training steps, we regard it as a failure.

Table~\ref{tab:dtg_results} shows the quantitative results. For mountain car, our method outperforms all competing methods, although the improvement over DTG is marginal. This is not surprising since mountain car is a simple task. In short, exploration-based methods (MaxEnt, DTG, and its variants including our DTG-CS) can achieve the goal faster than reward-driven Q-learning and random policy. For pendulum, in contrast to the other two environments where the agent can only obtain rewards by achieving the goals, the agent of the pendulum receives rewards at arbitrary positions. The closer to the upright position, the higher the rewards\footnote{\url{https://www.gymlibrary.dev/environments/classic_control/pendulum/}}. Hence, it is not surprising that the reward-based Q-learning outperforms others. Our DTG-CS achieves the best result among the remaining methods, without explicitly utilizing rewards. For maze, our method outperforms all other methods by a substantial margin. We obtain similar results to basic DTG but improve significantly the success rate. MaxEnt obtains a higher success rate than Q-learning and random policy but takes a longer time to achieve the goal.

Fig.~\ref{fig:dtg_results} visualizes results using heatmaps of training steps. For mountain car, we plot the log-probability of occupancy of the $2$D state space with location on the $X$-axis and speed on the $Y$-axis. In summary, states of Q-learning and random policy cluster near the initial state while exploration-based methods (especially DTG and DTG-CS) evenly distributed in the entire state space. For pendulum, we plot the $2$D positions of the free end to show the traces of the bar. A full circle indicates the agent explores the state space completely. Random policy stays in the lower semicircle due to gravity. Q-learning can explore the full space with the help of dense rewards. MaxEnt explores more than random but cannot distribute evenly before reaching $50,000$ steps training. DTG-based methods produce similar results to Q-learning but employ no rewards.  For maze, we plot $2$D state spaces with $(x,y)$ coordinates. States of Q-learning and random policy still cluster near the start point, and the density decreases monotonically with respect to the distance to the exit. MaxEnt produces two local maximums apart from the initial state (see the middle and the upper right corner), which increases the likelihood of getting stuck in a dead end. The DTG-based approach avoids both limitations.

\begin{figure}[t]
	\centering
	\subfigure[Mountain Car]{
		\centering
		\includegraphics[width=.3\linewidth]{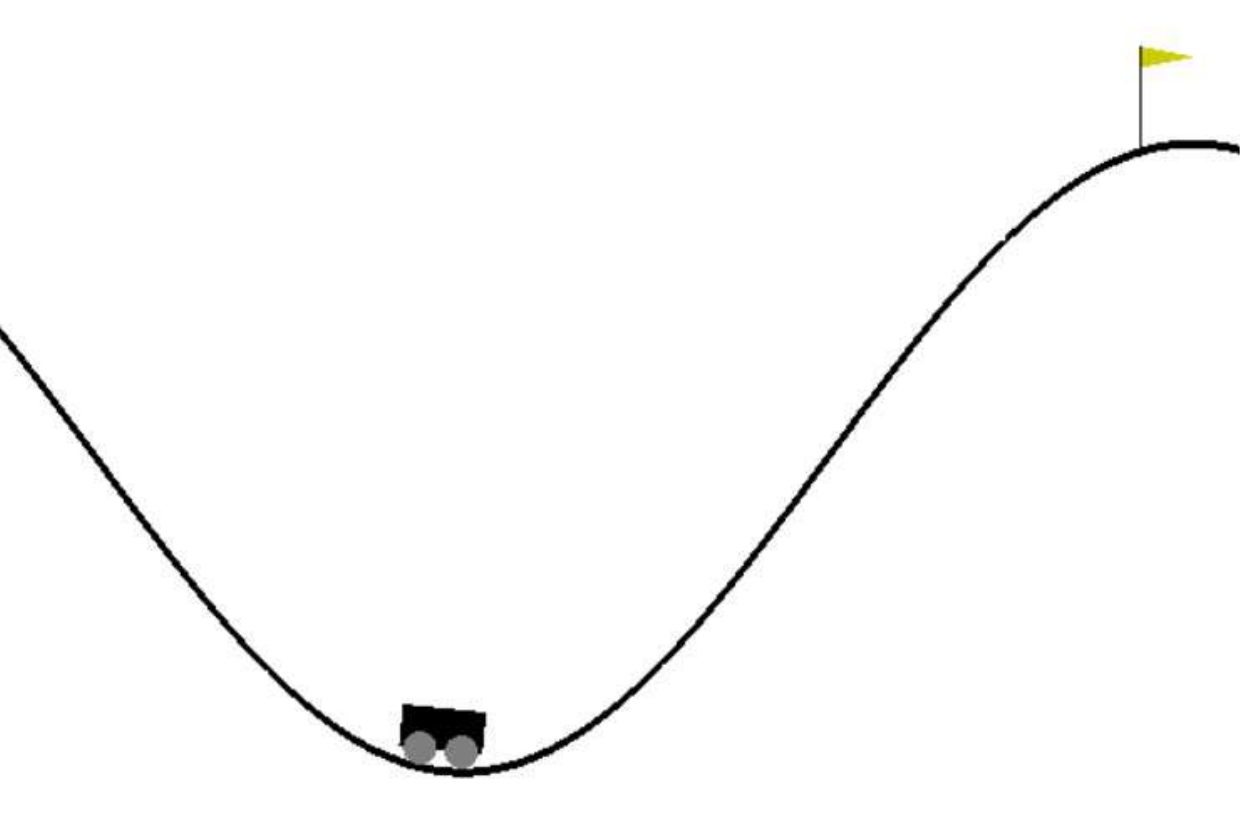}
	}
 	\subfigure[Pendulum]{
		\centering
		\includegraphics[width=.3\linewidth]{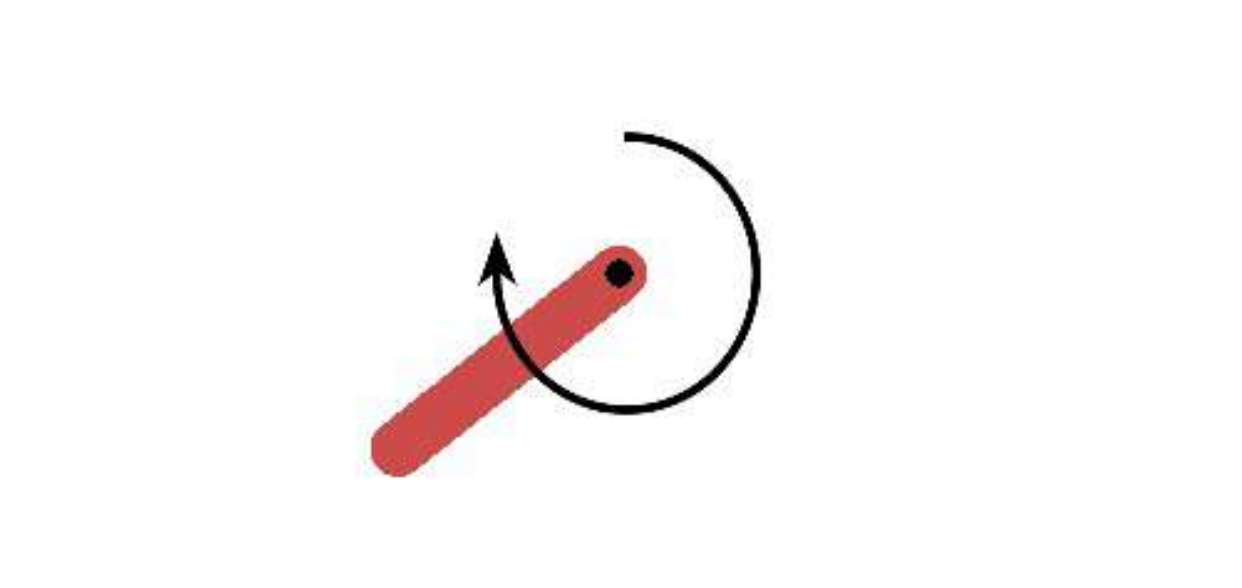}
	}
	\subfigure[Maze]{
		\centering
		\includegraphics[width=.3\linewidth]{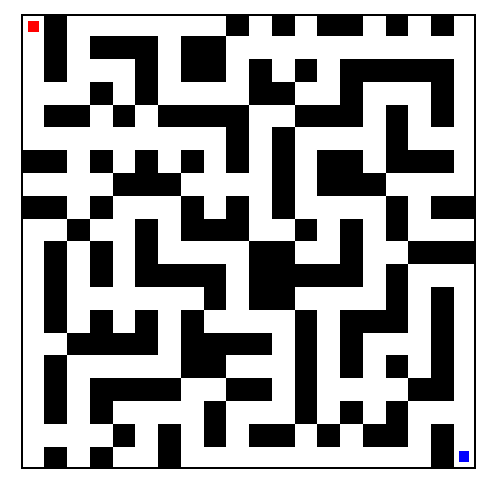}
	}
	\caption{(a) The goal of mountain car is to touch the flag; (b) The goal of Pendulum is to make the bar upright. (c) A $20 \times 20$ maze. The goal is to reach the exit (blue point) as quickly as possible from the start (red point).}
	\label{fig:env}
\end{figure}

\begin{table}[]
\centering
\caption{For mountain car and Pendulum, we show the number of steps used to achieve the goal for the first time. For Maze, we also show the success rate after ``/" given $50,000$ training steps.}
\begin{tabular}{|c|c|c|c|}
\hline
Methods    & \begin{tabular}[c]{@{}c@{}}Mountain \\ car (step)\end{tabular} & \begin{tabular}[c]{@{}c@{}}Pendulum \\ (step)\end{tabular} & \begin{tabular}[c]{@{}c@{}}Maze \\ (step/success rate)\end{tabular} \\ \hline
Random     & 42085                                                          & 5884                                                    & 15666/0.15                                                      \\ \hline
Q-learning & 53204                                                          & 3276                                                           & 12906/0.18                                                      \\ \hline
MaxEnt     & 13761                                                          & 5213                                                         & 31039/0.4                                                       \\ \hline
DTG        & 979                                                         & 4562                                                          & 2618/0.58                                                       \\ \hline
DTG-KL   & 1659                                                         & 4780                                                          & 4181/0.53                                                       \\ \hline
DTG-MMD  & 1853                                                         & 4282                                                          & 4383/0.61                                                       \\ \hline
DTG-CS     & $\mathbf{946}$                                       & $\mathbf{4105}$                                                   & $\mathbf{2581/0.75}$                                                       \\ \hline
\end{tabular}
\label{tab:dtg_results}
\end{table}

\begin{figure*}[t]
	\centering
\small
	\subfigure[Random]{
		\centering
		\includegraphics[width=.125\linewidth]{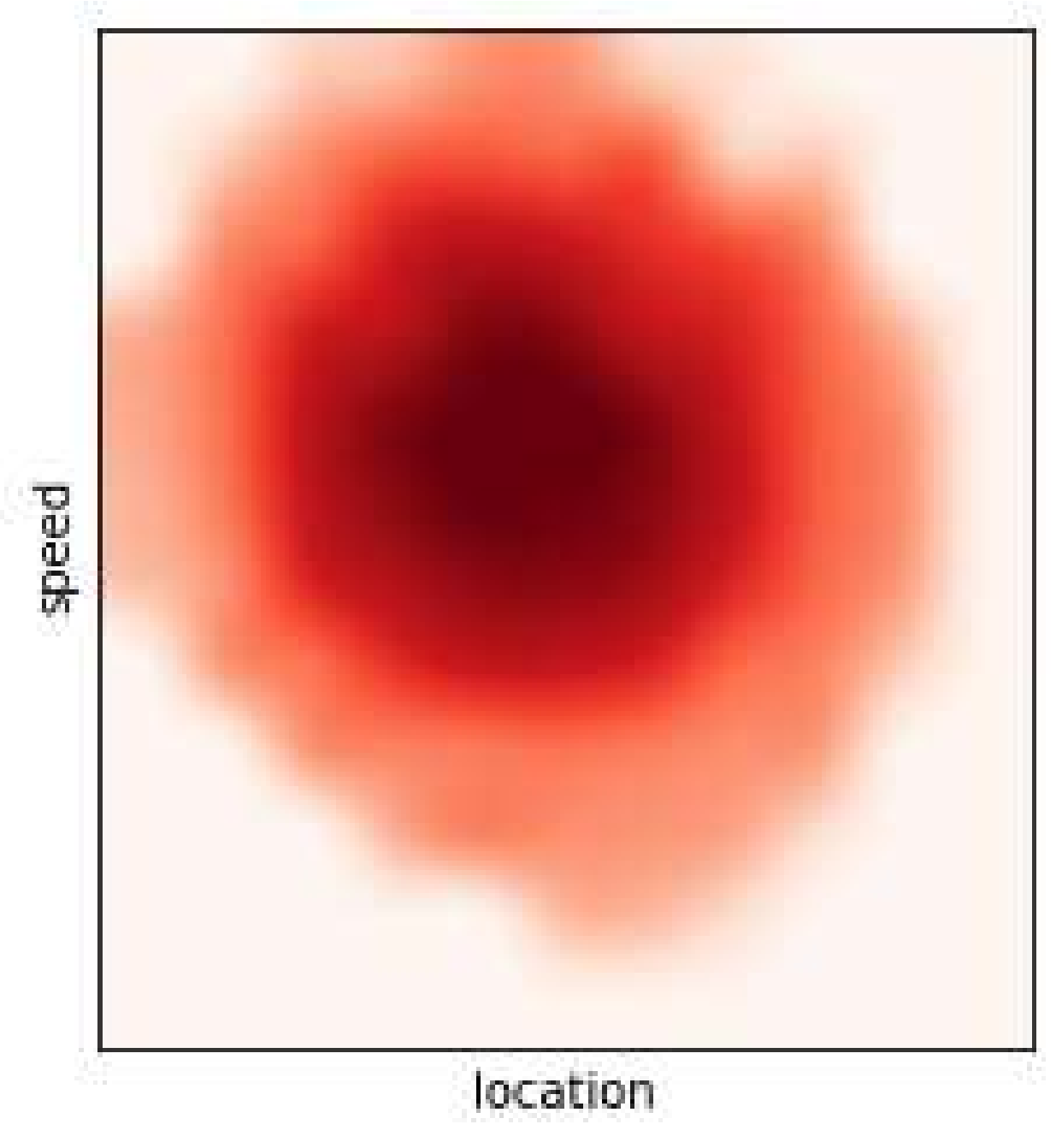}
	}
	\subfigure[Q-learning]{
		\centering
		\includegraphics[width=.125\linewidth]{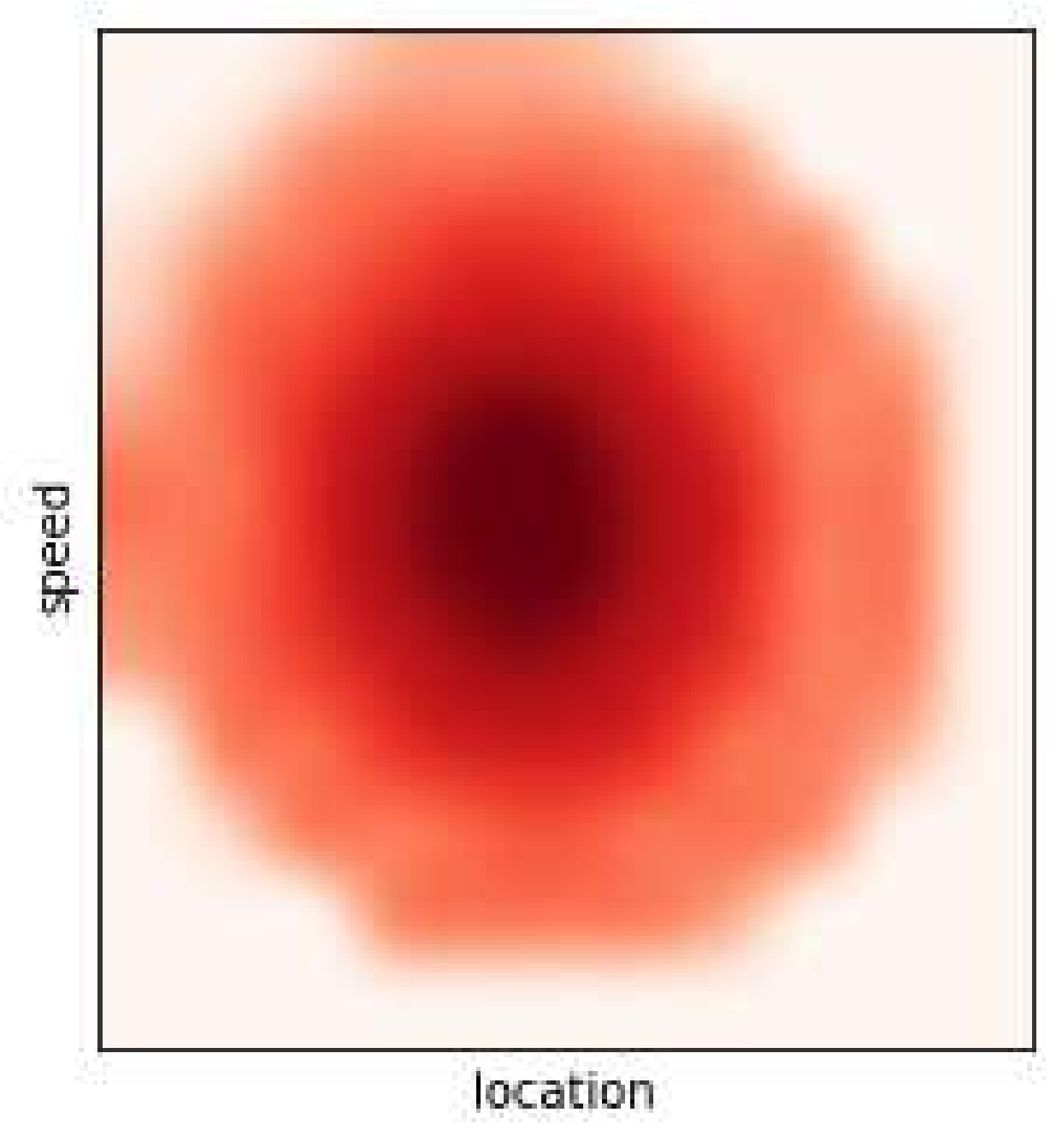}
	}
 	\subfigure[MaxEnt]{
		\centering
		\includegraphics[width=.125\linewidth]{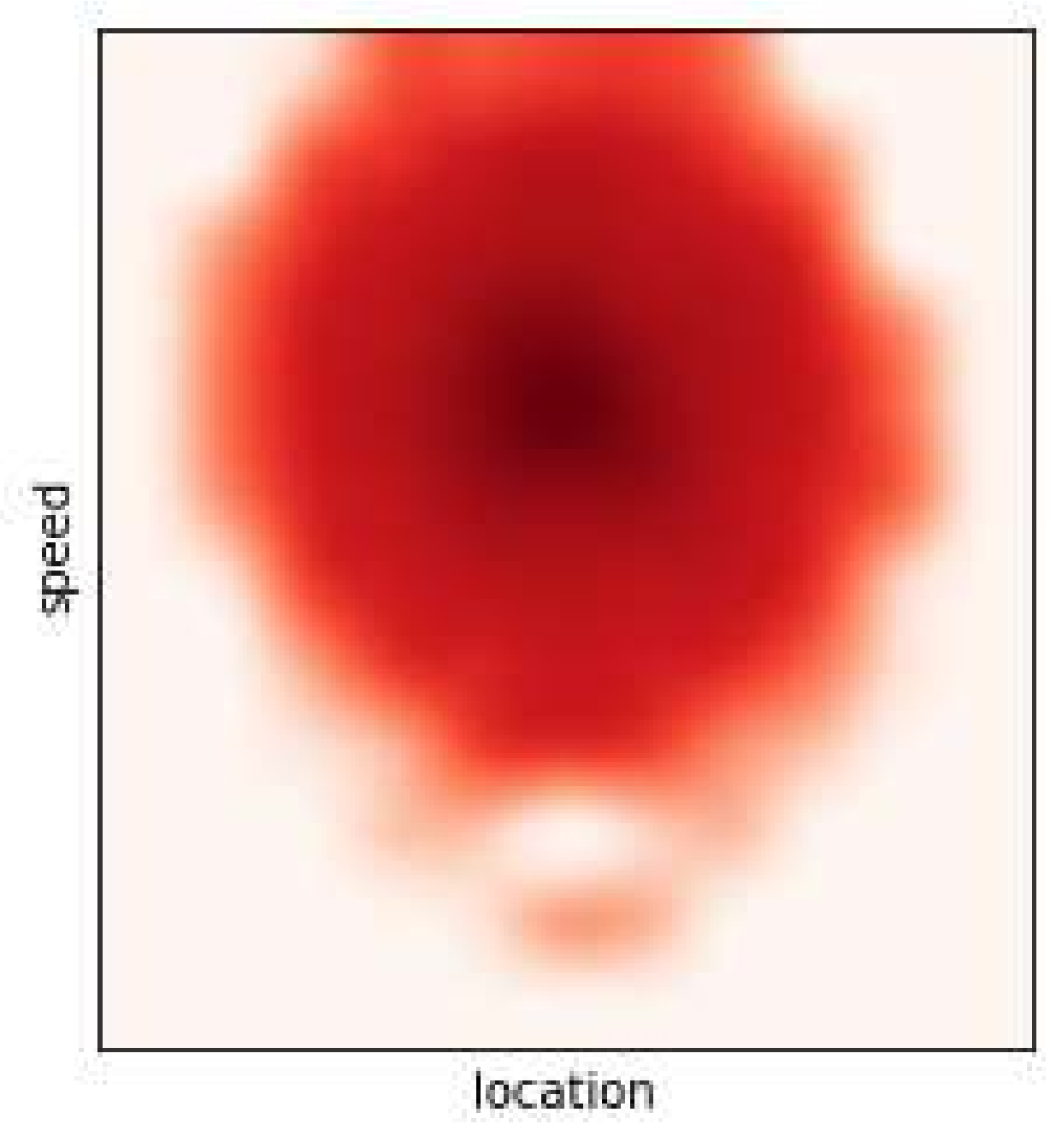}
	}
 	\subfigure[DTG]{
		\centering
		\includegraphics[width=.125\linewidth]{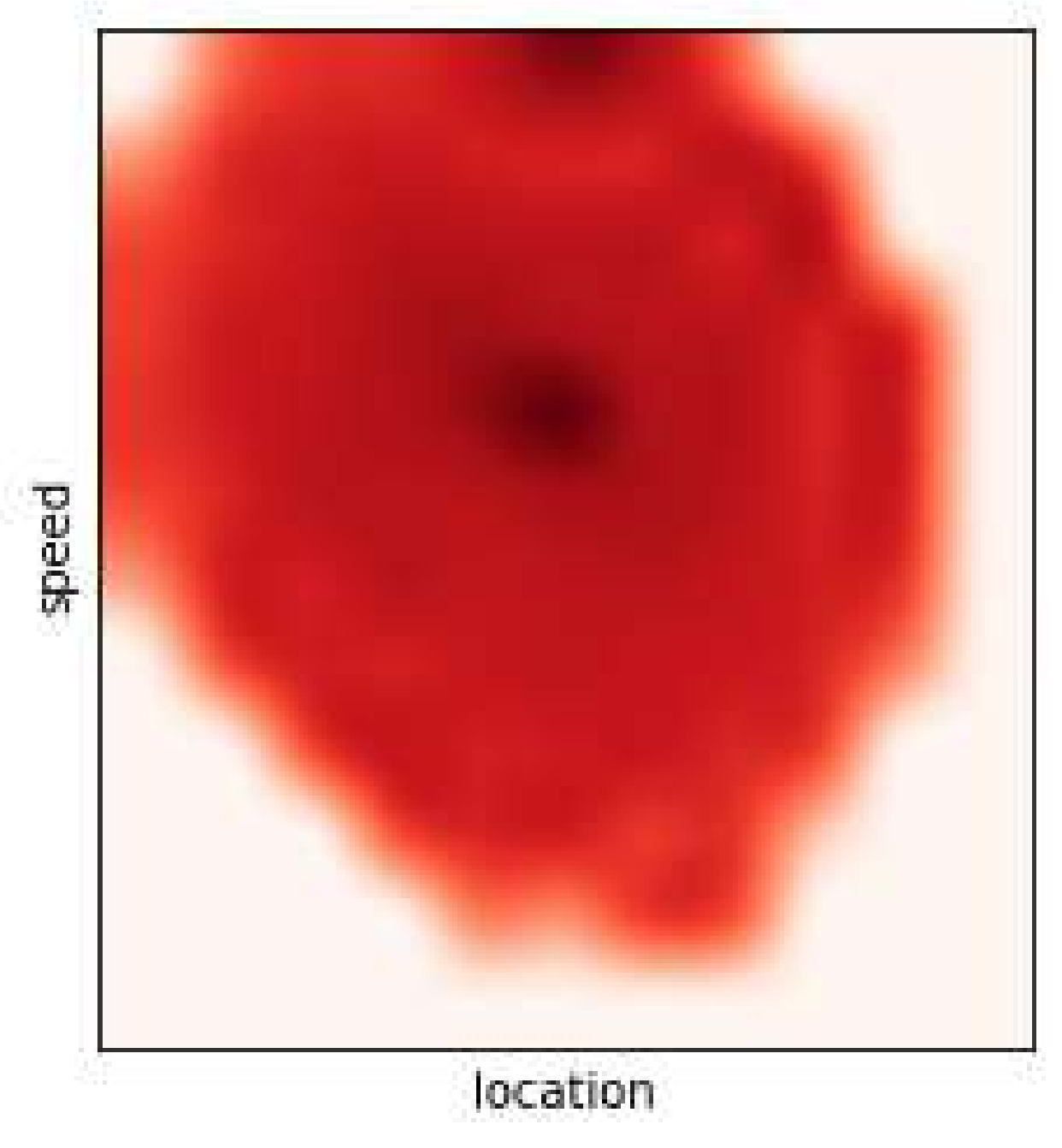}
	}
 	\subfigure[DTG-KL]{
		\centering
		\includegraphics[width=.125\linewidth]{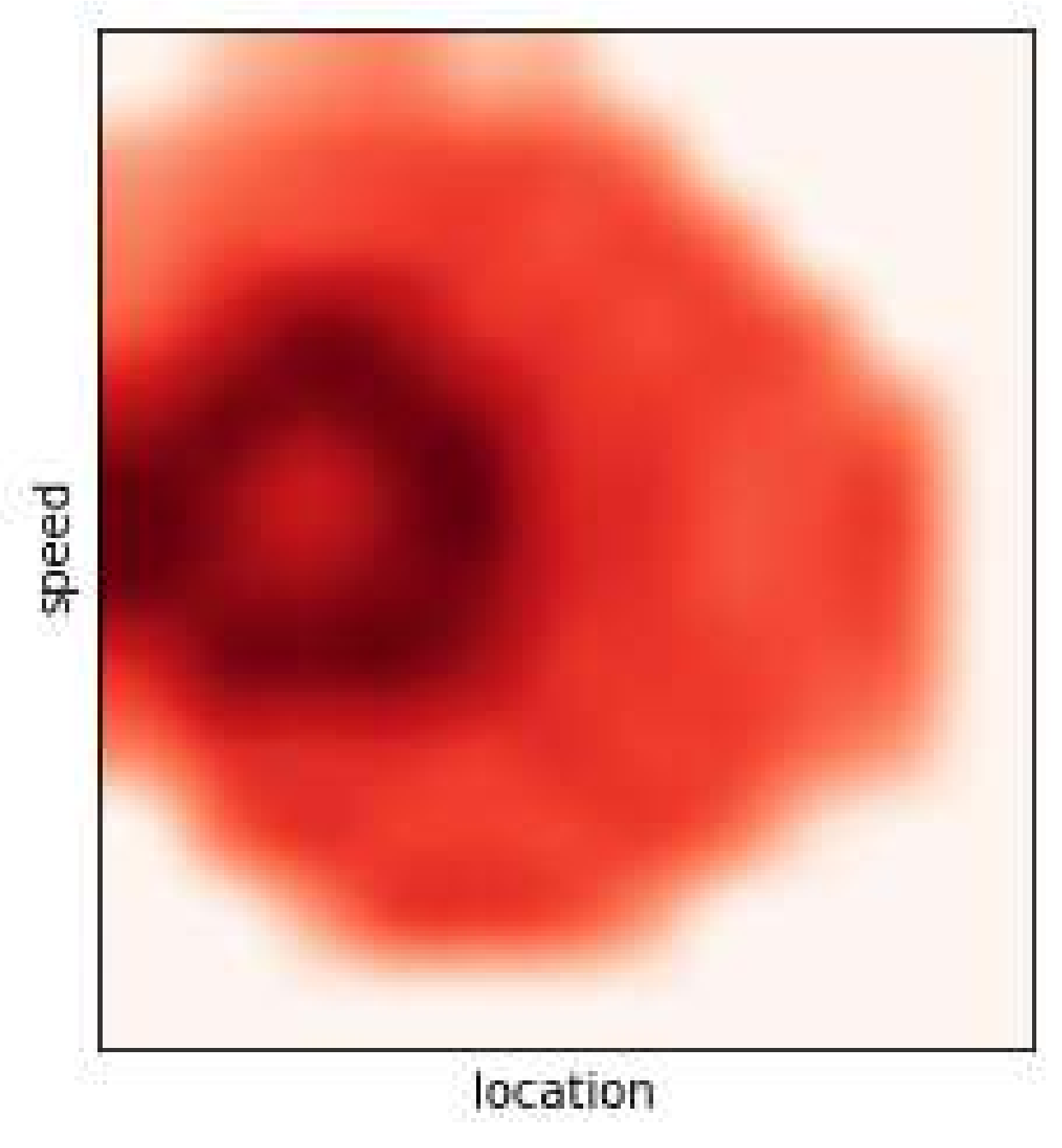}
	}
 	\subfigure[DTG-MMD]{
		\centering
		\includegraphics[width=.125\linewidth]{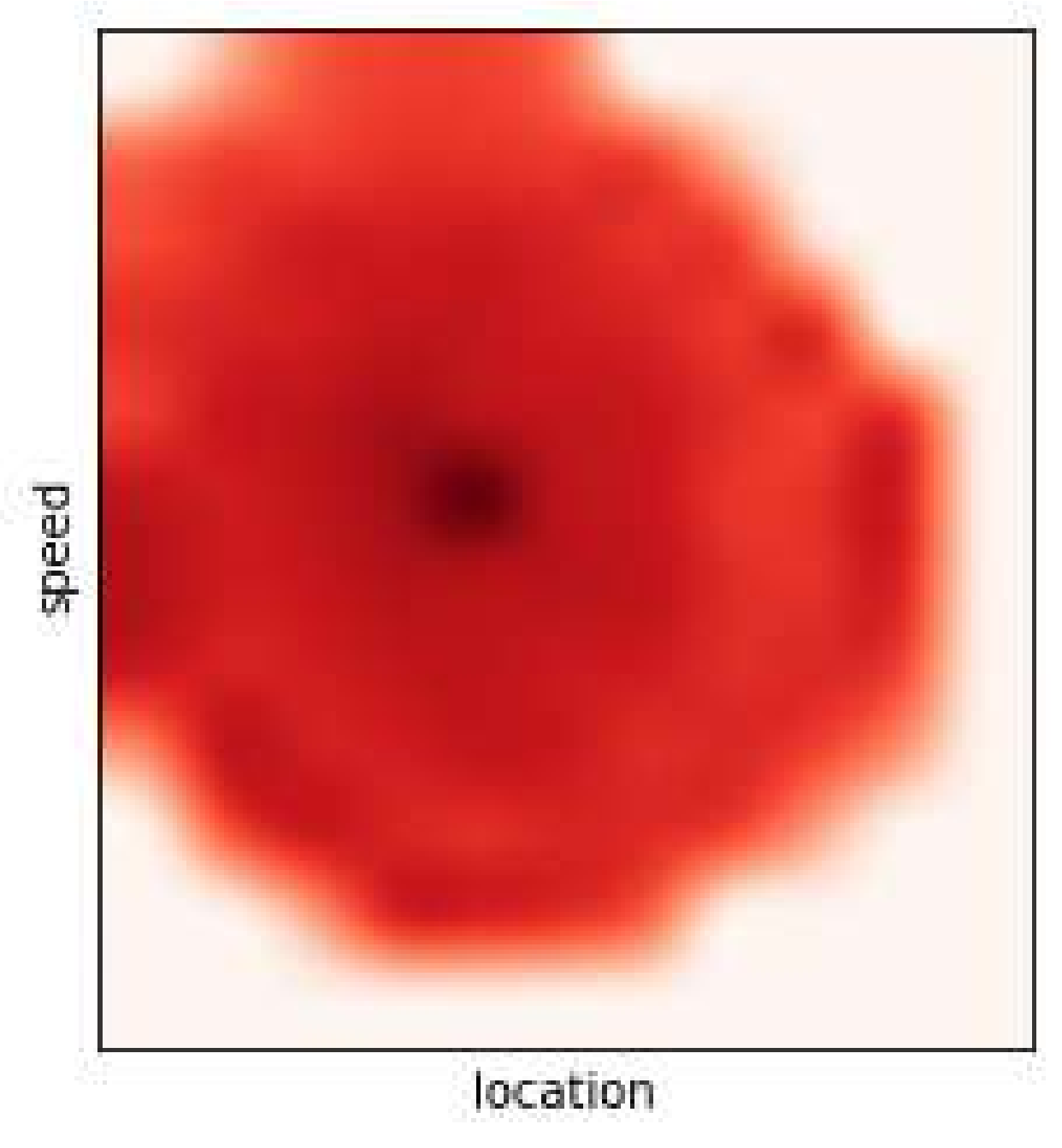}
	}
        \subfigure[DTG-CS]{
		\centering
		\includegraphics[width=.125\linewidth]{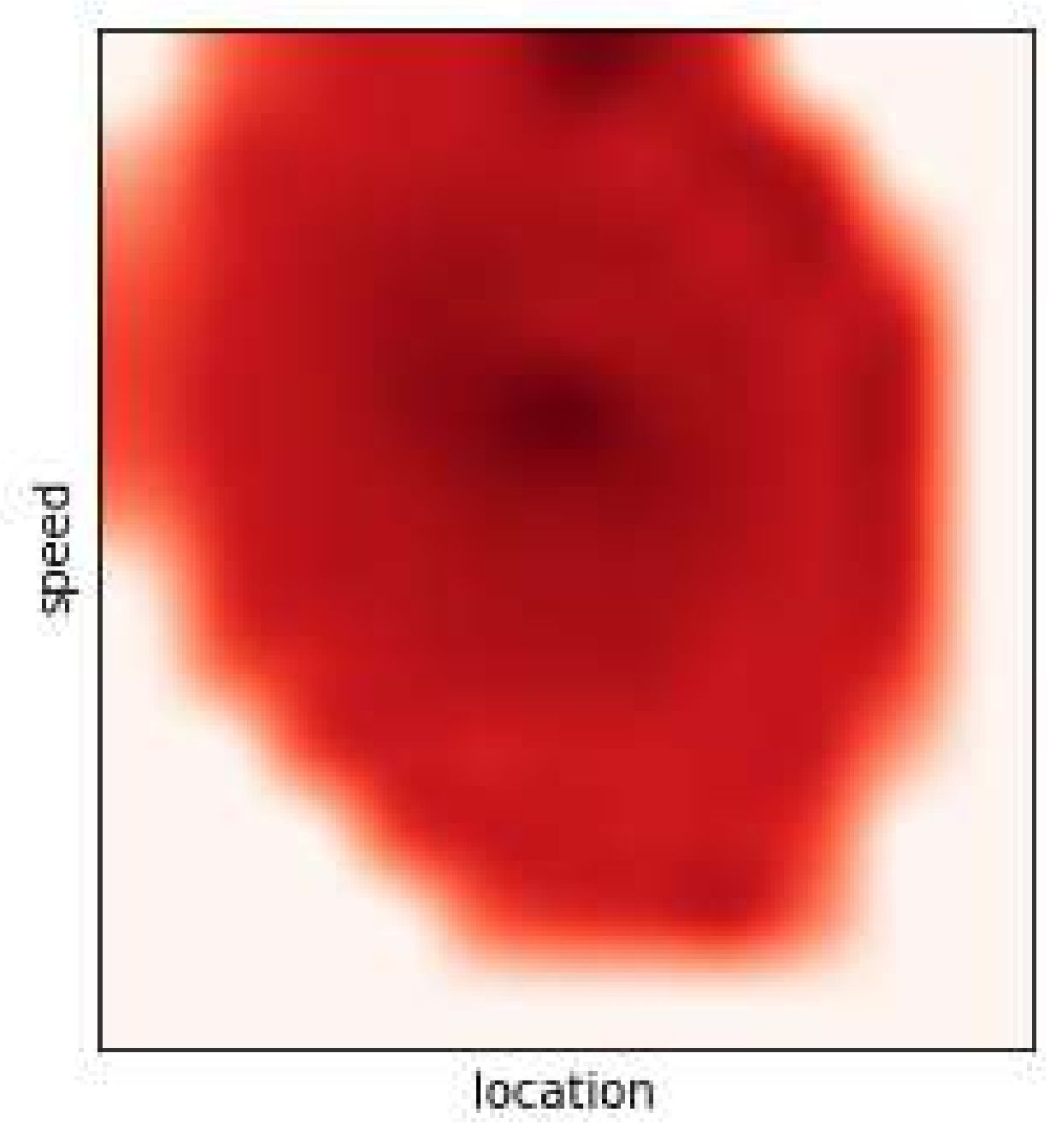}
	}\\
         \subfigure[Random]{
		\centering
		\includegraphics[width=.125\linewidth]{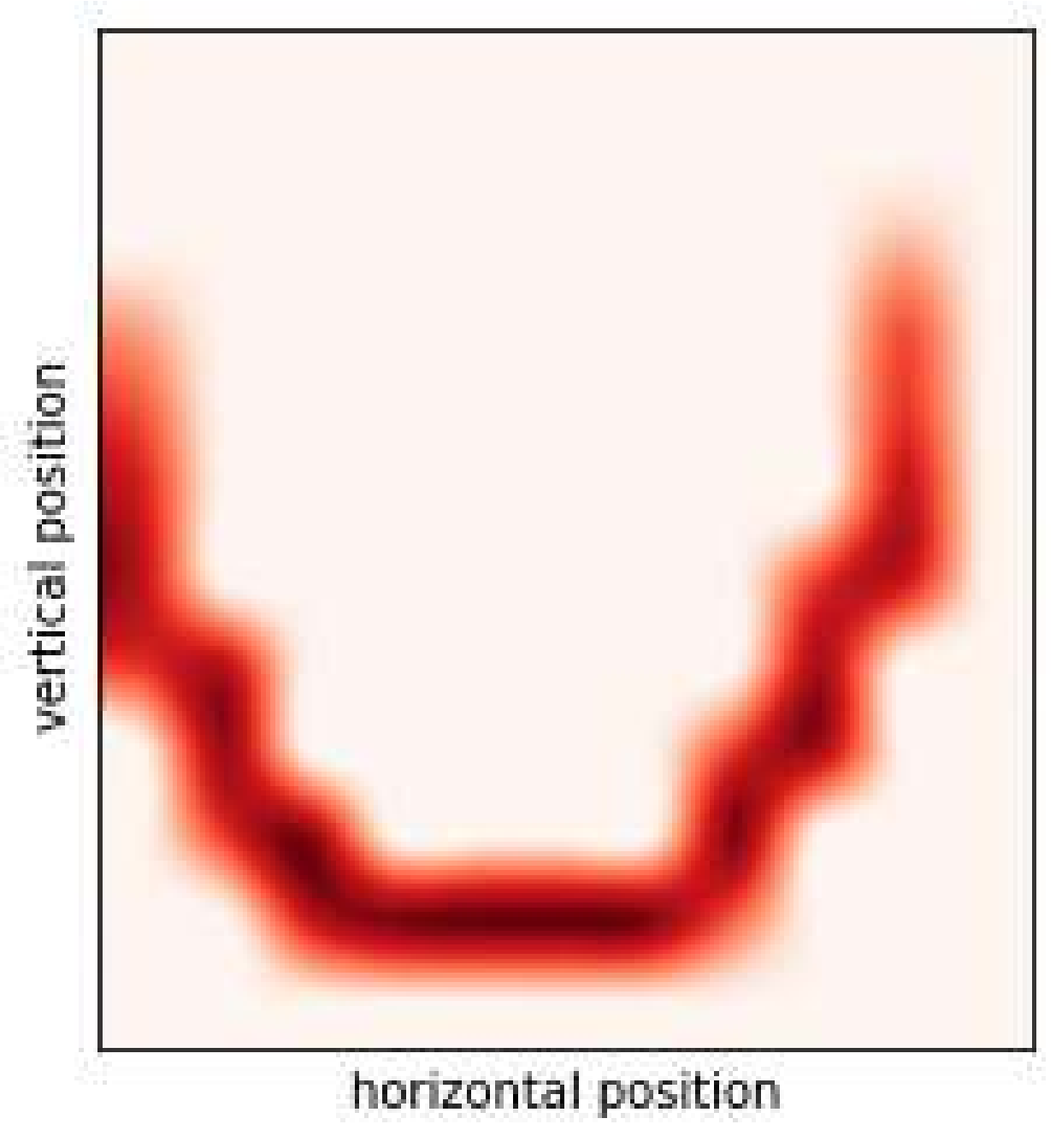}
	}
 	\subfigure[Q-learning]{
		\centering
		\includegraphics[width=.125\linewidth]{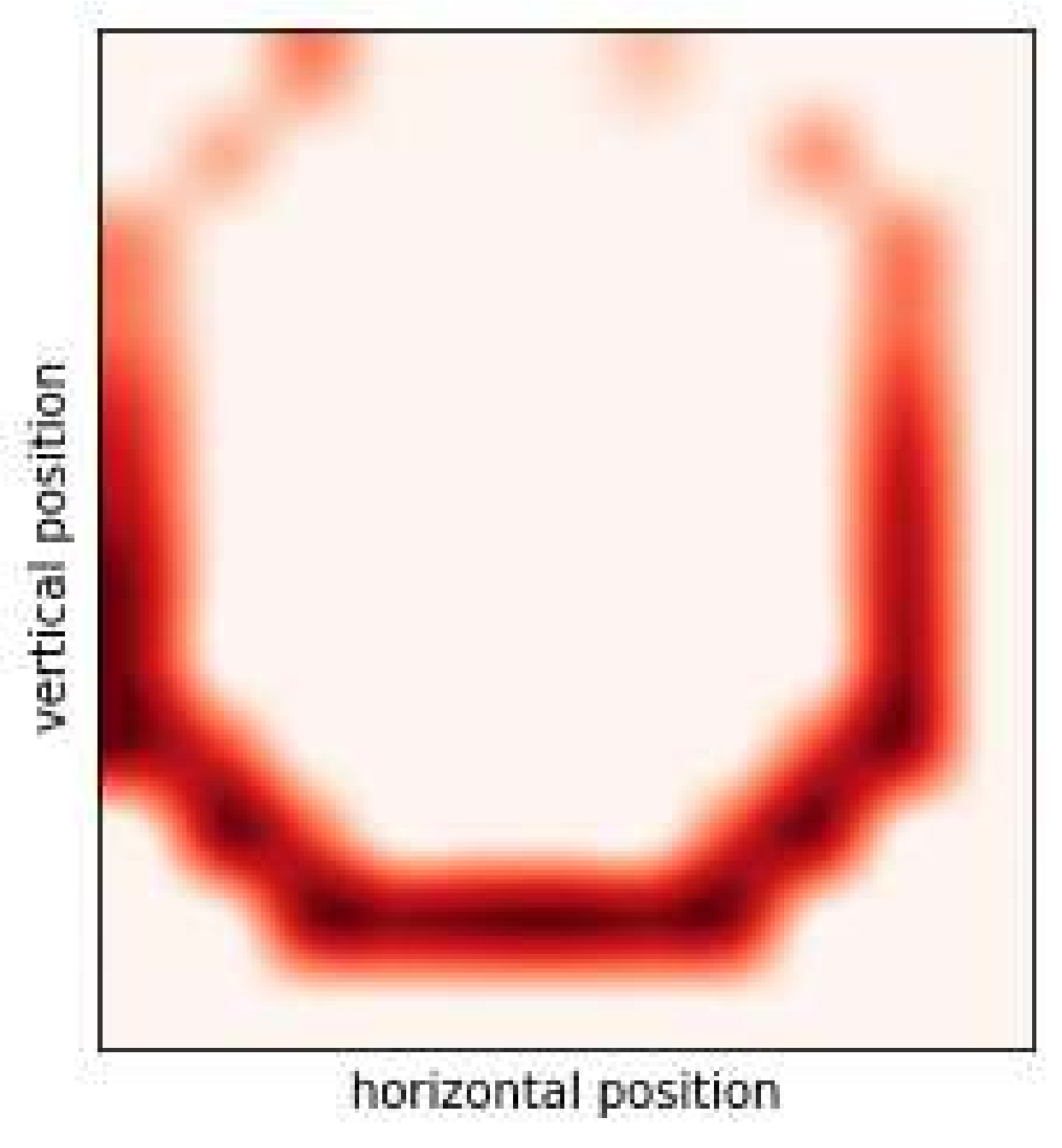}
	}
	\subfigure[MaxEnt]{
		\centering
		\includegraphics[width=.125\linewidth]{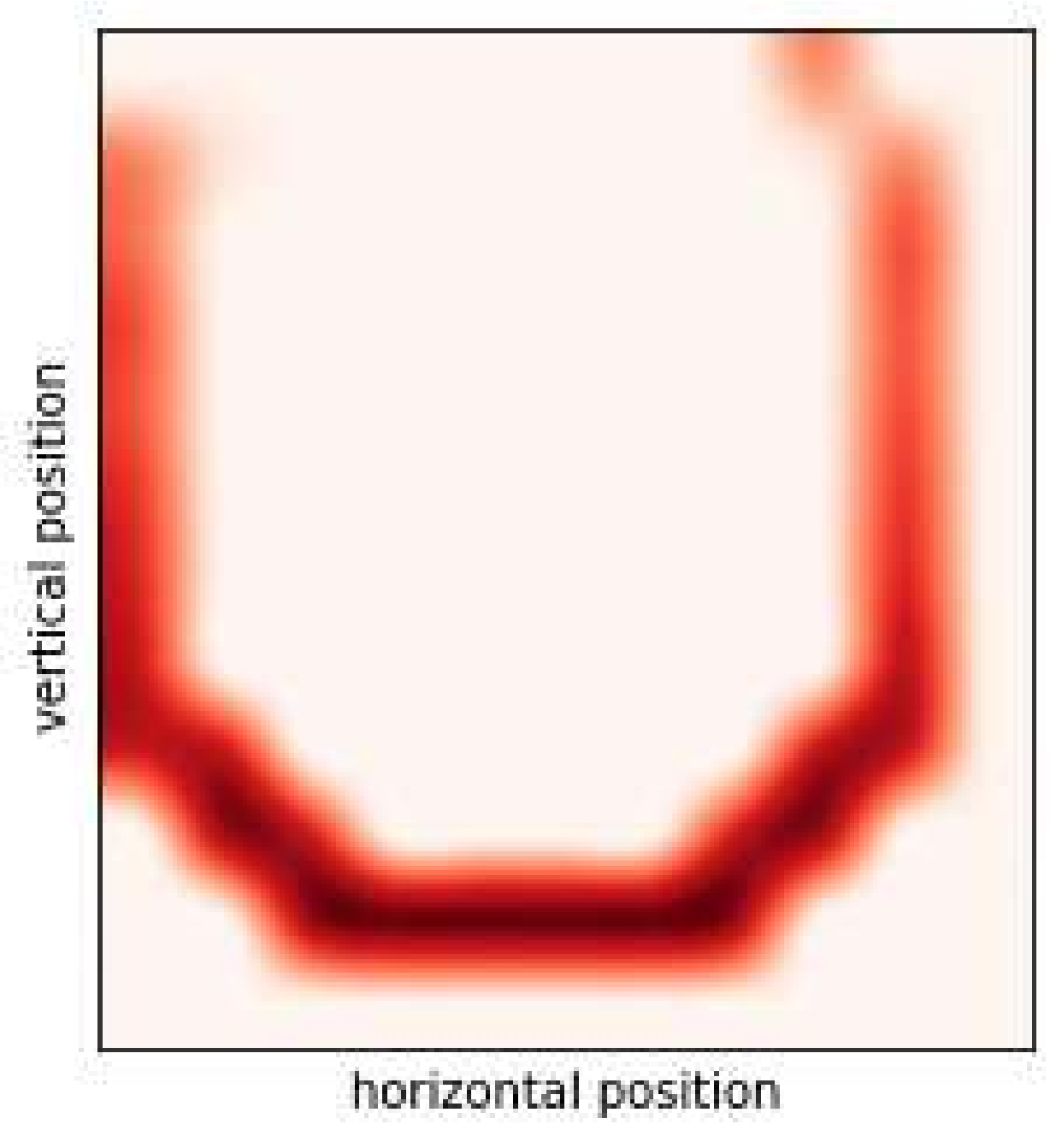}
	}
 	\subfigure[DTG]{
		\centering
		\includegraphics[width=.125\linewidth]{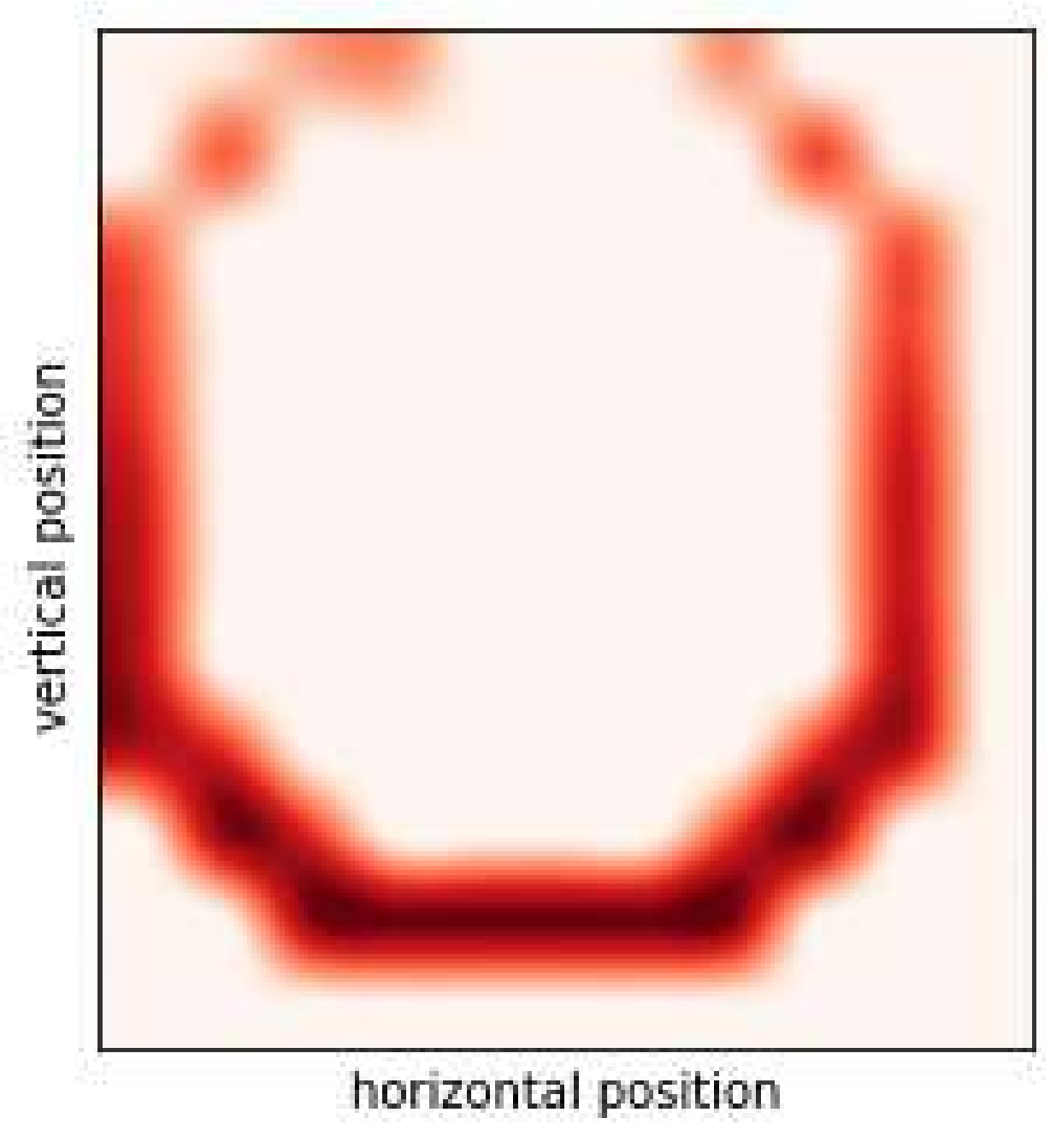}
	}
	\subfigure[DTG-KL]{
		\centering
		\includegraphics[width=.125\linewidth]{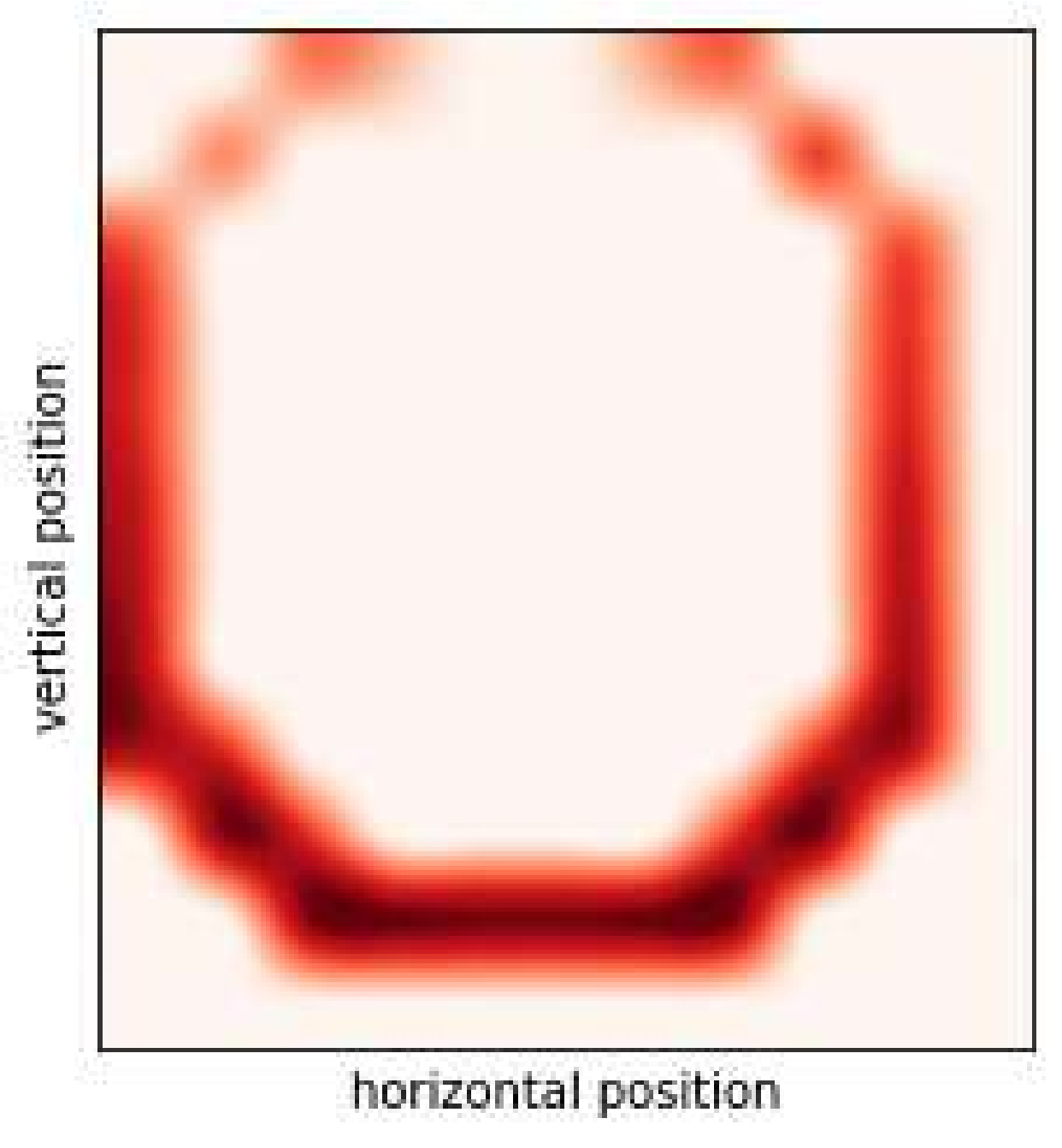}
	}
	\subfigure[DTG-MMD]{
		\centering
		\includegraphics[width=.125\linewidth]{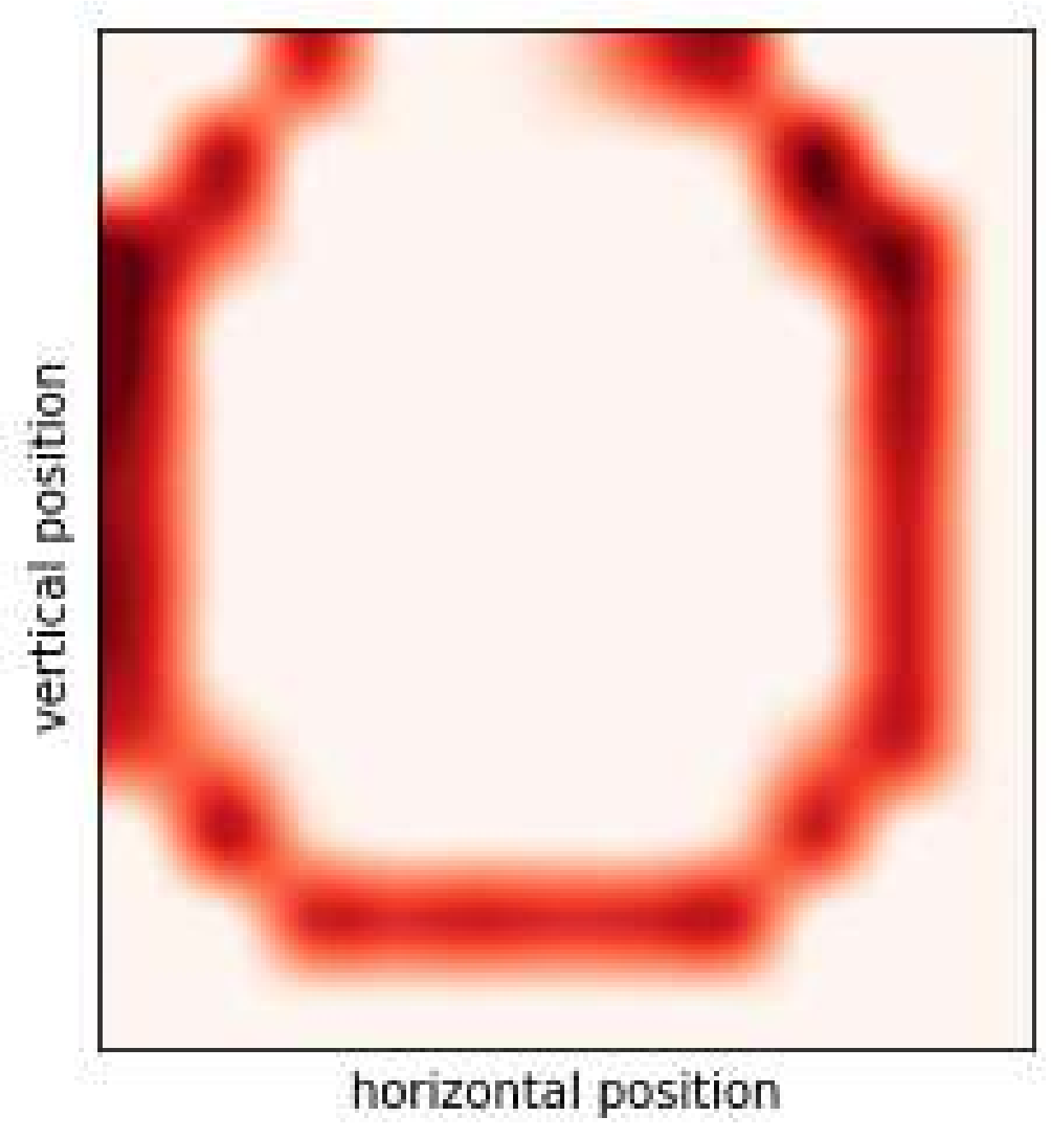}
	}
 	\subfigure[DTG-CS]{
		\centering
		\includegraphics[width=.125\linewidth]{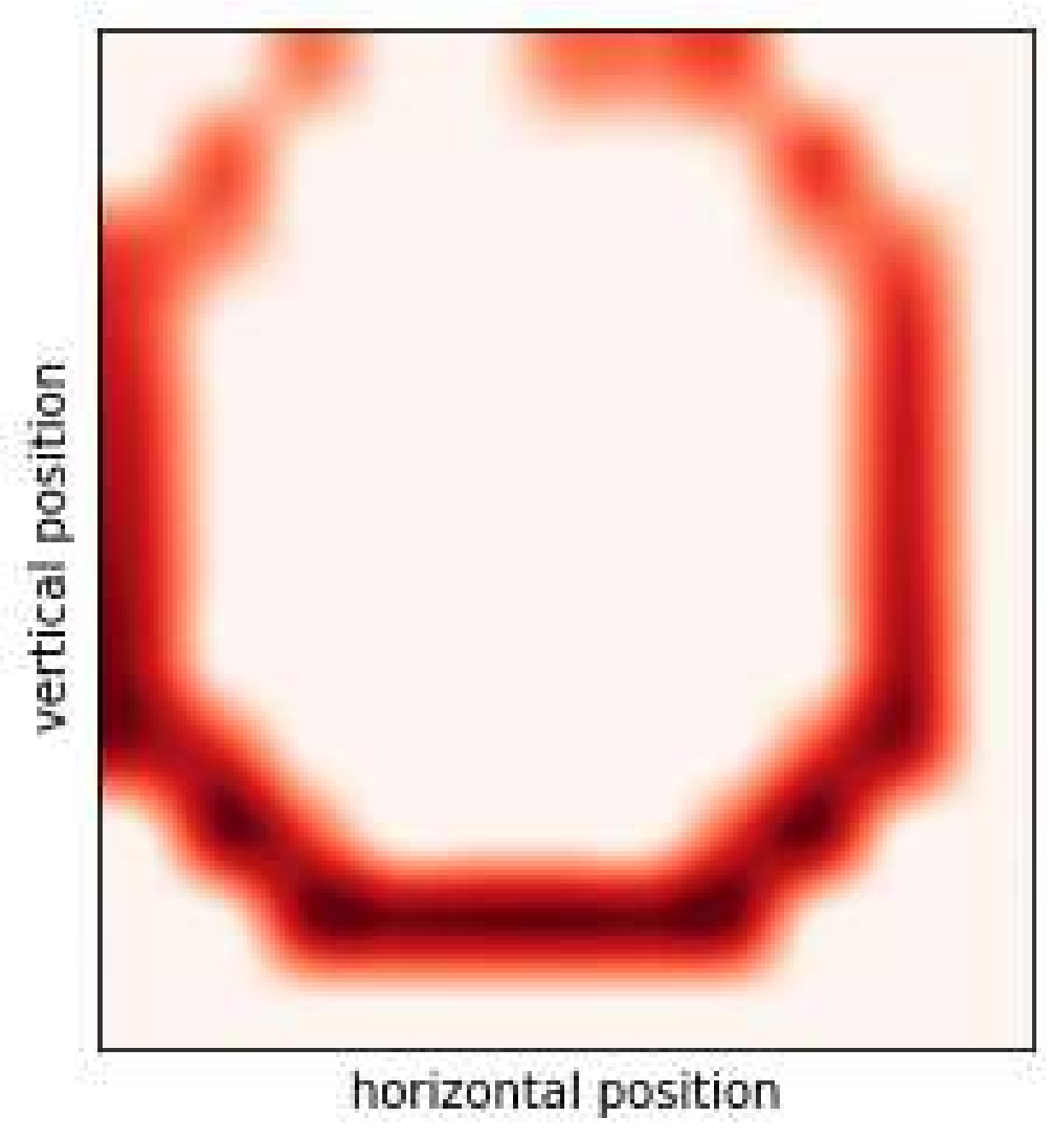}
	}\\
          \subfigure[Random]{
		\centering
		\includegraphics[width=.125\linewidth]{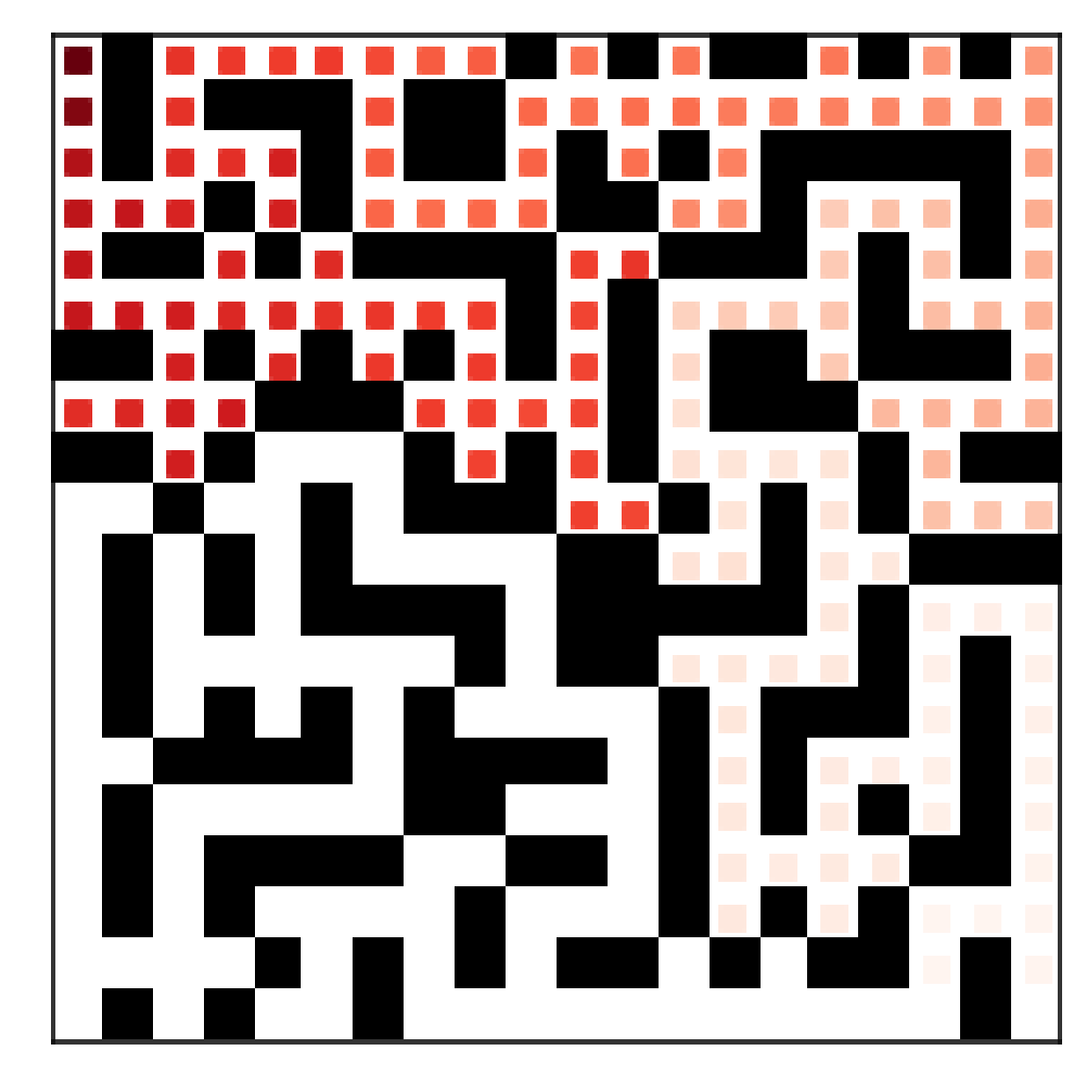}
	}
 	\subfigure[Q-learning]{
		\centering
		\includegraphics[width=.125\linewidth]{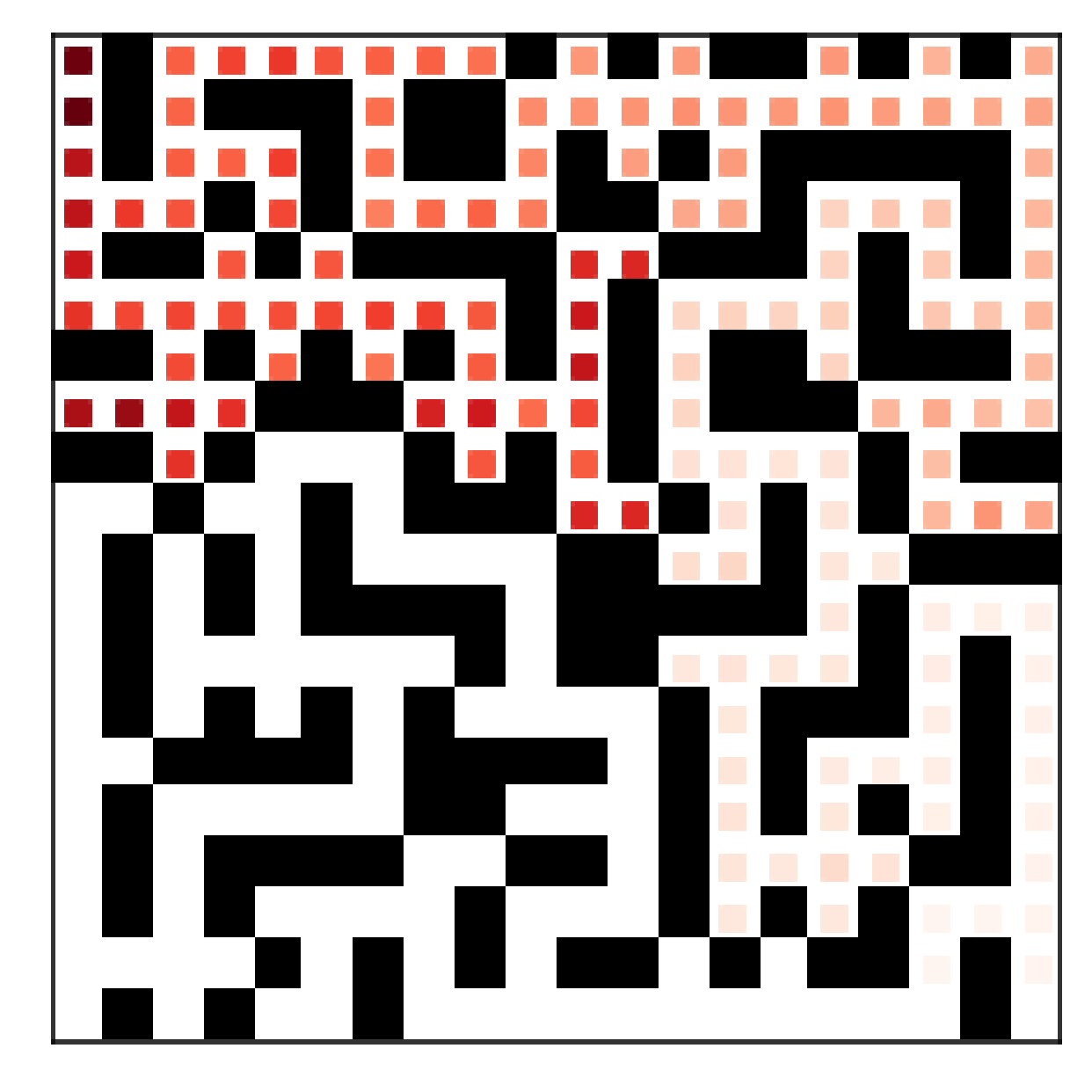}
	}
	\subfigure[MaxEnt]{
		\centering
		\includegraphics[width=.125\linewidth]{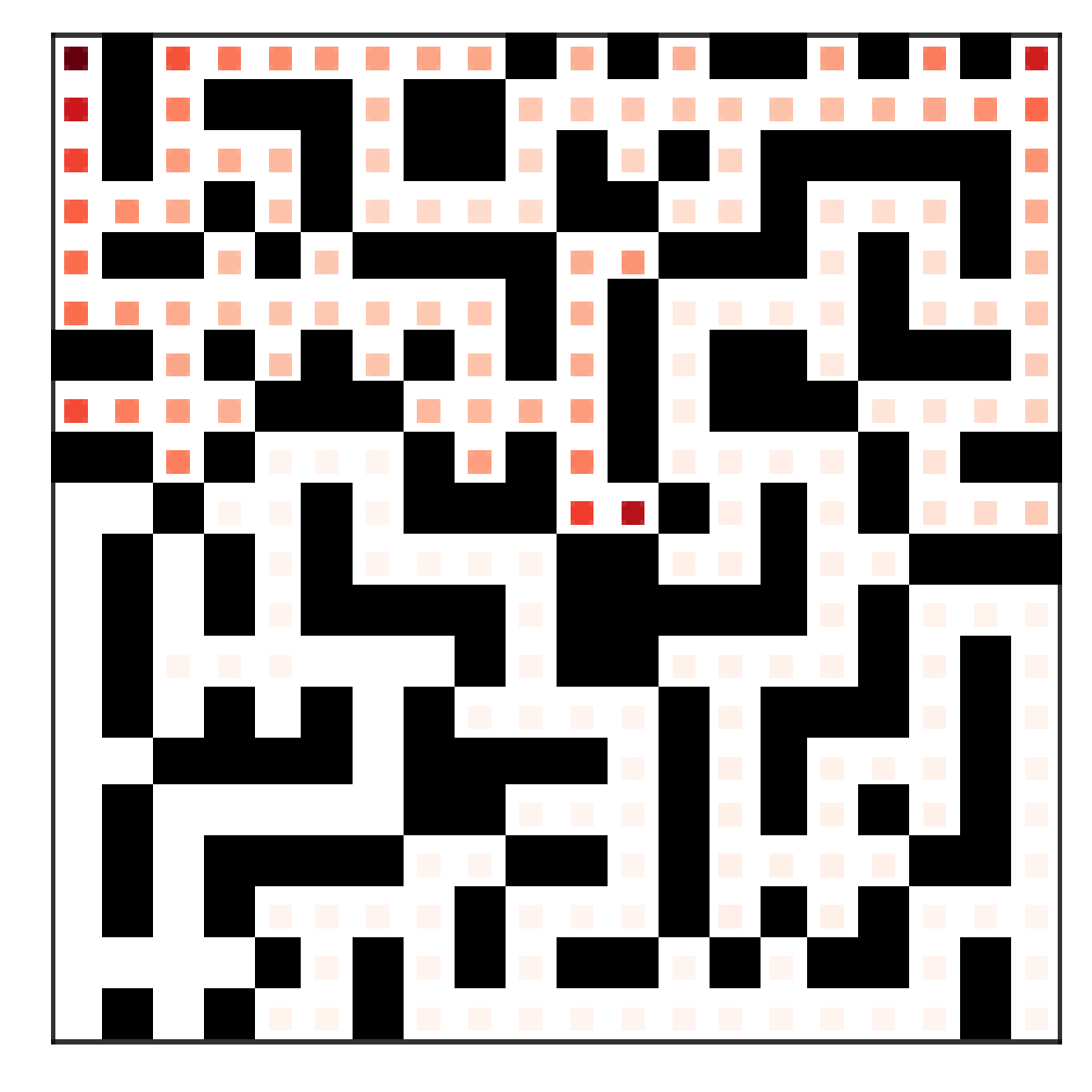}
	}
 	\subfigure[DTG]{
		\centering
		\includegraphics[width=.125\linewidth]{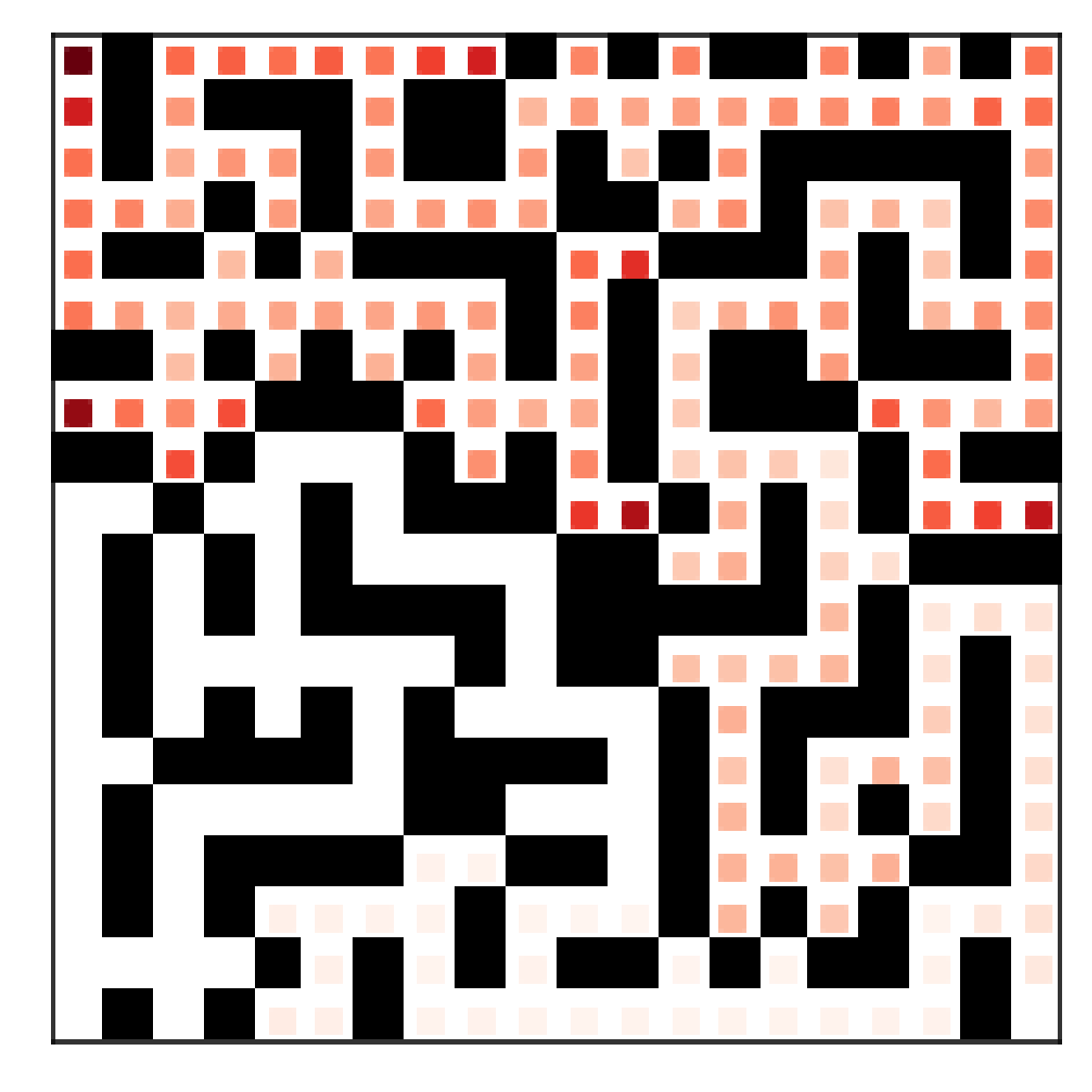}
	}
 	\subfigure[DTG-KL]{
		\centering
		\includegraphics[width=.125\linewidth]{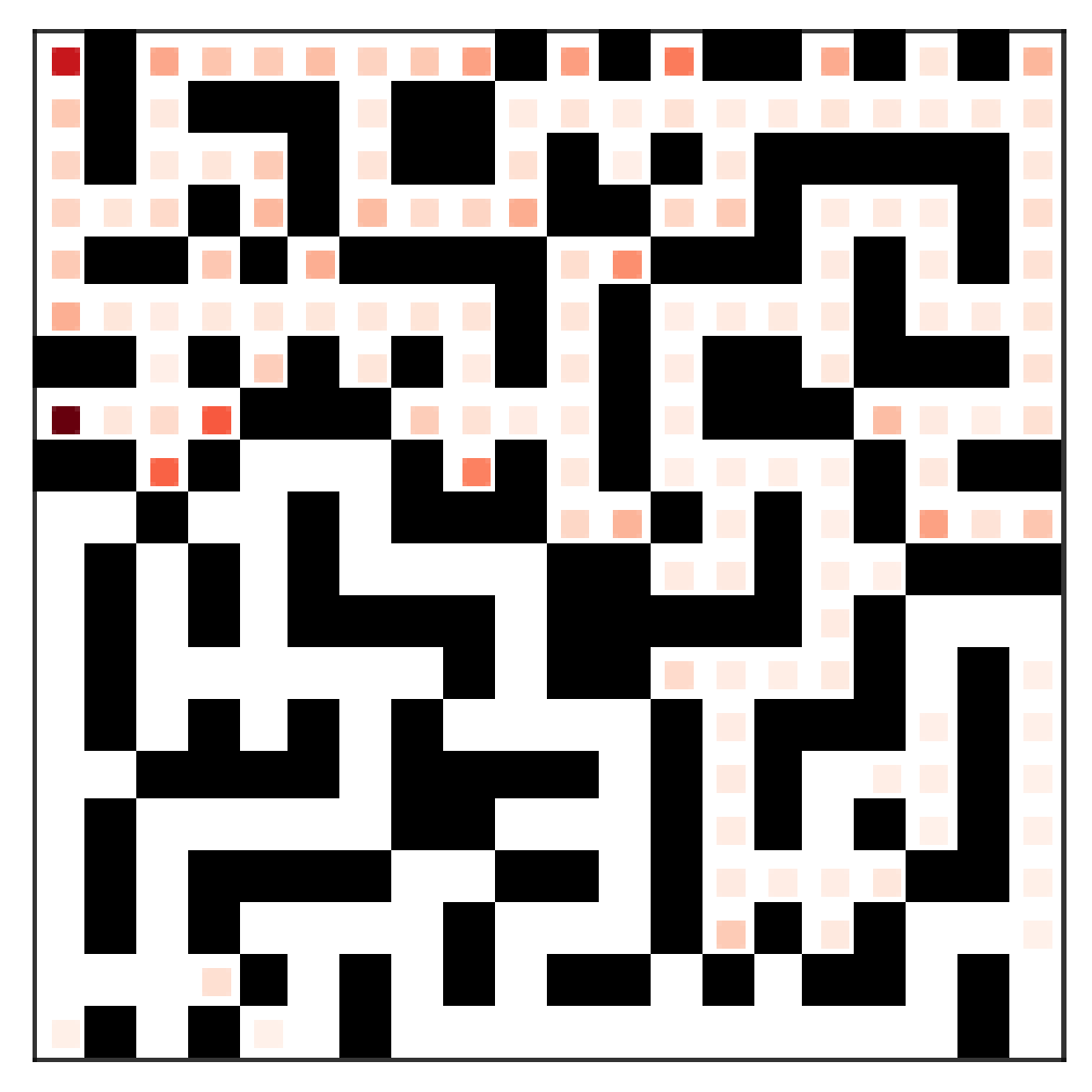}
	}
 	\subfigure[DTG-MMD]{
		\centering
		\includegraphics[width=.125\linewidth]{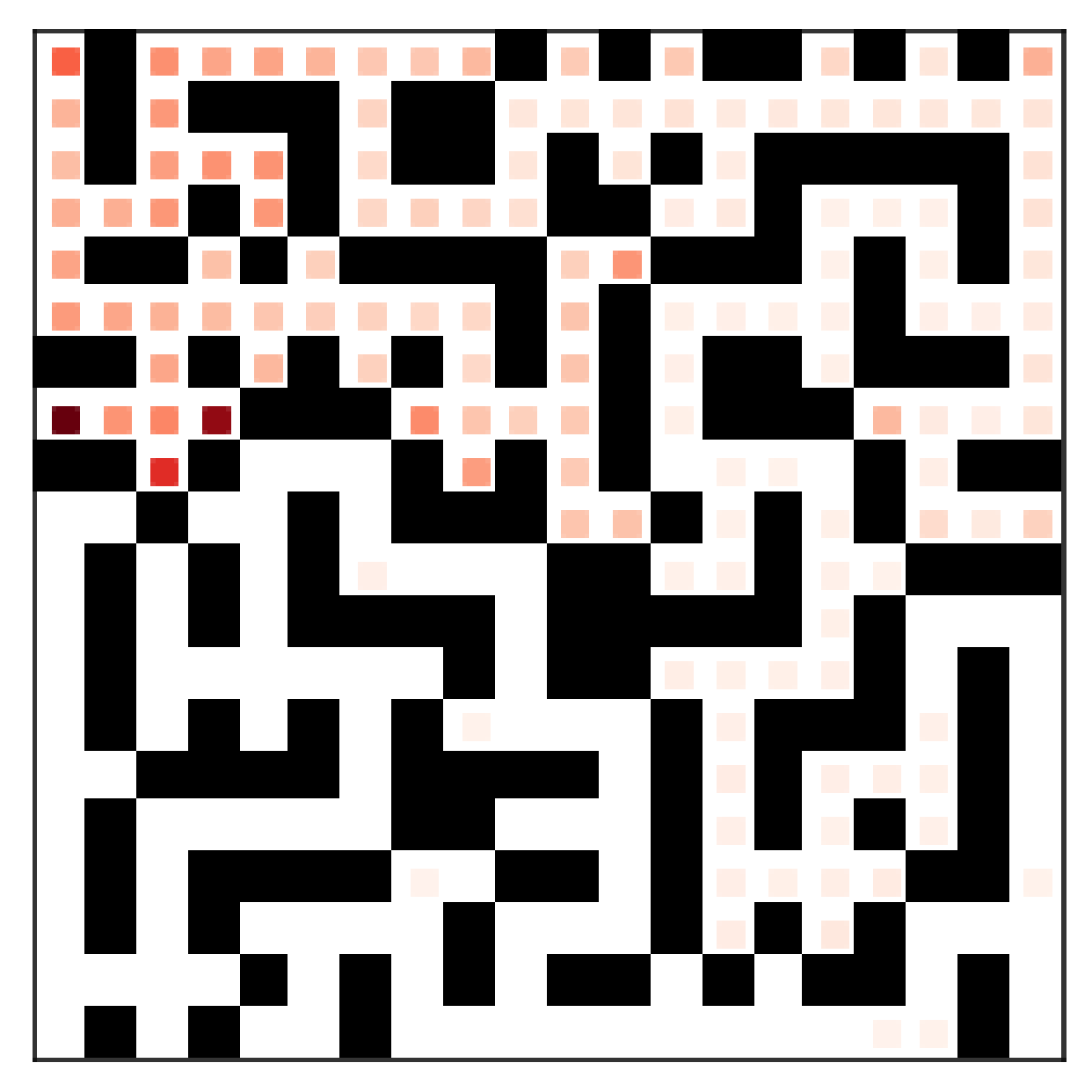}
	}
 	\subfigure[DTG-CS]{
		\centering
		\includegraphics[width=.125\linewidth]{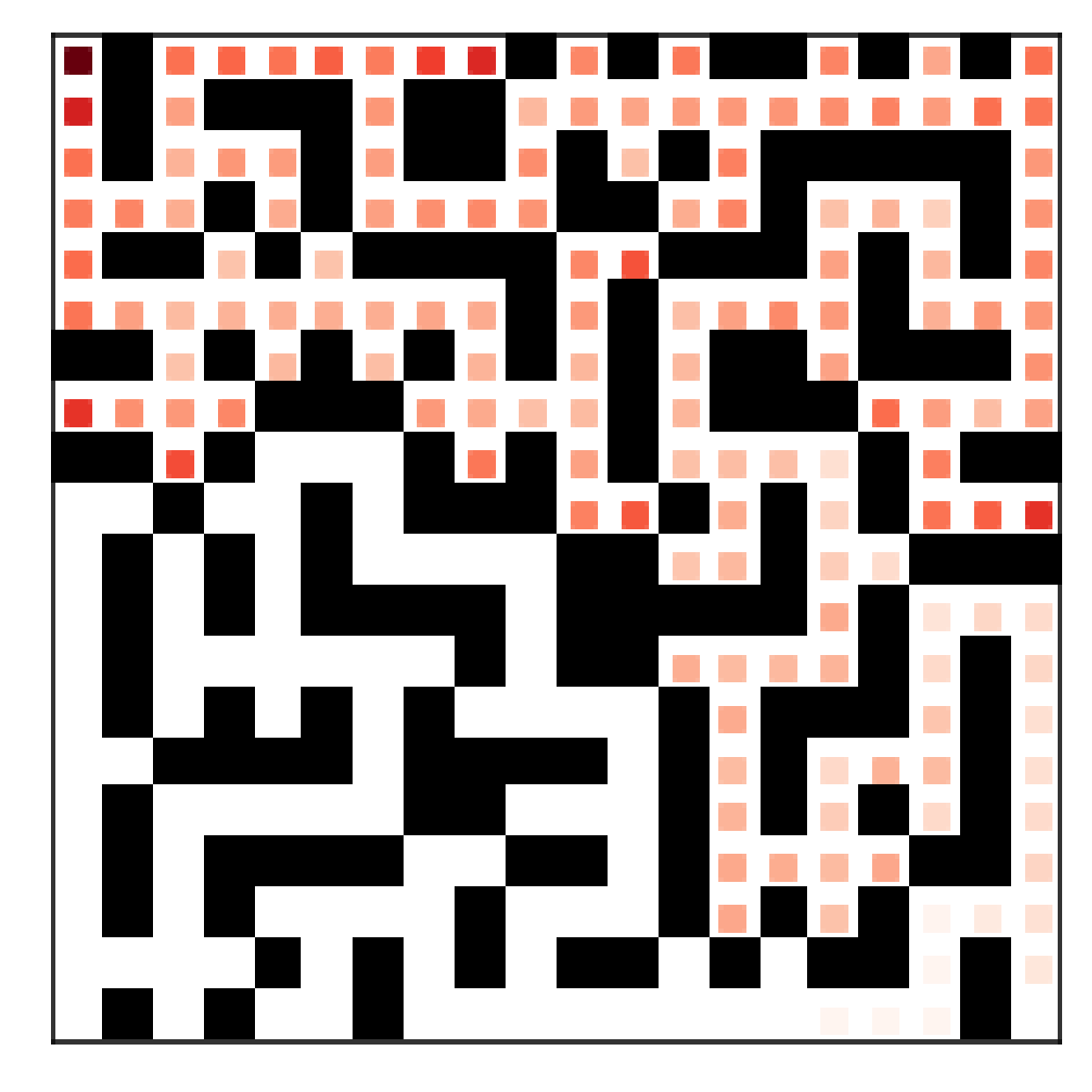}
	}
	\caption{The log-probability of occupancy of the two-dimensional state space in mountain car (first row), the log-probability of occupancy of x-y coordinates in Pendulum (second row) and the heatmap of traces for maze within $50,000$ training steps (third row). We repeat the experiment $100$ times with different random seeds and show the average results.}
	\label{fig:dtg_results}
\end{figure*}

\section{Conclusions and Implications for Future Work}\label{sec:conclusion}

We developed the conditional Cauchy-Schwarz (CS) divergence to quantify the closeness between two conditional distributions from samples, which can be elegantly evaluated with a kernel density estimator (KDE). Our conditional CS divergence enjoys simultaneously relatively lower computational complexity, differentiability, and faithfulness guarantee. The new divergence can be applied to a variety of time series and sequential decision making applications in a versatile way. With regard to time series clustering, it demonstrated obvious performance gain for multivariate or high-dimensional time series. With regard to reinforcement learning without explicit rewards, it outperforms the popular maximum entropy strategy and encourages significantly the exploration to states that have not been visited sufficiently for the agent to be familiar with it. We additionally analyzed two special cases of conditional CS divergence and illustrated their implications in other challenging areas such as time series causal discovery and the loss function design of deep regression models.

Our research opens up exciting avenues for future work, especially from an application perspective. Given the promising performance of causal score with CS divergence $(\text{CS})^{2}$ on benchmark simulated data in Fig.~\ref{fig:causal}, it would be interesting to systematically investigate its ability on causal discovery for realistic data in other domains, such as neuroscience, econometrics, and climate science. Meanwhile, it is compelling to apply our conditional CS divergence estimator in Eq.~(\ref{eq:conditional_CS_est}) and a by-product on sample efficient measure on conditional independence in Eq.~(\ref{eq:CS_ext2}) to different downstream tasks that aim to learn \emph{invariant} and \emph{fair} latent representations (see also discussions in Sec.~\ref{sec:3_2_2}).

We conclude this paper with an initial investigation of unsupervised domain adaptation (UDA) on EEG data. Our motivation is to inspire interested readers to jointly explore broader applications of the conditional CS divergence and uncover its undiscovered theoretical properties.

In UDA, we have $M$ labeled samples $\mathcal{D}_s = \{(\mathbf{x}_s^i, y_s^i)\}_{i=1}^{M}$ from a source domain with distribution $p_s$ and $N$ unlabeled samples $\mathcal{D}_t = \{\mathbf{x}_t^j\}_{j=1}^{N}$ from a target domain with distribution $p_t$ ($p_t \neq p_s$). The primary goal of UDA is to learn a neural network $h_\theta=f\circ g$ such that the risk on the target domain is minimized, where $f: \mathbf{x} \mapsto  \mathbf{z}$ is a feature extractor and $g: \mathbf{z} \mapsto  y$ is a classifier. We denote the prediction by $\hat{y}=g(\mathbf{z})$. According to~\cite{nguyen2022kl}, the loss $l_{\text{test}}$ in the test distribution (a.k.a., target domain) is upper bounded by:
\begin{equation}\label{eq:KL_bound}
l_{\text{test}} \leq l_{\text{train}} + \frac{M}{\sqrt{2}} \sqrt{ D_{\text{KL}} (p_t (\mathbf{z}); p_s(\mathbf{z}) ) + D_{\text{KL}} (p_t (y|\mathbf{z}); p_s(y|\mathbf{z})) },
\end{equation}
in which $l_{\text{train}}$ is the loss in source domain, and it is assumed that $-\log \hat{p}(y|\mathbf{z})$ is upper bounded by a constant $M$\footnote{In classification, one can enforce this condition easily by augmenting the output softmax of the classifier so that each class probability is always at least $\exp(-M)$~\cite{nguyen2022kl}. For example, if we choose $M = 4$, then $\exp(-M) \approx 0.02$.}.

Eq.~(\ref{eq:KL_bound}) implies that achieving small test error necessitates matching both the marginal distribution $p(\mathbf{z})$ and the conditional distribution $p(y|\mathbf{z})$, which aligns with \cite{zhao2019learning}. Motivated by this theoretical basis, it is reasonable to replace KL divergence (due to the difficulty of estimation) with our CS divergence and optimize the following objective\footnote{Replacing the KL divergence with the CS divergence also results in a tighter generalization error bound, as thoroughly discussed in~\cite{yindomain2024}.}:
\begin{equation}
L_{\text{CE}} + \alpha \left[ D_{\text{CS}} (p_t (\mathbf{z}); p_s(\mathbf{z}) ) + D_{\text{CS}} (p_t (\hat{y}|\mathbf{z}); p_s(y|\mathbf{z})) \right],
\end{equation}
where $L_{\text{CE}}=-\frac{1}{M}\sum_{i=1}^M y_s^i \log \hat{y}_s^i $ is the cross-entropy loss on source domain and $\alpha>0$ controls the strength of regularization.
Note that, in target domain, there is no ground truth $y_t$. Hence, we approximate $y_t$ with its prediction $\hat{y}_t=h_\theta(x_t)$, which is a common trick in UDA literature.

We apply CS divergence-based domain adaptation to EEG classification on two benchmark real-world datasets, namely  BCI Competition IIIb~\cite{blankertz2006bci} and SEED~\cite{zheng2015investigating}. We also compare our method with representative MMD-based approaches, which include DDC~\cite{tzeng2014deep}, DAN~\cite{long2015learning}, JAN~\cite{long2017deep}, DJP-MMD~\cite{zhang2020discriminative}. The baseline model is EEGNet~\cite{lawhern2018eegnet}, and $\mathbf{z}$ is selected to be the latent representation before the final linear classification layer. A detailed discussion of experimental setup and the difference amongst all competing approaches are provided in Appendix~\ref{sec:UDA}. According to Table~\ref{tab:UDA}, our CS divergence always outperforms the best MMD-based approaches.

\begin{table}[]
\centering
\caption{Classification accuracy for different approaches.}
\begin{tabular}{|c|c|c|c|}
\hline
Methods    & S4 $\rightarrow$ X11 & X11 $\rightarrow$ S4 & SEED \\ \hline
EEGNet (no adaptation)     & 0.608   & 0.718  &  0.515  \\ \hline
DDC     &  0.610   & 0.766  &  0.564  \\ \hline
DAN & 0.615  & 0.768  &  0.547  \\ \hline
JAN & 0.609 & 0.771  &  0.543  \\ \hline
DJP-MMD    & 0.628  & 0.768   &  0.536  \\ \hline
\textbf{CS} (\textbf{ours})  & $\mathbf{0.631}$  & $\mathbf{0.801}$  & $\mathbf{0.578}$ \\ \hline
\end{tabular}
\label{tab:UDA}
\end{table}




\bibliographystyle{elsarticle-num}
\bibliography{CS_divergence}

\newpage
\onecolumn
\appendices

\section{Background Knowledge}
\subsection{On Kernel Functions in Multivariate Kernel Density Estimation}
Kernel density estimation (KDE) can be extended to estimate multivariate densities $f$ in $\mathbb{R}^d$ based on the same principle: perform an average of densities ``centered" at the data points~\cite{Garcia-Portugues2024}. Specifically, given $M$ samples $X_1,X_2,\cdots,X_M$ in $\mathbb{R}^d$, the KDE of $f$ evaluated on an arbitrary point $\mathbf{x}$ is given by:
\begin{equation}
\label{eq:KDE_multivariate}
\hat{f}(\mathbf{x},\Sigma) := \frac{1}{M |\Sigma|^{1/2}} \sum_{i=1}^M K( \Sigma^{-1/2} (\mathbf{x}-X_i) ),
\end{equation}
where $K$ is a $d$-variate symmetric kernel that is unimodal at $\mathbf{0}$ and that depends on the bandwidth matrix $\Sigma$, a symmetric positive definite (SPD) matrix. Then, the bandwidth matrix $\Sigma$ can be thought of as the variance-covariance matrix of a multivariate normal density with mean $X_i$ and the KDE in Eq.~(\ref{eq:KDE_multivariate}) can be regarded as a data-driven mixture of those densities~\cite{Garcia-Portugues2024}.

A common notation is $K_\Sigma(\mathbf{z}):=|\Sigma|^{-1/2} K(\Sigma^{-1/2}\mathbf{z})$, so the KDE can be compactly written as:
\begin{equation}
\hat{f}(\mathbf{x},\Sigma) := \frac{1}{M} \sum_{i=1}^M K_\Sigma (\mathbf{x}-X_i).
\end{equation}

Considering a full bandwidth matrix $\Sigma$ gives more flexibility to the KDE, but also quadratically increases the amount of bandwidth parameters ($\frac{d(d+1)}{2}$ in total) that need to be chosen precisely, which notably complicates bandwidth selection as the dimension $d$ grows, and increases the variance of the KDE. A common simplification is to consider a diagonal bandwidth matrix $\Sigma = \diag (\sigma_1^2,\cdots,\sigma_d^2)$, which yields the KDE employing product kernels~\cite{Garcia-Portugues2024,li2023nonparametric}:
\begin{equation}
\label{eq:KDE_product_kernel}
\hat{f}(\mathbf{x},\Sigma) := \frac{1}{M} \sum_{i=1}^M \kappa_{\sigma_1}(x_1-X_{i,1}) \times \kappa_{\sigma_2}(x_2-X_{i,2}) \times \cdots \times \kappa_{\sigma_d}(x_d-X_{i,d}),
\end{equation}
where $\mathbf{x} = [x_1, x_2,\cdots,x_d]^T$, $X_i = [X_{i,1}, X_{i,2},\cdots,X_{i,d}]^T$, $\mathbf{\sigma} = [\sigma_1, \sigma_2,\cdots,\sigma_d]^T$ is the vector of bandwidths, $\kappa_\sigma(\cdot)$ is a univariate kernel function such as Gaussian $\kappa_{\sigma}(\cdot)=\frac{1}{\sqrt{2\pi}\sigma} \exp(-\frac{\|\cdot\|^2}{2\sigma^2})$.

Note that, the only assumption in Eq.~(\ref{eq:KDE_product_kernel}) was that the samples are independent. That is, no restrictions were placed on the $s$ index for each dimension $X_{i,s} (s=1,2,\cdots,d)$. The product kernel is used simply for convenience, and it certainly \emph{does not} require that the $X_{i,s}$'s are independent across the $s$ index. In other words, the multivariate KDE in Eq.~(\ref{eq:KDE_product_kernel}) is capable of capturing general dependence among the different dimensions of $X_i$~\cite[Chapter 1.6]{li2023nonparametric}.

In practice, a much simpler and common choice is to consider $\sigma=\sigma_1=\sigma_2=\cdots=\sigma_d$.

\subsection{\emph{Resubstitution Estimator} and \emph{Plug-in Estimator}}\label{sec:two_estimators}
The CS divergence involves the estimation of the inner product of two density functions. In fact, quantities like $\int p^2 d\mu$ and $\int pq d\mu$ can be estimated in a couple of ways~\cite{beirlant1997nonparametric}. In this paper, we use a so-called \emph{resubstitution estimator}; whereas the authors in \cite{jenssen2006cauchy} use the \emph{plug-in estimator} that simply inserts KDE of the density into the formula, i.e.,
\begin{equation}
\begin{split}
        \int p^2 d\mu \approx \int \widehat{p}^2(x)dx & = \int \left(
\frac{1}{N}\sum_{i=1}^N \kappa_\sigma(x_i-x) \right)^2 dx \\
& = \frac{1}{N^2} \sum_{i=1}^N \sum_{j=1}^N \int \kappa_\sigma(x_i-x)\times \kappa_\sigma(x_j-x) dx.
\end{split}
\end{equation}

Authors of \cite{jenssen2006cauchy} then assume a Gaussian kernel and rely on the property that the integral of the product of
two Gaussians is exactly evaluated as the value of the Gaussian computed at the difference of the arguments and whose variance is the sum of the variances of the two original Gaussian functions. Hence,
\begin{equation}\label{eq:106}
    \int p^2 d\mu  \approx \frac{1}{N^2} \sum_{i=1}^N \sum_{j=1}^N \int \kappa_\sigma(x_i-x)\times \kappa_\sigma(x_j-x) dx
     = \frac{1}{N^2} \sum_{i=1}^N \sum_{j=1}^N \kappa_{\sqrt{2}\sigma}(x_i-x_j).
\end{equation}

To our knowledge, other kernel functions, however, do not result in such a convenient evaluation of the integral because the Gaussian maintains the functional form under convolution.

By contrast, we estimate $\int p^2d\mu$ as:
\begin{equation}\label{eq:107}
        \int p^2 d\mu = \mathbb{E}_p(p) = \frac{1}{N} \sum_{i=1}^N p(x_i)
        = \frac{1}{N} \sum_{i=1}^N \left(
\frac{1}{N} \sum_{j=1}^N \kappa_\sigma(x_i-x_j) \right) dx
 = \frac{1}{N^2} \sum_{i=1}^N \sum_{j=1}^N \kappa_{\sigma}(x_i-x_j).
\end{equation}

Although Eq.~(\ref{eq:107}) only differs from Eq.~(\ref{eq:106}) by replacing $\sqrt{2}\sigma$ with $\sigma$,
our estimator offers two immediate advantages over that in \cite{jenssen2006cauchy}: 1) our estimator is generalizable to all valid kernel functions; and 2) the empirical estimator of conditional CS divergence can be achieved much more easily using only the \emph{resubstitution estimator}.

\subsection{CS Divergence and its Connection to MMD} \label{sec:CS_and_MMD}

Given $M$ samples $\{\mathbf{y}_i^p\}_{i=1}^M$ drawn from distribution $p$ and $N$ samples $\{\mathbf{y}_j^q\}_{j=1}^N$ drawn from distribution $q$. By the kernel density estimation (KDE) with a kernel function $\kappa_\sigma$ of width $\sigma$\footnote{For example, the most popular Gaussian kernel function is expressed as $\kappa_{\sigma}(\cdot)=\frac{1}{\sqrt{2\pi}\sigma}\exp(-\frac{\|\cdot\|^2}{2\sigma^2})$.}, we have:
\begin{equation}\label{eq:appendix_KDE1}
\hat{p}(\mathbf{y}) = \frac{1}{M} \sum_{i=1}^M \kappa_\sigma (\mathbf{y}-\mathbf{y}_i^p),
\end{equation}
and
\begin{equation}\label{eq:appendix_KDE2}
\hat{q}(\mathbf{y}) = \frac{1}{N} \sum_{j=1}^N \kappa_\sigma (\mathbf{y}-\mathbf{y}_j^q).
\end{equation}

To evaluate the dissimilarity between $p$ and $q$, a natural choice is the Euclidean distance:
\begin{equation}
\begin{split}
D_{\text{ED}} (p;q) & = \int (p(\mathbf{y}) - q(\mathbf{y}))^2 dy \\
 & = \int p^2(\mathbf{y})d\mathbf{y} + \int q^2(\mathbf{y})d\mathbf{y} - 2\int p(\mathbf{y})q(\mathbf{y}) d\mathbf{y}
\end{split}
\end{equation}

By Eq.~(\ref{eq:107}), we have:
\begin{equation}
\label{eq:appendix_KDE3}
\int p^2(\mathbf{y})d\mathbf{y} = \mathbb{E}_{p}(p) = \frac{1}{M^2} \sum_{i=1}^M\sum_{i'=1}^M \kappa_{\sigma} (\mathbf{y}_i^p-\mathbf{y}_{i'}^p).
\end{equation}

Similarly,
\begin{equation}
\int q^2(\mathbf{y})d\mathbf{y} = \mathbb{E}_{q}(q) = \frac{1}{N^2} \sum_{j=1}^N\sum_{j'=1}^N \kappa_{\sigma} (\mathbf{y}_j^q-\mathbf{y}_{j'}^q),
\end{equation}
and
\begin{equation}\label{eq:appendix_KDE5}
\int p(\mathbf{y})q(\mathbf{y})d\mathbf{y} = \mathbb{E}_{p}(q) =  \frac{1}{MN} \sum_{i=1}^M\sum_{j=1}^N \kappa_{\sigma} (\mathbf{y}_j^q-\mathbf{y}_i^p).
\end{equation}

Combining Eqs.~(\ref{eq:appendix_KDE3})-(\ref{eq:appendix_KDE5}), we have:
\begin{equation}\label{eq:appendix_ED}
D_{\text{ED}} (p;q ) = \underbrace{\frac{1}{M^2} \sum_{i,i'=1}^M \kappa_{\sigma} (\mathbf{y}_i^p-\mathbf{y}_{i'}^p)}_{\encircle{A}}
+ \underbrace{\frac{1}{N^2} \sum_{j,j'=1}^N \kappa_{\sigma} (\mathbf{y}_j^q-\mathbf{y}_{j'}^q)}_{\encircle{B}}
- \underbrace{\frac{2}{MN} \sum_{i,j=1}^{M,N} \kappa_{\sigma} (\mathbf{y}_j^q-\mathbf{y}_i^p)}_{\encircle{C}}.
\end{equation}

Note that Eq.~(\ref{eq:appendix_ED}) is exactly the same (in terms of mathematical expression) to the square of MMD using \emph{V}-statistic estimator~\cite{gretton2012kernel}:
\begin{equation}
\widehat{\text{MMD}}_{v} [p (\mathbf{x}),q (\mathbf{x})]
 = \left[ \frac{1}{M^2}\sum_{i,i'=1}^M \kappa(\mathbf{y}_i^p,\mathbf{y}_{i'}^p) + \frac{1}{N^2}\sum_{j,j'=1}^N \kappa(\mathbf{y}_j^q,\mathbf{y}_{j'}^q)
 - \frac{2}{MN}\sum_{i,j=1}^{M,N}\kappa(\mathbf{y}_j^q,\mathbf{y}_i^p) \right]^{\frac{1}{2}}.
\end{equation}

Although, in practice, the \emph{U}-statistic estimator of MMD is more widely used:
\begin{equation}
\widehat{\text{MMD}}_{u} [p_s (\mathbf{x}),p_t (\mathbf{x})]
 = \left[ \frac{1}{M(M-1)}\sum_{i=1}^M \sum_{i'\neq i}^M \kappa(\mathbf{y}_i^p,\mathbf{y}_{i'}^p) + \frac{1}{N(N-1)}\sum_{j=1}^N\sum_{j'\neq j}^N \kappa(\mathbf{y}_j^q,\mathbf{y}_{j'}^q)
 - \frac{2}{MN}\sum_{i=1}^{M}\sum_{j=1}^{N} \kappa(\mathbf{y}_j^q,\mathbf{y}_i^p) \right]^{\frac{1}{2}},
\end{equation}

\textcolor{black}{On the other hand, by substituting Eqs.~(\ref{eq:appendix_KDE3})-(\ref{eq:appendix_KDE5}) into the definition of CS divergence ($D_{\text{CS}} (p\|q) = -2\log(\int p(\mathbf{y})q(\mathbf{y})d\mathbf{y}) + \log(\int p(\mathbf{y})^2 d\mathbf{y}) + \log(\int q(\mathbf{y})^2 d\mathbf{y})$)}, we obtain:
\begin{equation}
\widehat{D}_{\text{CS}} (p\|q) = \log\left(\frac{1}{M^2}\sum_{i,i'=1}^M \kappa_{\sigma}(\mathbf{y}_i^p-\mathbf{y}_{i'}^p)\right) +
\log\left(\frac{1}{N^2}\sum_{j,j'=1}^N \kappa_{\sigma}(\mathbf{y}_j^q-\mathbf{y}_{j'}^q)\right)
-2 \log\left(\frac{1}{MN}\sum_{i=1}^M \sum_{j=1}^N \kappa_{\sigma}(\mathbf{y}_j^q-\mathbf{y}_i^p)\right).
\end{equation}

To summarize, the mathematical expression of MMD can be derived by either taking the distance of kernel mean embedding in a reproducing kernel Hilbert space (RKHS) or taking the Euclidean distance of two distributions which are estimated by KDE. More interestingly, one can estimate MMD by $\left( \encircle{A}+\encircle{B}-2\encircle{C} \right)^\frac{1}{2}$, and CS divergence by $\log\left(\encircle{A}\right)+ \log\left(\encircle{B}\right)-2\log\left(\encircle{C}\right)$. Note, however, that this observation does not apply to conditional MMD and conditional CS divergence.

\section{Proofs} \label{sec:appendix_proof}

\subsection{Proof to Proposition 1}

\begin{proof}
Note that, the Cauchy-Schwarz inequality can be generalized to double integrals~\cite{steele2004cauchy}. Specifically, for $f: \mathcal{X}\times\mathcal{Y} \rightarrow \mathbb{R}$, and $g: \mathcal{X}\times\mathcal{Y} \rightarrow \mathbb{R}$, then the double integrals
\begin{equation*}
A = \iint_{\mathcal{X}\times\mathcal{Y}} f^2 dxdy,\quad B = \iint_{\mathcal{X}\times\mathcal{Y}} fg dxdy, \quad C = \iint_{\mathcal{X}\times\mathcal{Y}} g^2 dxdy
\end{equation*}
must satisfy the inequality:
\begin{equation}
\label{eq:double_CS}
|B| \leq \sqrt{A}\cdot\sqrt{C},
\end{equation}
and the inequality is strict unless $f$ and $g$ are proportional, i.e., $f=\lambda g$.

Let $f=p(y|x)$ and $g=q(y|x)$. According to Eq.~(\ref{eq:double_CS}), we have:
\begin{equation}
\label{eq:conditional_CS_eq1}
\iint p(y|x)q(y|x)dxdy \leq \sqrt{\left(\iint p^2(y|x) dxdy \right) \left(\iint q^2(y|x) dxdy \right)}.
\end{equation}

This equality holds when $p(y|x)=\lambda q(y|x)$ for a scalar $\lambda$, which implies linear dependence between $p(y|x)$ and $q(y|x)$.

On the other hand, $p(y|x)$ and $q(y|x)$ are valid probability density functions in the sense that $1 = \int p(y|x)dy = \lambda \int q(y|x)dy = \lambda$, which implies that $\lambda=1$.
\end{proof}

\subsection{Proof to Proposition 2}

The conditional CS divergence for $p(\mathbf{y}|\mathbf{x})$ and $q(\mathbf{y}|\mathbf{x})$ is expressed as:
\begin{equation}
D_{\text{CS}}(p(\mathbf{y}|\mathbf{x});q(\mathbf{y}|\mathbf{x}))
 = - 2 \log \left(\int_\mathcal{X}\int_\mathcal{Y} \frac{p(\mathbf{x},\mathbf{y})q(\mathbf{x},\mathbf{y})}{p(\mathbf{x})q(\mathbf{x})} d\mathbf{x}d\mathbf{y} \right)
+ \log \left(\int_\mathcal{X}\int_\mathcal{Y} \frac{p^2(\mathbf{x},\mathbf{y})}{p^2(\mathbf{x})} d\mathbf{x}d\mathbf{y}\right) + \log \left(\int_\mathcal{X}\int_\mathcal{Y} \frac{q^2(\mathbf{x},\mathbf{y})}{q^2(\mathbf{x})} d\mathbf{x}d\mathbf{y}\right),
\end{equation}
which contains two conditional quadratic terms (i.e., $\int_\mathcal{X}\int_\mathcal{Y} \frac{p^2(\mathbf{x},\mathbf{y})}{p^2(\mathbf{x})} d\mathbf{x}d\mathbf{y}$ and $\int_\mathcal{X}\int_\mathcal{Y} \frac{q^2(\mathbf{x},\mathbf{y})}{q^2(\mathbf{x})} d\mathbf{x}d\mathbf{y}$) and a cross term (i.e., $\int_\mathcal{X}\int_\mathcal{Y} \frac{p(\mathbf{x},\mathbf{y})q(\mathbf{x},\mathbf{y})}{p(\mathbf{x})q(\mathbf{x})} d\mathbf{x}d\mathbf{y}$).

Assume we are given observations $\psi_p=\{(\mathbf{x}_i^p,\mathbf{y}_i^p )\}_{i=1}^M$ and $\psi_q=\{(\mathbf{x}_j^q,\mathbf{y}_j^q )\}_{j=1}^N$, sampled from distributions $p(\mathbf{x},\mathbf{y})$ and $q(\mathbf{x},\mathbf{y})$, respectively. Let $K^p$ and $L^p$ denote, respectively, the Gram matrices for the variable $\mathbf{x}$ and the output variable $\mathbf{y}$ in the distribution $p$. Similarly, let $K^q$ and $L^q$ denote, respectively, the Gram matrices for the variable $\mathbf{x}$ and the output variable $\mathbf{y}$ in the distribution $q$. Meanwhile, let $K^{pq}\in \mathbb{R}^{M\times N}$ (i.e., $\left(K^{pq}\right)_{ij}=\kappa(\mathbf{x}^p_i - \mathbf{x}^q_j)$) denote the Gram matrix from distribution $p$ to distribution $q$ for input variable $\mathbf{x}$, and $L^{pq}\in \mathbb{R}^{M\times N}$ the Gram matrix from distribution $p$ to distribution $q$ for output variable $\mathbf{y}$.
Similarly, let $K^{qp}\in \mathbb{R}^{N\times M}$ (i.e., $\left(K^{qp}\right)_{ji}=\kappa(\mathbf{x}^q_j - \mathbf{x}^p_i)$) denote the Gram matrix from distribution $q$ to distribution $p$ for input variable $\mathbf{x}$, and $L^{qp}\in \mathbb{R}^{N\times M}$ the Gram matrix from distribution $q$ to distribution $p$ for output variable $\mathbf{y}$.
The empirical estimation of $D_{\text{CS}}(p(\mathbf{y}|\mathbf{x});q(\mathbf{y}|\mathbf{x}))$ is given by:
\begin{equation}\label{eq:conditional_CS_est}
\begin{split}
& \widehat{D}_{\text{CS}}(p(\mathbf{y}|\mathbf{x});q(\mathbf{y}|\mathbf{x}))  \approx \log\left( \sum_{j=1}^M \left( \frac{ \sum_{i=1}^M K_{ji}^p L_{ji}^p }{ (\sum_{i=1}^M K_{ji}^p)^2 } \right) \right)
 + \log\left( \sum_{j=1}^N \left( \frac{ \sum_{i=1}^N K_{ji}^q L_{ji}^q }{ (\sum_{i=1}^N K_{ji}^q)^2 } \right) \right) \\
& - \log \left( \sum_{j=1}^M \left( \frac{ \sum_{i=1}^N K_{ji}^{pq} L_{ji}^{pq} }{ (\sum_{i=1}^M K_{ji}^p) (\sum_{i=1}^N K_{ji}^{pq}) } \right) \right)
 - \log \left( \sum_{j=1}^N \left( \frac{ \sum_{i=1}^M K_{ji}^{qp} L_{ji}^{qp} }{ (\sum_{i=1}^M K_{ji}^{qp}) (\sum_{i=1}^N K_{ji}^q) } \right) \right)
\end{split}
\end{equation}

In the following, we first demonstrate how to estimate the two conditional quadratic terms (i.e., $\int_\mathcal{X}\int_\mathcal{Y} \frac{p^2(\mathbf{x},\mathbf{y})}{p^2(\mathbf{x})} d\mathbf{x}d\mathbf{y}$ and $\int_\mathcal{X}\int_\mathcal{Y} \frac{q^2(\mathbf{x},\mathbf{y})}{q^2(\mathbf{x})} d\mathbf{x}d\mathbf{y}$) from samples. We then demonstrate how to estimate the cross term (i.e., $\int_\mathcal{X}\int_\mathcal{Y} \frac{p(\mathbf{x},\mathbf{y})q(\mathbf{x},\mathbf{y})}{p(\mathbf{x})q(\mathbf{x})} d\mathbf{x}d\mathbf{y}$). We finally explain the empirical estimation of $D_{\text{CS}}(p(\mathbf{y}|\mathbf{x});q(\mathbf{y}|\mathbf{x}))$.

\begin{proof}

\definecolor{lightmintbg}{rgb}{.88,.96,.99}
\colorbox{lightmintbg}{[The conditional quadratic term]}

The empirical estimation of $\int_\mathcal{X}\int_\mathcal{Y} \frac{p^2(\mathbf{x},\mathbf{y})}{p^2(\mathbf{x})} d\mathbf{x}d\mathbf{y}$ can be expressed as:
\begin{equation}
\int_\mathcal{X}\int_\mathcal{Y} \frac{p^2(\mathbf{x},\mathbf{y})}{p^2(\mathbf{x})} d\mathbf{x}d\mathbf{y} = \mathbb{E}_{p(X,Y)} \left[ \frac{p(X,Y)}{p^2(X)} \right] \approx \frac{1}{M} \sum_{j=1}^M \frac{p(\mathbf{x}_j,\mathbf{y}_j)}{p^2(\mathbf{x}_j)}.
\end{equation}

By kernel density estimator (KDE), we have:
\begin{equation}
\frac{p(\mathbf{x}_j,\mathbf{y}_j)}{p^2(\mathbf{x}_j)} \approx M \frac{\sum_{i=1}^M \kappa_\sigma(\mathbf{x}_j^p - \mathbf{x}_i^p)\kappa_\sigma(\mathbf{y}_j^p - \mathbf{y}_i^p) }{ \left(\sum_{i=1}^M \kappa_\sigma (\mathbf{x}_j^p - \mathbf{x}_i^p)\right)^2 }.
\end{equation}

Therefore,
\begin{equation}
\int_\mathcal{X}\int_\mathcal{Y} \frac{p^2(\mathbf{x},\mathbf{y})}{p^2(\mathbf{x})} d\mathbf{x}d\mathbf{y} \approx \sum_{j=1}^M \left( \frac{\sum_{i=1}^M \kappa_\sigma(\mathbf{x}_j^p - \mathbf{x}_i^p)\kappa_\sigma(\mathbf{y}_j^p - \mathbf{y}_i^p) }{ \left(\sum_{i=1}^M \kappa_\sigma (\mathbf{x}_j^p - \mathbf{x}_i^p)\right)^2 } \right).
\end{equation}

Similarly, the empirical estimation of $\int_\mathcal{X}\int_\mathcal{Y} \frac{q^2(\mathbf{x},\mathbf{y})}{q^2(\mathbf{x})} d\mathbf{x}d\mathbf{y}$ is given by:
\begin{equation}
\int_\mathcal{X}\int_\mathcal{Y} \frac{q^2(\mathbf{x},\mathbf{y})}{q^2(\mathbf{x})} d\mathbf{x}d\mathbf{y} \approx \sum_{j=1}^N \left( \frac{\sum_{i=1}^N \kappa_\sigma(\mathbf{x}_j^q - \mathbf{x}_i^q)\kappa_\sigma(\mathbf{y}_j^q - \mathbf{y}_i^q) }{ \left(\sum_{i=1}^N \kappa_\sigma (\mathbf{x}_j^q - \mathbf{x}_i^q)\right)^2 } \right).
\end{equation}

\colorbox{lightmintbg}{[The cross term]}

Again, the empirical estimation of $\int_\mathcal{X}\int_\mathcal{Y} \frac{p(\mathbf{x},\mathbf{y})q(\mathbf{x},\mathbf{y})}{p(\mathbf{x})q(\mathbf{x})} d\mathbf{x}d\mathbf{y}$ can be expressed as:
\begin{equation}
\int_\mathcal{X}\int_\mathcal{Y} \frac{p(\mathbf{x},\mathbf{y})q(\mathbf{x},\mathbf{y})}{p(\mathbf{x})q(\mathbf{x})} d\mathbf{x}d\mathbf{y} = \mathbb{E}_{p(X,Y)} \left[ \frac{q(X,Y)}{p(X)q(X)} \right] \approx \frac{1}{M} \sum_{j=1}^M \frac{q(\mathbf{x}_j,\mathbf{y}_j)}{p(\mathbf{x}_j)q(\mathbf{x}_j)}.
\end{equation}

By KDE, we further have:
\begin{equation}
\frac{q(\mathbf{x}_j,\mathbf{y}_j)}{p(\mathbf{x}_j)q(\mathbf{x}_j)} \approx M \frac{\sum_{i=1}^N \kappa_\sigma(\mathbf{x}_j^p - \mathbf{x}_i^q)\kappa_\sigma(\mathbf{y}_j^p - \mathbf{y}_i^q) }{\sum_{i=1}^M \kappa_\sigma (\mathbf{x}_j^p - \mathbf{x}_i^p) \sum_{i=1}^N \kappa_\sigma (\mathbf{x}_j^p - \mathbf{x}_i^q)}.
\end{equation}

Therefore,
\begin{equation}\label{eq:cross}
\int_\mathcal{X}\int_\mathcal{Y} \frac{p(\mathbf{x},\mathbf{y})q(\mathbf{x},\mathbf{y})}{p(\mathbf{x})q(\mathbf{x})} d\mathbf{x}d\mathbf{y} \approx \sum_{j=1}^M \left( \frac{\sum_{i=1}^N \kappa_\sigma(\mathbf{x}_j^p - \mathbf{x}_i^q)\kappa_\sigma(\mathbf{y}_j^p - \mathbf{y}_i^q) }{\sum_{i=1}^M \kappa_\sigma (\mathbf{x}_j^p - \mathbf{x}_i^p) \sum_{i=1}^N \kappa_\sigma (\mathbf{x}_j^p - \mathbf{x}_i^q)} \right).
\end{equation}

Note that, one can also empirically estimate $\int_\mathcal{X}\int_\mathcal{Y} \frac{p(\mathbf{x},\mathbf{y})q(\mathbf{x},\mathbf{y})}{p(\mathbf{x})q(\mathbf{x})} d\mathbf{x}d\mathbf{y}$ over $q(\mathbf{x},\mathbf{y})$, which can be expressed as:
\begin{equation}\label{eq:cross_alternative}
\begin{split}
\int_\mathcal{X}\int_\mathcal{Y} \frac{p(\mathbf{x},\mathbf{y})q(\mathbf{x},\mathbf{y})}{p(\mathbf{x})q(\mathbf{x})} d\mathbf{x}d\mathbf{y} & = \mathbb{E}_{q(X,Y)} \left[ \frac{p(X,Y)}{p(X)q(X)} \right]
\approx \frac{1}{N} \sum_{j=1}^N \frac{p(\mathbf{x}_j,\mathbf{y}_j)}{p(\mathbf{x}_j)q(\mathbf{x}_j)} \\
& \approx \sum_{j=1}^N \left( \frac{\sum_{i=1}^M \kappa_\sigma(\mathbf{x}_j^q - \mathbf{x}_i^p)\kappa_\sigma(\mathbf{y}_j^q - \mathbf{y}_i^p) }{\sum_{i=1}^M \kappa_\sigma (\mathbf{x}_j^q - \mathbf{x}_i^p) \sum_{i=1}^N \kappa_\sigma (\mathbf{x}_j^q - \mathbf{x}_i^q)} \right).
\end{split}
\end{equation}

\colorbox{lightmintbg}{[Empirical Estimation]}


Let $K^p$ and $L^p$ denote, respectively, the Gram matrices for the input variable $\mathbf{x}$ and output variable $\mathbf{y}$ in the distribution $p$. Further, let $\left(K\right)_{ji}$ denote the $(j,i)$-th element of a matrix $K$ (i.e., the $j$-th row and $i$-th column of $K$). We have:
\begin{equation}\label{eq:estimate_quadratic}
\int_\mathcal{X}\int_\mathcal{Y} \frac{p^2(\mathbf{x},\mathbf{y})}{p^2(\mathbf{x})} d\mathbf{x}d\mathbf{y} \approx \sum_{j=1}^M \left( \frac{ \sum_{i=1}^M K_{ji}^p L_{ji}^p }{ (\sum_{i=1}^M K_{ji}^p)^2 } \right).
\end{equation}


Similarly, let $K^q$ and $L^q$ denote, respectively, the Gram matrices for input variable $\mathbf{x}$ and output variable $\mathbf{y}$ in the distribution $q$. We have:
\begin{equation}
\int_\mathcal{X}\int_\mathcal{Y} \frac{q^2(\mathbf{x},\mathbf{y})}{q^2(\mathbf{x})} d\mathbf{x}d\mathbf{y} \approx \sum_{j=1}^N \left( \frac{ \sum_{i=1}^N K_{ji}^q L_{ji}^q }{ (\sum_{i=1}^N K_{ji}^q)^2 } \right).
\end{equation}


Further, let $K^{pq}\in \mathbb{R}^{M\times N}$ denote the Gram matrix between distributions $p$ and $q$ for input variable $\mathbf{x}$, and $L^{pq}$ the Gram matrix between distributions $p$ and $q$ for output variable $\mathbf{y}$. According to Eq.~(\ref{eq:cross}), we have:
\begin{equation}\label{eq:estimate_cross1}
\int_\mathcal{X}\int_\mathcal{Y} \frac{p(\mathbf{x},\mathbf{y})q(\mathbf{x},\mathbf{y})}{p(\mathbf{x})q(\mathbf{x})} d\mathbf{x}d\mathbf{y} \approx \sum_{j=1}^M \left( \frac{ \sum_{i=1}^N K_{ji}^{pq} L_{ji}^{pq} }{ (\sum_{i=1}^M K_{ji}^p) (\sum_{i=1}^N K_{ji}^{pq}) } \right).
\end{equation}


Therefore, according to Eqs.~(\ref{eq:estimate_quadratic})-(\ref{eq:estimate_cross1}), an empirical estimate of $D_{\text{CS}}(p(\mathbf{y}|\mathbf{x});q(\mathbf{y}|\mathbf{x}))$ is given by:
\begin{equation}\label{eq:conditional_CS_est1}
\begin{split}
D_{\text{CS}}(p(\mathbf{y}|\mathbf{x});q(\mathbf{y}|\mathbf{x})) & \approx \log\left( \sum_{j=1}^M \left( \frac{ \sum_{i=1}^M K_{ji}^p L_{ji}^p }{ (\sum_{i=1}^M K_{ji}^p)^2 } \right) \right)
+ \log\left( \sum_{j=1}^N \left( \frac{ \sum_{i=1}^N K_{ji}^q L_{ji}^q }{ (\sum_{i=1}^N K_{ji}^q)^2 } \right) \right) \\
& - 2 \log \left( \sum_{j=1}^M \left( \frac{ \sum_{i=1}^N K_{ji}^{pq} L_{ji}^{pq} }{ (\sum_{i=1}^M K_{ji}^p) (\sum_{i=1}^N K_{ji}^{pq}) } \right) \right).
\end{split}
\end{equation}

Note that, according to Eq.~(\ref{eq:cross_alternative}), $D_{\text{CS}}(p(\mathbf{y}|\mathbf{x});q(\mathbf{y}|\mathbf{x}))$ can also be expressed as:
\begin{equation}
\begin{split}
D_{\text{CS}}(p(\mathbf{y}|\mathbf{x});q(\mathbf{y}|\mathbf{x})) & \approx \log\left( \sum_{j=1}^M \left( \frac{ \sum_{i=1}^M K_{ji}^p L_{ji}^p }{ (\sum_{i=1}^M K_{ji}^p)^2 } \right) \right)
+ \log\left( \sum_{j=1}^N \left( \frac{ \sum_{i=1}^N K_{ji}^q L_{ji}^q }{ (\sum_{i=1}^N K_{ji}^q)^2 } \right) \right) \\
& - 2 \log \left( \sum_{j=1}^N \left( \frac{ \sum_{i=1}^M K_{ji}^{qp} L_{ji}^{qp} }{ (\sum_{i=1}^M K_{ji}^{qp}) (\sum_{i=1}^N K_{ji}^q) } \right) \right).
\end{split}
\end{equation}

Therefore, to obtain a consistent and symmetric expression, we estimate $D_{\text{CS}}(p(\mathbf{y}|\mathbf{x});q(\mathbf{y}|\mathbf{x}))$ by:
\begin{equation}
\begin{split}
D_{\text{CS}}(p(\mathbf{y}|\mathbf{x});q(\mathbf{y}|\mathbf{x})) & \approx \log\left( \sum_{j=1}^M \left( \frac{ \sum_{i=1}^M K_{ji}^p L_{ji}^p }{ (\sum_{i=1}^M K_{ji}^p)^2 } \right) \right)
+ \log\left( \sum_{j=1}^N \left( \frac{ \sum_{i=1}^N K_{ji}^q L_{ji}^q }{ (\sum_{i=1}^N K_{ji}^q)^2 } \right) \right) \\
& - \log \left( \sum_{j=1}^M \left( \frac{ \sum_{i=1}^N K_{ji}^{pq} L_{ji}^{pq} }{ (\sum_{i=1}^M K_{ji}^p) (\sum_{i=1}^N K_{ji}^{pq}) } \right) \right)
- \log \left( \sum_{j=1}^N \left( \frac{ \sum_{i=1}^M K_{ji}^{qp} L_{ji}^{qp} }{ (\sum_{i=1}^M K_{ji}^{qp}) (\sum_{i=1}^N K_{ji}^q) } \right) \right).
\end{split}
\end{equation}

\end{proof}

\subsection{Proof to Proposition 3}

Assume that we are given observations $\psi=\{(\mathbf{x}_i,\mathbf{y}_i^1,\mathbf{y}_i^2 )\}_{i=1}^N$, where $\mathbf{x}\in \mathbb{R}^{d_\mathbf{x}}$, $\mathbf{y}^1\in \mathbb{R}^{d_\mathbf{y}}$ and $\mathbf{y}^2\in \mathbb{R}^{d_\mathbf{y}}$. Let $K$, $L^1$, and $L^2$ denote, respectively, the Gram matrices for variables $\mathbf{x}$, $\mathbf{y}^1$, and $\mathbf{y}^2$. Further, let $L^{21}$ denote the Gram matrix between $\mathbf{y}^2$ and $\mathbf{y}^1$.  The empirical estimation of $D_{\text{CS}}(p(\mathbf{y}_1|\mathbf{x});p(\mathbf{y}_2|\mathbf{x}))$ is given by:
\begin{equation}\label{eq:CS_ext1}
\small
\begin{split}
& D_{\text{CS}}(p(\mathbf{y}_1|\mathbf{x});p(\mathbf{y}_2|\mathbf{x}))  \approx \log\left( \sum_{j=1}^N \left( \frac{ \sum_{i=1}^N K_{ji} L_{ji}^1 }{ (\sum_{i=1}^N K_{ji})^2 } \right) \right) \\
& + \log\left( \sum_{j=1}^N \left( \frac{ \sum_{i=1}^N K_{ji} L_{ji}^2 }{ (\sum_{i=1}^N K_{ji})^2 } \right) \right)
 - 2 \log \left( \sum_{j=1}^N \left( \frac{ \sum_{i=1}^N K_{ji} L_{ji}^{21} }{ (\sum_{i=1}^N K_{ji})^2 } \right) \right).
\end{split}
\end{equation}

\begin{proof}
Eq.~(\ref{eq:CS_ext1}) can be obtained by setting $\mathbf{x}_1 = \mathbf{x}_2 = \mathbf{x}$. In this sense, we have:
\begin{equation}\label{eq:series_K}
K = K^p = K^q = K^{pq} = K^{qp} \in \mathbb{R}^{N\times N}.
\end{equation}

Plugging in Eq.~(\ref{eq:series_K}) into Eq.~(\ref{eq:conditional_CS_est}), we obtain Eq.~(\ref{eq:CS_ext1}).
\end{proof}

\subsection{Proof to Proposition 4}

\begin{proof}
The CS divergence for conditional distribution $p(\mathbf{y}|\mathbf{x}_1)$ and conditional distribution $p(\mathbf{y}|\mathbf{x}_1,\mathbf{x}_2)$ can be expressed as (denote $\vec{\mathbf{x}}=[\mathbf{x}_1;\mathbf{x}_2]$):
\begin{equation}
\begin{split}
& D_{\text{CS}}(p(\mathbf{y}|\mathbf{x}_1);p(\mathbf{y}|\{\mathbf{x}_1,\mathbf{x}_2\}) \\
& = - 2 \log \left(\int_\mathcal{X}\int_\mathcal{Y} p(\mathbf{y}|\mathbf{x}_1) p(\mathbf{y}|\vec{\mathbf{x}}) d\vec{\mathbf{x}}d\mathbf{y} \right)
+ \log \left(\int_\mathcal{X}\int_\mathcal{Y} p^2(\mathbf{y}|\mathbf{x}_1) d\vec{\mathbf{x}}d\mathbf{y}\right) + \log \left(\int_\mathcal{X}\int_\mathcal{Y} p^2(\mathbf{y}|\vec{\mathbf{x}} ) d\vec{\mathbf{x}}d\mathbf{y}\right) \\
& = - 2 \log \left(\int_\mathcal{X}\int_\mathcal{Y} \frac{p(\mathbf{x}_1,\mathbf{y})p(\vec{\mathbf{x}},\mathbf{y})}{p(\mathbf{x}_1)p(\vec{\mathbf{x}})} d\vec{\mathbf{x}}d\mathbf{y} \right)
 + \log \left(\int_\mathcal{X}\int_\mathcal{Y} \frac{p^2(\mathbf{x}_1,\mathbf{y})}{p^2(\mathbf{x}_1)} d\vec{\mathbf{x}}d\mathbf{y}\right) + \log \left(\int_\mathcal{X}\int_\mathcal{Y} \frac{p^2(\vec{\mathbf{x}},\mathbf{y})}{p^2(\vec{\mathbf{x}})} d\vec{\mathbf{x}}d\mathbf{y}\right),
\end{split}
\end{equation}

Following the proof to Proposition 2, the two conditional quadratic terms can be estimated empirically by:
\begin{equation}\label{eq:appendix_1}
    \int_\mathcal{X}\int_\mathcal{Y} \frac{p^2(\mathbf{x}_1,\mathbf{y})}{p^2(\mathbf{x}_1)} d\vec{\mathbf{x}}d\mathbf{y} \approx \sum_{j=1}^N \left( \frac{ \sum_{i=1}^N K_{ji}^1 L_{ji} }{ (\sum_{i=1}^N K_{ji}^1)^2 } \right),
\end{equation}
and
\begin{equation}\label{eq:appendix_2}
    \int_\mathcal{X}\int_\mathcal{Y} \frac{p^2(\vec{\mathbf{x}},\mathbf{y})}{p^2(\vec{\mathbf{x}})} d\vec{\mathbf{x}}d\mathbf{y} \approx \sum_{j=1}^N \left( \frac{ \sum_{i=1}^N K_{ji}^{12} L_{ji} }{ (\sum_{i=1}^N K_{ji}^{12})^2 } \right).
\end{equation}

We only discuss below the empirical estimation to the term $\int_\mathcal{X}\int_\mathcal{Y} \frac{p(\mathbf{x}_1,\mathbf{y})p(\vec{\mathbf{x}},\mathbf{y})}{p(\mathbf{x}_1)p(\vec{\mathbf{x}})} d\vec{\mathbf{x}}d\mathbf{y}$.

We have:
\begin{equation}
\iint \frac{p(\mathbf{x}_1,\mathbf{y}) p(\vec{\mathbf{x}},\mathbf{y})}{p(\mathbf{x}_1)p(\vec{\mathbf{x}})} = \mathbb{E}_{p(\vec{\mathbf{x}},\mathbf{y})} \left[ \frac{p(\mathbf{x}_1,\mathbf{y})}{p(\mathbf{x}_1)p(\vec{\mathbf{x}})} \right] \approx \frac{1}{N} \sum_{j=1}^N \frac{p((\mathbf{x}_1)_j,\mathbf{y}_j)}{p((\mathbf{x}_1)_j)p(\vec{\mathbf{x}}_j)}.
\end{equation}

By KDE, we have:
\begin{equation}
\frac{p((\mathbf{x}_1)_j,\mathbf{y}_j)}{p((\mathbf{x}_1)_j)p(\vec{\mathbf{x}}_j)} \approx N \frac{\sum_{i=1}^N \kappa_\sigma(\mathbf{y}_j - \mathbf{y}_i)\kappa_\sigma((\mathbf{x}_1)_j - (\mathbf{x}_1)_i) }{\sum_{i=1}^N \kappa_\sigma((\mathbf{x}_1)_j - (\mathbf{x}_1)_i)\times \sum_{i=1}^N \kappa_\sigma((\mathbf{x}_1)_j - (\mathbf{x}_1)_i)\kappa_\sigma((\mathbf{x}_2)_j - (\mathbf{x}_2)_i)}.
\end{equation}

Therefore,
\begin{equation}\label{eq:appendix_3}
\begin{split}
\iint \frac{p(\mathbf{x}_1,\mathbf{y}) p(\vec{\mathbf{x}},\mathbf{y})}{p(\mathbf{x}_1)p(\vec{\mathbf{x}})}
& \approx \sum_{j=1}^N \left( \frac{\sum_{i=1}^N \kappa_\sigma(\mathbf{y}_j - \mathbf{y}_i)\kappa_\sigma((\mathbf{x}_1)_j - (\mathbf{x}_1)_i) }{\sum_{i=1}^N \kappa_\sigma((\mathbf{x}_1)_j - (\mathbf{x}_1)_i)\times \sum_{i=1}^N \kappa_\sigma((\mathbf{x}_1)_j - (\mathbf{x}_1)_i)\kappa_\sigma((\mathbf{x}_2)_j - (\mathbf{x}_2)_i)} \right) \\
& = \sum_{j=1}^N \left( \frac{ \sum_{i=1}^N K_{ji}^1 L_{ji} }{ (\sum_{i=1}^N K_{ji}^1)(\sum_{i=1}^N K_{ji}^{12}) } \right).
\end{split}
\end{equation}

Combining Eq.~(\ref{eq:appendix_1}), Eq.~(\ref{eq:appendix_2}), and Eq.~(\ref{eq:appendix_3}), we obtain Eq.~(\ref{eq:CS_ext2}).

\begin{equation}\label{eq:CS_ext2}
\begin{split}
& D_{\text{CS}}(p(\mathbf{y}|\mathbf{x}_1);p(\mathbf{y}|\{\mathbf{x}_1,\mathbf{x}_2\}) \approx \log\left( \sum_{j=1}^N \left( \frac{ \sum_{i=1}^N K_{ji}^1 L_{ji} }{ (\sum_{i=1}^N K_{ji}^1)^2 } \right) \right) \\
& + \log\left( \sum_{j=1}^N \left( \frac{ \sum_{i=1}^N K_{ji}^{12} L_{ji} }{ (\sum_{i=1}^N K_{ji}^{12})^2 } \right) \right)
 - 2 \log \left( \sum_{j=1}^N \left( \frac{ \sum_{i=1}^N K_{ji}^1 L_{ji} }{ (\sum_{i=1}^N K_{ji}^1)(\sum_{i=1}^N K_{ji}^{12}) } \right) \right).
\end{split}
\end{equation}

\end{proof}

\section{Experimental Details and Additional Results}

\subsection{Detailed Experiment Settings of Causal Discovery in Section~\ref{sec:extensions}}\label{sec:causal_setup}

We first consider $5$ coupled H\'enon chaotic maps~\cite{kugiumtzis2013direct}, whose ground truth causal relationship is $x_{i-1} \rightarrow x_i$. The system of $K$ coupled H\'enon chaotic maps is defined as:
\begin{equation}
    \begin{cases}
      x_{1,t}=1.4- x_{1,t-1}^2+0.3x_{1,t-2} \\
      x_{i,t}=1.4-(Cx_{i-1,t-1}+(1-C) x_{i,t-1} )^2+0.3x_{i,t-2} & \text{for $ i=2,3,\cdots,K$}.
    \end{cases}
\end{equation}

In our simulation, we set the coupling strength $C=0.3$ and generate $1,024$ samples in $10$ independent realizations respectively.

Next, we consider a nonlinear VAR process of order $1$ with $3$ variables (NLVAR3)~\cite{gourevitch2006linear}:
\begin{equation}
    \begin{cases}
        x_{1,t}=3.4x_{1,t-1} (1-x_{1,t-1}^2 ) \exp{(-x_{1,t-1}^2)} +0.01w_{1,t} \\
        x_{2,t}=3.4x_{2,t-1} (1-x_{2,t-1}^2 ) \exp{(-x_{2,t-1}^2)} +0.5x_{1,t-1}x_{2,t-1} +0.01w_{2,t} \\
        x_{3,t}=3.4x_{3,t-1} (1-x_{3,t-1}^2 ) \exp{(-x_{3,t-1}^2)} +0.3x_{2,t-1} +0.5x_{1,t-1}^2 +0.01w_{3,t}
    \end{cases}
\end{equation}
where $w_{1,t}$, $w_{2,t}$ and $w_{3,t}$ denotes independent standard Gaussian noise. The true causal directions in NLVAR3 are $x_1 \rightarrow x_2$, $x_1 \rightarrow x_3$, $x_2 \rightarrow x_3$. Again, we generate $1,024$ samples in $10$ independent realizations respectively.

The delay $\tau$ and the embedding dimension $d$ are vital parameters for all Wiener and Granger causality methods. In our simulations, we simply use the ground truth from the synthetic models, that is $\{\tau=1,d=1\}$ for NLVAR3 and $\{\tau=1,d=2\}$ for H\'enon maps, since selecting the embedding automatically is not the topic of this paper.

For kernel Granger causality (KGC), we use the official MATLAB code from authors\footnote{\url{https://github.com/danielemarinazzo/KernelGrangerCausality}} and select kernel size $\sigma$ by the median rule. For transfer entropy, we use the $k$NN entropy estimator from the Information Theoretical Estimators Toolbox\footnote{\url{https://bitbucket.org/szzoli/ite/src/master/}} and set $k=3$ as default.

Once we obtain the causal score $C_{\mathbf{x}\rightarrow \mathbf{y}}$ with one of the measures, we need to test its significance. To carry out this hypothesis test, we may use the Monte Carlo method by constructing a surrogate time series. The constructed surrogate time series must satisfy the null hypothesis that the causal influence from $\mathbf{x}$ to $\mathbf{y}$ is completely destroyed; at the same time, the statistical properties of $\mathbf{x}$ and $\mathbf{y}$ should remain the same. To construct the surrogate time series that satisfies these two conditions, we apply the surrogate time series construction method in~\cite{theiler1992testing,duan2014transfer}.

Specifically, given $N$ samples, let $T$ denote the length of the training set (in our case $T=1,024$), a pair of surrogate time series for $\mathbf{x}$ and $\mathbf{y}$ is constructed as:
\begin{equation}
    \begin{cases}
        x^{\text{surr}}=[x_i,x_{i+1},\cdots,x_{i+T-1}] \\
        y^{\text{surr}}=[y_j,y_{j+1},\cdots,y_{j+T-1}],
    \end{cases}
\end{equation}
where $i$ and $j$ are randomly chosen from ${1,2,\cdots, N-T+1}$ and $|j-i|\geq e$ and $e$ is a sufficient large integer such that there is almost no correlation between $x^{\text{surr}}$ and $y^{\text{surr}}$. In our simulation, we set $e=512$.
We compare $C(x\rightarrow y)$ with $C(x^{\text{surr}}\rightarrow y^{\text{surr}})$ from the $P$ permutations ($P=100$ in our experiment) to obtain a $p$-value. $p$-values below $0.05$ were considered as significant. Details are provided in Algorithm~\ref{CausalAlg}.

\begin{algorithm}
\caption{Test the significance of $C(x\rightarrow y)$}
\label{CausalAlg}
 \begin{algorithmic}[1]
 \renewcommand{\algorithmicrequire}{\textbf{Input:}}
 \renewcommand{\algorithmicensure}{\textbf{Output:}}
 \REQUIRE Two time series $\{x_t\}$ and $\{y_t\}$;
Number of permutations $P$;
Significance level $\eta=0.05$.
 \ENSURE  Test \emph{decision} (Is $\mathcal{H}_0: C(x\rightarrow y)$ significant or not?).
  \STATE Construct $\{y_{t+1},\mathbf{x}_t^m, \mathbf{y}_t^n\}_{t=1}^T$ ($T$ represents the total number of observations) from $\{x_t\}$ and $\{y_t\}$;
  \STATE Compute $C(x\rightarrow y)=D (p(y_{t+1}|\mathbf{y}_t^n);p(y_{t+1}|\mathbf{y}_t^n,\mathbf{x}_t^m))$ with one of the conditional divergence measures (e.g., conditional KL, or conditional Bregman divergence, or conditional MMD, or conditional CS divergence).
 \\ 
  \FOR {$m = 1$ to $P$}
  \STATE Construct a pair of surrogate time series $x^{\text{surr}}_m$ and $y^{\text{surr}}_m$.
  \STATE Compute $C(x^{\text{surr}}_m \rightarrow y^{\text{surr}}_m )$ with the selected conditional divergence measure.
  \ENDFOR
  \IF {$\frac{1+\sum\nolimits_{m=1}^P\mathbf{1}[C(x\rightarrow y) \leq C(x^{\text{surr}}_m \rightarrow y^{\text{surr}}_m )]}{1+P}\leq\eta$}
  \STATE $C(x\rightarrow y)$ is not significant large.
  \ELSE
  \STATE $C(x\rightarrow y)$ is significant large.
  \ENDIF
 \RETURN \emph{decision}
 \end{algorithmic}
\end{algorithm}

\subsection{Details of Permutation Test in Section~\ref{sec:simulation1}}\label{sec:permutation_test}

Details of permutation test are elaborated in Algorithm~\ref{PermutationAlg}.

\begin{algorithm}[htb]
\caption{Test the conditional distribution divergence (CDD) based on the matrix Bregman divergence}
\label{PermutationAlg}
\begin{algorithmic}[1]
\REQUIRE
Two groups of observations $\psi_p = \{(\mathbf{x}_i^p,\mathbf{y}_i^p)\}_{i=1}^{M}$ and $\psi_q = \{(\mathbf{x}_j^q,\mathbf{y}_j^q)\}_{j=1}^{N}$;
$D_{\varphi,B}$;
Number of permutations $P$;
Significance level $\eta$.
\ENSURE
Test \emph{decision} (Is $H_0: p(\mathbf{y}|\mathbf{x})=q(\mathbf{y}|\mathbf{x})$ $True$ or $False$?).
\STATE Compute conditional divergence value $d_0$ on $\psi_p$ and $\psi_q$ with one of the conditional divergence measures (e.g., conditional KL, or conditional Bregman divergence, or conditional MMD, or conditional CS divergence).
\FOR {$m = 1$ to $P$}
\STATE $(\psi^m_p, \psi^m_q)\leftarrow$ random split of $\psi_ p\bigcup \psi_q$.
\STATE Compute conditional divergence value $d_{m}$ on $\psi^m_p$ and $\psi^m_q$ with the selected conditional divergence measure.
\ENDFOR
\IF {$\frac{1+\sum\nolimits_{m=1}^P\mathbf{1}[d_{0}\leq d_m]}{1+P}\leq\eta$}
\STATE \emph{decision}$\leftarrow$$False$
\ELSE
\STATE \emph{decision}$\leftarrow$$True$
\ENDIF
\RETURN \emph{decision}
\end{algorithmic}
\end{algorithm}

\subsection{Detailed Experiment Settings and Additional Results of Time Series Clustering in Section~\ref{sec:clustering}}

\subsubsection{More on Dynamic Texture Datasets}\label{sec:data_clustering}
The dynamic texture (DT) is a sequence of images of moving scenes such as flames, smoke, and waves, which exhibits certain stationarity in time and can be modeled using a linear dynamic system (LDS)~\cite{doretto2003dynamic}:
\begin{equation}
    \begin{cases}
      \mathbf{h}_t = A \mathbf{h}_{t-1} + v_t, v_t \sim N(0,Q) \\
      \mathbf{y}_t = C \mathbf{h}_t + w_t, w_t \sim N(0,R)
    \end{cases}
\end{equation}
where $\mathbf{y}_t\in \mathbb{R}^m$ is the observational frame at time $t$, $\mathbf{h}_t\in \mathbb{R}^n$ ($n\ll m$) is the hidden state vector at time $t$, $A\in \mathbb{R}^{n\times n}$ is the state transition matrix, $C\in \mathbb{R}^{m\times n}$ is the output matrix that maps hidden states to observations. Finally, $v_t$ and $w_t$ are Gaussian noises with covariance matrices $Q$ and $R$, respectively. Thus, a single dynamic texture can be described by a transition function $f$ from $\mathbf{y}_{t-1}$ to $\mathbf{y}_t$, i.e., $\mathbf{y}_t=f(\mathbf{y}_{t-1})$~\cite{you2016kernel}. In this sense, it is natural to expect to distinguish different DTs by the conditional distribution $p(\mathbf{y}_t |\mathbf{y}_{t-1})$.

The UCLA database\footnote{\url{https://drive.google.com/file/d/0BxMIVlhgRmcbN3pRa0dyaHpHV1E/view?resourcekey=0-_OdWQfyRKH_FHh84bxZLcg}} originally contains $200$ DT sequences from $50$ categories, and each category contains $4$ video sequences captured from different viewpoints. All the video sequences are of the size $48\times 48\times 75$, where $75$ is the number of frames. By combining sequences from different viewpoints, the original $50$ categories are merged to $9$ categories: boiling water ($8$), fire ($8$), flowers ($12$), fountains ($20$), plants ($108$), sea ($12$), smoke ($4$), water ($12$) and waterfall ($16$), where the numbers in parentheses denote the number of the sequences in each category. This dataset is however very challenging and imbalanced, because the category ``plants” contains too many videos. Therefore, we follow \cite{xu2011dynamic} and remove the category of ``plants”. Finally, there are a total of $92$ video sequences from $8$ categories.

The traffic database\footnote{\url{http://www.svcl.ucsd.edu/projects/dytex/}}~\cite{chan2005classification}, consists of $253$ videos divided into three classes: light, medium, and heavy traffic. Videos have $42$ to $52$ frames with a resolution of $48\times 48$ pixels.

The statistics of all datasets are summarized in Table~\ref{tab:property_clustering}.

\begin{table}[]
\centering
\caption{Properties of the datasets (in time series clustering).}
\begin{tabular}{c|c|c|c|c}
\toprule
Datasets          & Time series & Length & Dimension & Clusters \\ \midrule
Coffee            & 28             & 286    & 1         & 2           \\
Diatom            & 306             & 345    & 1         & 4          \\
DistalPhalanxTW   & 400             & 80    & 1         & 6          \\
ECG5000           & 469          & 140    & 1          &  2         \\
FaceAll           & 560            & 131    & 1         & 14          \\
Synthetic control & 600            & 60     & 1         & 6           \\
\midrule
PenDigits         & 3498           & 8      & 2         & 10
 \\
Libras         & 180           & 45      & 2        & 15
 \\
uWave         & 200        & 315      & 3         & 8
 \\
Robot failure LP1    & 88           & 15      & 6       & 4
 \\
Robot failure LP2    & 47           & 15      & 6       & 5
 \\
Robot failure LP3    & 47           & 15      & 6       & 4
 \\
Robot failure LP4    & 117           & 15      & 6       & 3
 \\
Robot failure LP5    & 164           & 15      & 6       & 5
 \\
 \midrule
Traffic           & 253           & 42       & 2304      & 3           \\
UCLA              &  92         & 75    &  2304       & 8           \\ \bottomrule
\end{tabular}
\label{tab:property_clustering}
\end{table}

\subsubsection{Competing Baselines and Their Hyperparameters Setting}\label{sec:baseline_clustering}

We include the following representative baselines in the literature of time series clustering for evaluations:
\begin{itemize}
    \item DTW is a well-known approach to measure the similarity between two temporal time series sequences. It uses the dynamic programming technique to find the optimal temporal matching between elements of two-time series. In our experiments, we use the implementation of DTW and its multivariate extension provided by Schultz and Jain\footnote{ \url{https://www.ai4europe.eu/research/ai-catalog/dtw-mean}}~\cite{schultz2018nonsmooth}. We apply the default warping window constraint, i.e., the length of the longer sequence.
    \item MSM distance is a distance metric for time series. It is conceptually similar to other edit distance approaches. MSM metric uses as building blocks three fundamental operations: Move, Split, and Merge, which can be applied in sequence to transform any time series into any other time series. In our experiments, we use the official code provided by authors\footnote{ \url{https://athitsos.utasites.cloud/projects/msm/}}. There is no hyperparameters need to be tuned for MSM.
    \item TWED is a distance measure for discrete time series matching with time ``elasticity". It has two vital parameters: the penalty for deletion operation $\lambda$ and the elasticity parameter $\nu$. In our experiments, we use the official code provided by authors\footnote{\url{ http://people.irisa.fr/Pierre-Francois.Marteau/}} and set $\lambda=1$ and $\nu= 0.001$. A big limitation of TWED is that its computational complexity is usually $\mathcal{O}(n^2)$, in which $n$ is the length of time series.
    \item TCK is able to learn a valid kernel function to describe the similarity between pairwise time series. Distinct to above mentioned measures, TCK is an ensemble approach that combines multiple GMM models, whose diversity is ensured by training the models on subsamples of data, attributes, and time segments, to capture different levels of granularity in the data. It is robust to missing values and applies to multivariate time series. There are two optional parameters for practitioners: the maximum number of mixture components for each GMM ($C$) and the number of randomizations for each number of components ($G$). In our experiment, we use the default values of $C$ and $G$ provided by authors\footnote{\url{ https://github.com/kmi010/Time-series-cluster-kernel-TCK-}}.
\end{itemize}

\subsubsection{A Case study of exploratory data analysis} \label{sec:exploratory}
This section presents a network modeling example using real data.
We use temperature data over $109$ cities at US in the year of $2010$, which can be obtained from the US National Oceanic and Atmospheric Administration (NOAA)\footnote{\url{https://www.ncei.noaa.gov/access/monitoring/climate-at-a-glance/city/time-series}}. For each city, we obtain temperature records in each hour and take the daily average over $24$ hours,
thus forming a univariate time series of length $364$.

We first compute a $109\times 109$ matrix which encodes conditional CS divergence measures over all pairs of time series. We then obtain a nearest neighbor network over different cities using $5\%$ of the smallest conditional CS divergence measures (see Fig.~\ref{fig:temperature}(a)) and find communities using spectral clustering \textcolor{black}{with the number of clusters set to $2$}. The result is shown in Fig.~\ref{fig:temperature}(b). We can make two observations: 1) closer cities are more likely to have similar temperature patterns; and 2) the north-south difference is much more obvious than the east-west difference. Both observations make sense and match well with a recent exploratory data analysis (EDA) package~\cite{ferreira2022time}.

\begin{figure}[t]
	\centering
	\subfigure[]{
		\centering
		\includegraphics[width=.45\linewidth]{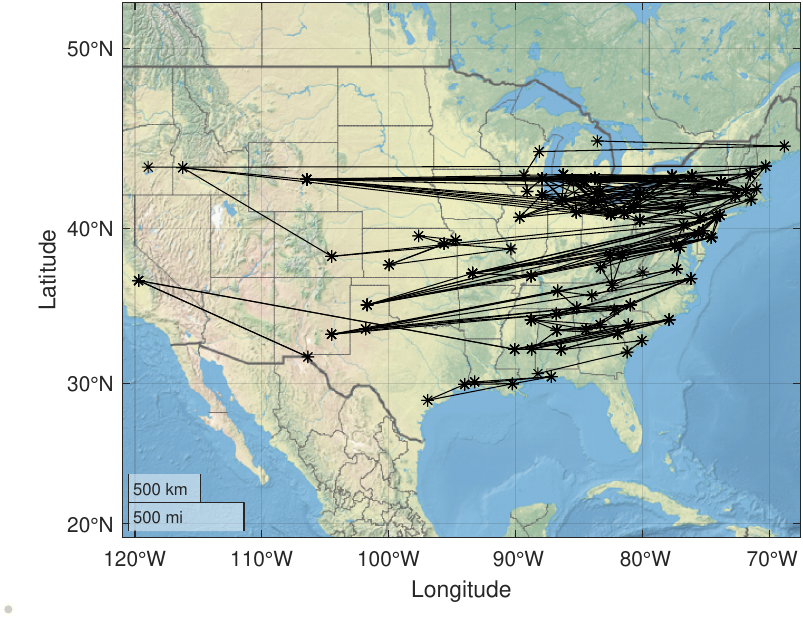}
	}
\hfill
	\subfigure[]{
		\centering
		\includegraphics[width=.45\linewidth]{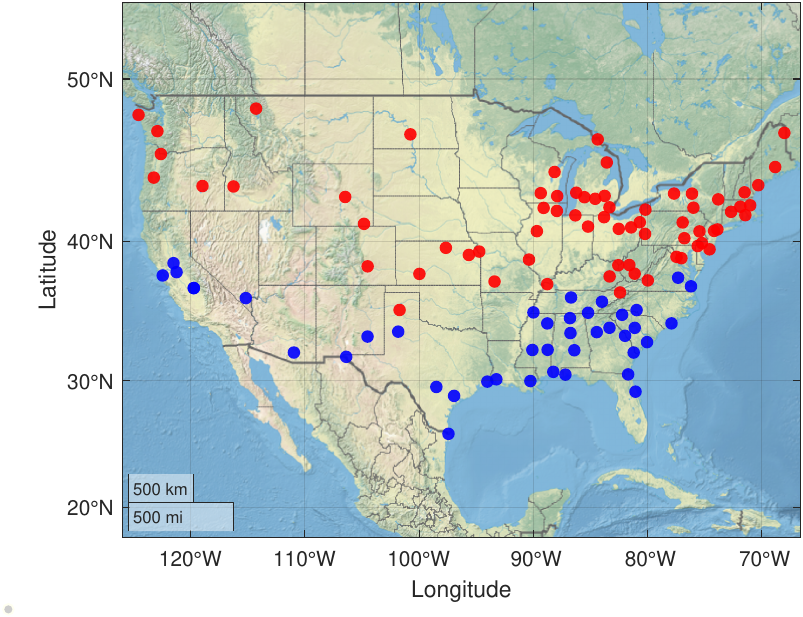}
	}
	\caption{(a) The nearest neighbor network constructed using $5\%$ of the smallest conditional CS divergences; (b) The detected communities by conditional CS divergence. Node colors represent communities.}
	\label{fig:temperature}
\end{figure}

\subsection{Details of divergence to go in Section~\ref{CS-dtg}}\label{sec:dtg_details}
Analogous to Q-learning and its extensions, we employ the divergence-to-go (dtg) metric in a dynamic programming framework. In short, we replace the reward in Q-learning with the divergence between old states and new counterparts $D(p_{\text{new}} (s_{t+1}|s_t,a_t );p_{\text{old}} (s_{t+1}|s_t,a_t))$. To obtain samples for divergence estimation, we introduce a replay buffer $B$ (of size $2\tau$) to record $2\tau$ steps of $\{ s_{t+1}, s_t,a_t\}$ trios. We use trios in the first half of the buffer to estimate $p_{\text{old}}$ and trios in the second half of the buffer to estimate $p_{\text{new}}$, such that we can evaluate the corresponding conditional CS divergence using Eq.~(\ref{eq:conditional_CS_est}) (by treating $s_{t+1}$ as variable $\mathbf{y}$ and the concatenation of $[s_t, a_t]$ as variable $\mathbf{x}$). Same to the initial dtg paper~\cite{emigh2015model}, we use a kernelized version of Q-learning, i.e., Kernel Temporal Difference (KTD) algorithm \cite{bae2011stochastic} as the backbone of dtg. More formally, we have:

\begin{equation}
\delta_{t} = D + \gamma \underset{a'}{\mathbf{max}}[\text{dtg}_{t+1}(s',a')]- \text{dtg}_{t}(s,a),
\end{equation}
and
\begin{equation}
\text{dtg}_{t}(s,a) = \alpha\sum_{j = 1}^{t}\delta_{j}\kappa(z;z_j),
\end{equation}
where $D$ is the divergence $D(p_{\text{new}} (s_{t+1}|s_t,a_t );p_{\text{old}} (s_{t+1}|s_t,a_t))$; $\kappa$ denotes Gaussian kernel; $z:=\{s,t\}$ denotes the vector of a state-action pair. Steps in detail are shown in Algorithm~\ref{algorithm: dtg}. In our experiment, the kernel size is fixed to be $0.1$. The length of the buffer $B$ is $2,000$.

\begin{algorithm}\label{algorithm: dtg}
\caption{The divergence-to-go (DTG) algorithm using conditional Cauchy-Schwarz divergence.}
\label{DTG-CS}
 \begin{algorithmic}[1]

  \STATE Initialize $\text{dtg}_0(s,a) = 0$, set learning rate $\alpha$, discount factor $\gamma$, empty first-in-first-out (FIFO) replay buffer $B$.
 \\ 
  \WHILE{not achieve the goal}
  \STATE $s\leftarrow$ current state.
  \STATE $a\leftarrow \underset{a'}{\mathbf{argmax}}[\text{dtg}(s,a')]$.
  \STATE Take action $a$, observe state transition.
  \STATE $s'\leftarrow$ new state.
  \STATE Record trio $\{s', s, a\}$ to the buffer $B$
  \STATE Compute conditional divergence $D(p_{\text{new}} (s_{t+1}|s_t,a_t );p_{\text{old}} (s_{t+1}|s_t,a_t))$ with Eq.~(\ref{eq:conditional_CS_est}) from trios in the buffer $B$.
  \STATE $\delta = D + \gamma \underset{a'}{\mathbf{max}}[\text{dtg}(x',a')] - \text{dtg}(x,a)$
  \STATE $\text{dtg}(s,a) \leftarrow \text{dtg}(s,a) + \alpha \cdot \delta$
  \ENDWHILE
 \end{algorithmic}
\end{algorithm}

In the following, we also show detailed steps of the standard reward-based Q-learning in Algorithm~\ref{algorithm: qlearning}. The major differences between dtg and Q-learning are highlighted by blue texts. In summary, our dtg algorithm is not reward driven. Instead of obtaining rewards from environments, our dtg estimates conditional CS divergence using state-action pairs only (see line 8 in Algorithm~\ref{algorithm: dtg} and line 7 in Algorithm~\ref{algorithm: qlearning}). In order to estimate divergence, we have to introduce a replay buffer for collecting samples in dtg algorithm. The buffer is not required in original Q-learning though this memory is a popular trick to conduct mini-batch learning and offline learning.

\begin{algorithm}\label{algorithm: qlearning}
\caption{The standard reward-based Q-learning. We highlight the main differences between dtg and Q-learning with blue texts.}
\label{qlearning}
 \begin{algorithmic}[1]

  \STATE Initialize $Q_0(s,a) = 0$, set learning rate $\alpha$, discount factor $\gamma$.
 \\ 
  \WHILE{not achieve the goal}
  \STATE $s\leftarrow$ current state.
  \STATE $a\leftarrow \underset{a'}{\mathbf{argmax}}[\textcolor{blue}{Q(s,a')}]$.
  \STATE Take action $a$, observe state transition.
  \STATE $s'\leftarrow$ new state.
  \STATE \textcolor{blue}{$r \leftarrow$ reward given by the environment.}
  \STATE \textcolor{blue}{$\delta = r + \gamma \underset{a'}{\mathbf{max}}[Q(s',a')] - Q(s,a)$}
  \STATE \textcolor{blue}{$Q(s,a) \leftarrow Q(s,a) + \alpha \cdot \delta$}
  \ENDWHILE
 \end{algorithmic}
\end{algorithm}

Another exploration-based reinforcement learning algorithm, which shares the same motivation to us, is the maximum entropy exploration (MaxEnt)~\cite{hazan2019provably}. Instead of estimating divergence in time series, the goal of MaxEnt is to obtain the policy which maximizes the entropy of states. In our experiments, we compare to MaxEnt with code in \url{https://github.com/abbyvansoest/maxent_base}. We keep all hyper-parameters of DTG and MaxEnt identical, as shown in Table~\ref{tab:hyper-MaxEnt}.


\begin{table}[]
\centering
\caption{Hyperparameters in DTG-CS and standard Q-learning in Algorithms~\ref{DTG-CS} and \ref{qlearning}.}
\begin{tabular}{ll}
\toprule
Hyper-parameter                            & value \\
\midrule
discount factor $\gamma$                     & 0.99  \\
learning rate $\alpha$                              & 1e-3  \\
$\epsilon$-greedy ratio                         & 0.1   \\
number of steps to roll out a policy & 1000 \\
\bottomrule
\end{tabular}
\label{tab:hyper-MaxEnt}
\end{table}

\subsection{Detailed Experiment Settings of Unsupervised Domain Adaptation in Section~\ref{sec:conclusion}}\label{sec:UDA}
\subsubsection{About the Datasets}

We apply CS divergence-based domain adaptation to EEG classification on two benchmark real-world datasets, namely BCI Competition IIIb~\cite{blankertz2006bci} and SEED~\cite{zheng2015investigating}.

The BCI Competition IIIb focuses on cued motor imagery tasks with online feedback, employing a non-stationary classifier with two classes (left hand and right hand). It comprises EEG recordings from three subjects (S4, X11, and O3VR) with $1080$, $1080$, and $640$ trials. The length of each trial is $7$s. Due to a mistake, trials $1$-$160$ and $161$-$320$ in O3VR are repeated. Hence, we only use data from S4 and X11 in the present experiment. EEG data were recorded at a $125$ Hz sampling rate, band-pass filtered between $0.5$ and $30$ Hz, and converted into the GDF format for analysis, using a bipolar EEG amplifier from g.tec with additional Notch filtering.


The SEED consists of data from $15$ subjects. Curated movie excerpts were specifically selected to elicit three distinct emotions: positive, neutral, and negative. Each emotion category contains five movie excerpts. All subjects underwent three EEG recording sessions, with a two-week interval between each session. During each recording session, subjects were exposed to fifteen four-minute long movie excerpts, each intended to induce one of the specified emotions. Importantly, the same fifteen movie excerpts were used across all three recording sessions. Thus, the SEED dataset comprises $15$ EEG trials recorded for each subject in each session, with each emotion category consisting of $5$ trials. EEG signals were acquired using a $62$-channel ESI NeuroScan device, sampled at a rate of $1000$ Hz, and subsequently downsampled to $200$ Hz for analysis~\cite{lan2018domain}. In our experiment, we randomly select $10$ pairs of subjects as the source and target domains, and report the average domain adaptation result.

\subsubsection{Competing Baselines and Their Hyperparameters Setting}

We compare our method with four representative MMD-based approaches, which include DDC~\cite{tzeng2014deep}, DAN~\cite{long2015learning}, JAN~\cite{long2017deep}, DJP-MMD~\cite{zhang2020discriminative}. Given a baseline network as shown in Fig.~\ref{fig:UDA_framework}, comprising a shared feature extractor, fully connected layers of depth $|L|$ and a final linear classification layer, all competing methods optimize the following objective:
\begin{equation}
\label{eq:UDA_obj}
L_{\text{CE}} + \alpha D(p_s,p_t),
\end{equation}
where $L_{\text{CE}}$ is the cross-entropy loss on source domain labeled data, $D(p_s,p_t)$ represents the distributional discrepancy between the source domain and target domain.
However, they vary in their approach to evaluate $D(p_s,p_t)$.

Specifically, DDC just minimizes the discrepancy in the last layer (i.e., $\text{MMD}^2(\mathbf{z}_s^{|L|},\mathbf{z}_t^{|L|})$), DAN improves upon DDC by minimizing the discrepancies in all fully connected layers and replacing the basic MMD with multi-kernel MMD (MK-MMD):
\begin{equation}
\label{eq:DAN_obj}
L_{\text{CE}} + \alpha \sum_{l=1}^{|L|} \text{MK-MMD}^2(\mathbf{z}_s^l,\mathbf{z}_t^l),
\end{equation}
in which the kernel function in MK-MMD is expressed as:
\begin{equation}
\kappa(\mathbf{z}_s,\mathbf{z}_t) = \sum_{i=1}^m \beta_i \kappa_i(\mathbf{z}_s,\mathbf{z}_t), \quad \text{s.t.} \quad \beta_i>0, \sum_{i=1}^m \beta_i=1.
\end{equation}

JAN argues that the $|L|$ constraints in Eq.~(\ref{eq:DAN_obj}) is independent and does not incorporate the dependence between different layers. Hence, the authors of JAN minimize the joint MMD for all fully connected layers:
\begin{equation}
\label{eq:JAN_obj}
L_{\text{CE}} + \alpha \| \mathcal{C}_{\mathbf{z}_s^{1:|L|}} - \mathcal{C}_{\mathbf{z}_t^{1:|L|}} \|_{\otimes_{l=1}^{|L|} \mathcal{H}^l},
\end{equation}
in which $\mathcal{C}_{\mathbf{z}^{1:|L|}}$ refers to kernel mean embedding for $|L|$ variables using tensor product feature space:
\begin{equation}
\mathcal{C}_{\mathbf{z}_s^{1:|L|}} = \mathbb{E}_{\mathbf{z}_s^{1:|L|}} \left[ \otimes_{l=1}^{|L|} \phi^l(\mathbf{z}^l) \right].
\end{equation}

DJP-MMD aims to minimize the discrepancy in the joint distribution $p(\mathbf{z},y)$, but approximates the joint alignment in the following way:
\begin{equation}
D(p_s(\mathbf{z},y),p_t(\mathbf{z},y)) = D(p_s(y|\mathbf{z})p_s(\mathbf{z}),p_t(y|\mathbf{z})p_t(\mathbf{z})) \approx \mu_1 D(p_s(\mathbf{z}),p_t(\mathbf{z})) + \mu_2 D(p_s(\mathbf{z}|y),p_t(\mathbf{z}|y)).
\end{equation}

Obviously, such approximation simplifies the estimation, but does not obey the chain rule that $p(\mathbf{z},y) = p(y|\mathbf{z})p(\mathbf{z})$ (rather than $p(\mathbf{z}|y)$).

By contrast, our approach directly optimizes the following objective:
\begin{equation}
L_{\text{CE}} + \alpha \left[ D_{\text{CS}} (p_t (\mathbf{z}); p_s(\mathbf{z}) ) + D_{\text{CS}} (p_t (\hat{y}|\mathbf{z}); p_s(y|\mathbf{z})) \right],
\end{equation}
and approximates $y_t$ (the true label in target domain) with its prediction $\hat{y}_t$, which is a common trick in UDA literature.

\begin{figure}[t]
	\centering
		\includegraphics[width=.6\linewidth]{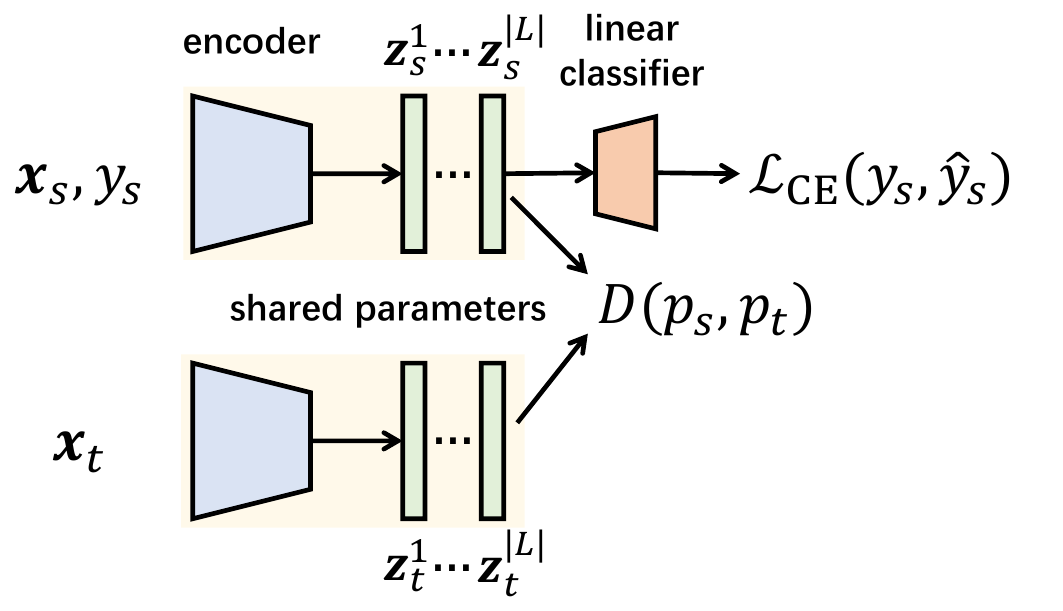}
	\caption{UDA framework.}
	\label{fig:UDA_framework}
\end{figure}

Our network for domain adaptation is built upon the EEGNet~\cite{lawhern2018eegnet}. Additionally, as a common practice, for all methods, we use a bottleneck subnetwork, comprising two fully connected layers (i.e., $|L|=2$) with ReLU activation functions before the final linear classification layer. Except for JAN which aligns the joint distribution $p(\mathbf{z}^1,\mathbf{z}^2)$, all other methods only takes $\mathbf{z}^2$ as the learned feature. The training is performed with PyTorch~\cite{paszke2019pytorch} on a NVIDIA GeForce RTX 3090 GPU. Adam optimizer is used with a learning rate of $1e-2$ and batch size $64$. In order to stabilize distribution alignment in the training phase, we employ a warm-up strategy, initially training without adaptation for the first $10$ epochs. For our method, we normalize the feature and adopt kernel size $\sigma=1$. For other methods, we choose either adaptive kernel size or fixed size $1$, depending on their performances. As for the weight $\alpha$ of the discrepancy term in each method, we perform a grid search from $1$ to $100$ with an interval of $10$.

\end{document}